\documentclass{article}
   
\usepackage{kpfonts, comment}  
\usepackage{comment}   
\usepackage{cpdef}
\usepackage{123}
\usepackage{float}
\usepackage{longtable}

\usepackage{lipsum}

\makeatletter
\newcommand{\hathat}[1]{%
\begingroup%
  \let\macc@kerna\z@%
  \let\macc@kernb\z@%
  \let\macc@nucleus\@empty%
  \hat{\raisebox{.2ex}{\vphantom{\ensuremath{#1}}}\smash{\hat{#1}}}%
\endgroup%
}
\makeatother

\newcommand{\cp}[1]{\textcolor{red}{}} 
\newcommand{\my}[1]{\textcolor{green!50!black}{}}
\newcommand{\cpedit}[1]{{#1}}
\newcommand{\dk}[1]{{}}
\newcommand{\kzedit}[1]{{#1}}

\newcommand{\neurips}[1]{{}}
\newcommand{\arxiv}[1]{{#1}}

\title{RLHF  from Heterogeneous Feedback via  \\
Personalization and Preference Aggregation}

\author{
Chanwoo Park\\
 \texttt{cpark97@mit.edu}\\
 MIT  \and
 Mingyang Liu\\\texttt{liumy19@mit.edu}\\ MIT \and 
Dingwen Kong \\ \texttt{dingwenk@mit.edu}\\ MIT  \and
 Kaiqing Zhang \\ \texttt{kaiqing@umd.edu} \\ University of Maryland, College Park \and
 Asuman Ozdaglar \\ \texttt{asuman@mit.edu} \\ MIT
   }

\date{April 30, 2024}

\begin{document}
\maketitle
\begin{abstract}
Reinforcement learning from human feedback (RLHF) has been an effective technique for aligning  AI systems with human values, with remarkable successes in fine-tuning large-language models recently. Most existing RLHF paradigms make the underlying assumption that human preferences are relatively \emph{homogeneous}, and can be encoded by a single reward model. 
In this paper, we focus on addressing the issues due to the inherent \textit{heterogeneity} in human preferences, as well as their potential \emph{strategic}  behavior in providing feedback. Specifically, we propose two frameworks to address heterogeneous human feedback in principled ways: personalization-based one and preference-aggregation-based one. For the former, we propose two approaches based on representation learning and clustering, respectively, for learning \emph{multiple} reward models that trade-off the bias (due to preference heterogeneity) and variance (due to the use of fewer data for learning each model by personalization). We then establish sample complexity guarantees for both approaches. For the latter, we aim to adhere to the single-model framework, as already deployed in the current RLHF paradigm, by carefully \emph{aggregating} diverse and truthful preferences from humans. We propose two approaches based on reward and preference aggregation, respectively:  the former utilizes social choice theory to aggregate individual reward models, with sample complexity guarantees; the latter directly aggregates the human feedback in the form of probabilistic opinions. Under the probabilistic-opinion-feedback model, we also develop an approach to handle strategic human labelers who may bias and manipulate the aggregated preferences with untruthful feedback. Based on the ideas in mechanism design, our approach ensures truthful preference reporting, with the induced aggregation rule maximizing  social welfare functions. 
\end{abstract}
\neurips{\vspace{-10pt}}
\section{Introduction}
\neurips{\vspace{-7pt}}
As AI models are becoming more powerful, there is greater emphasis on aligning their performance and priorities with the preferences of human users. In this context, reinforcement learning from human feedback (RLHF) has emerged as a promising approach, because it combines pre-trained large language models with direct human feedback \citep{ziegler2019fine, ouyang2022training, bai2022training}. RLHF utilizes human feedback in the form of \emph{preferences} over multiple responses in order to fine-tune the output of a pre-trained model, for example, by encouraging certain responses or types of output. The finetuning can be done by either learning a user reward model over user preference data, or by using the preference data directly (through direct preference optimization  \citep{rafailov2024direct}). In either case, accurately approximating user preferences is an important task, which becomes way more challenging when the target group of users is heterogeneous (\Cref{fig:intro-fig})  \citep{pollak1992demand, boxall2002understanding}.

\neurips{\vspace{-5pt}}
This paper contributes to this literature by providing a holistic study of learning (different) reward models from heterogeneous user preference data. There are two major challenges in this context. The first (\textbf{{\color{blue}C1}}) is a \textbf{pure learning one:} preference data from each individual might not be sufficiently rich to construct an accurate model of heterogeneous users. The second (\textbf{{\color{blue}C2}}) is after learning different reward models for heterogeneous users, \textbf{how to aggregate} them carefully to learn a single model. Moreover, with humans (who are oftentimes viewed as rational decision-makers) involved in the loop, they might \textit{strategically misreport} their preferences to manipulate this aggregated model. 
For example, in online rating systems, users may provide extreme feedback to disproportionately influence the overall ratings toward their viewpoint. Our approach develops ways of tackling these challenges.

\begin{figure}
    \centering
    \includegraphics[width= \textwidth]{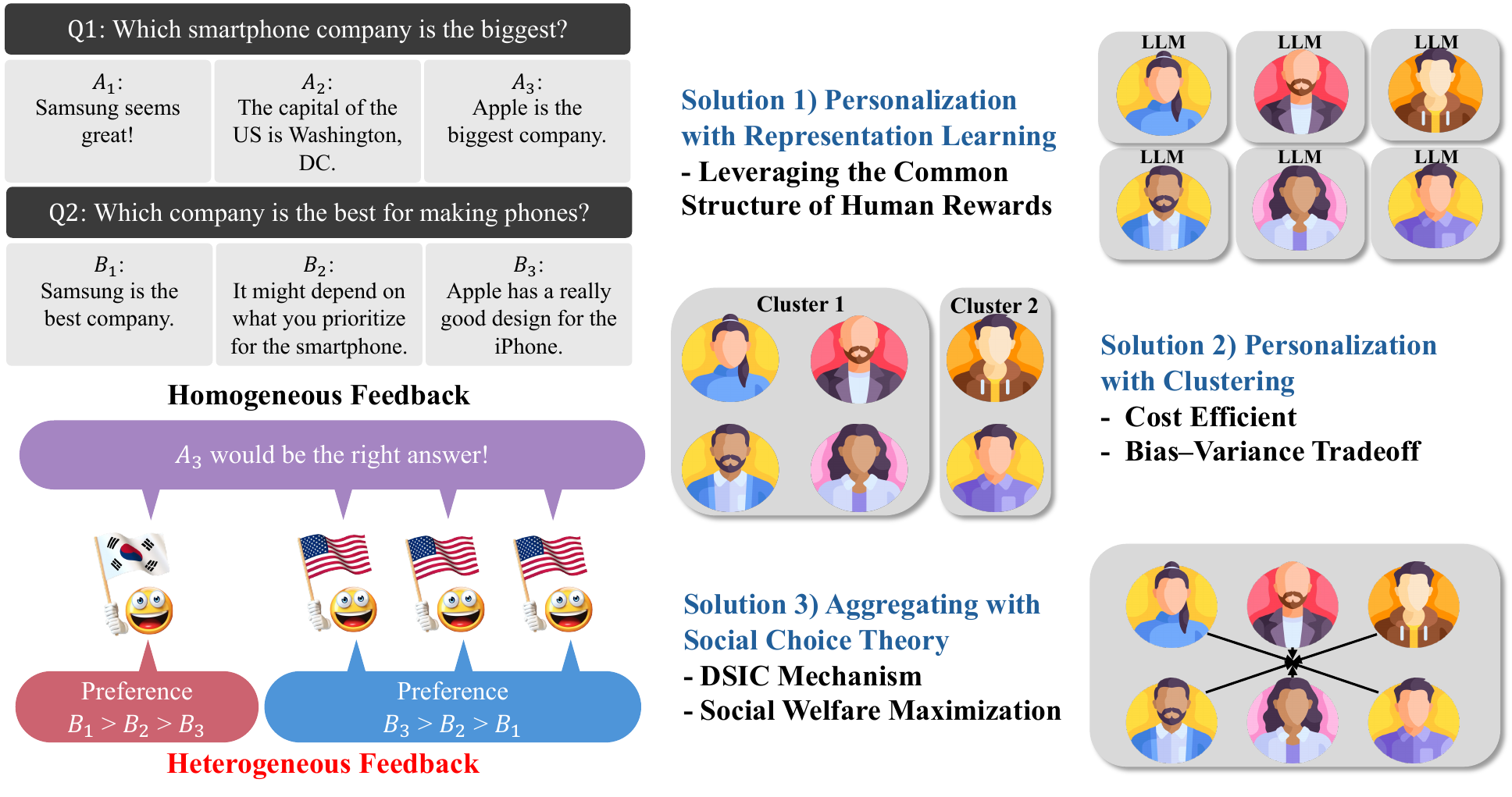}
    \caption{We demonstrate a setting where humans might have heterogeneous feedback. We provide a personalization-based framework and a human preference aggregation-based framework.}
    \label{fig:intro-fig}
    \neurips{\vspace{-17pt}}
\end{figure}
\neurips{\vspace{-4pt}}

To address (\textbf{{\color{blue}C1}}), we adopt two approaches based on representation learning, which assume that individual reward functions share a structure through a common representation. We model each reward function as the inner product of a common representation and a parameter vector. Given the lack of sufficient individual feedback, having a shared structure by representation helps articulate each user’s reward model.
The first approach \textbf{constructs a personalized reward model} for each user. In this approach, we find a common representation and learn each individual’s parameter vector by pooling every individual feedback. The second approach \textbf{segments user preferences into clusters} and learns a reward model for each cluster. This approach is useful when individual reward functions might not be available due to insufficient data. By assuming ``diversity of user’s parameter vectors'', which means that individual parameter vectors span the entire space of parameters (a common assumption in multi-task learning),  we show that this approach enables better sample complexity results. Leveraging data from all users helps learn the common representation, as the diversity assumption guarantees sufficient information about every dimension of the representation.

\neurips{\vspace{-4pt}}
To address (\textbf{{\color{blue}C2}}), we first estimate the parameters for each
individual’s reward model using the individual’s preference comparison data. Then, we aggregate reward models using a family of reward aggregation rules, which follows six pivotal axioms from social choice theory. 
We then provide sample complexities of the policy induced from the single aggregated reward model. We additionally provide a model with a different feedback type - probabilistic opinion. Concretely, instead of choosing a single answer from a pool of candidate answers, we allow the human labeler to choose a probability distribution over the answers, which indicates how much the labeler likes those answers. This type of feedback can arguably express the labeler's preference more accurately. Moreover, probabilistic opinion feedback does not require the relationship between the human reward model and preference. We consider various aggregation rules to aggregate their probabilistic opinion vectors into one. We showed that our suggested probabilistic opinion aggregation rule is equivalent to reward aggregation rules following six pivotal axioms, under the Plackett-Luce model \citep{plackett1975analysis, luce2005individual}.

\neurips{\vspace{-4pt}}
To deal with the \textit{strategic misreport} problem, we adopt a mechanism design approach whereby users correctly reporting their preferences is incentivized.
We model each human labeler's utility as a quasi-linear function, considering both the distance between her probabilistic opinion vector and the aggregated opinion vector, and the associated costs. Under this model, we show that our proposed aggregation rule maximizes \textit{social welfare}. Lastly, we design an incentive-compatible mechanism to guarantee truthful reporting by inducing proper cost in the human feedback collection process.

\neurips{\vspace{-10pt}}

\neurips{\subsection{Related Work} 
\neurips{\vspace{-8pt}}
We defer a detailed related work and comparison with recent work to \Cref{ssec:related}. }
\arxiv{\subsection{Related Works} \label{ssec:related}
\paragraph{Reinforcement Learning from Human Feedback.}  
Empirical evidence has demonstrated the efficacy of incorporating human preferences into reinforcement learning (RL) for enhancing robotics \citep{abramson2022improving, hwang2023promptable} and for refining large-scale language models \citep{ziegler2019fine, ouyang2022training, bai2022training}. These human inputs take various forms, such as rankings \citep{ziegler2019fine, ouyang2022training, bai2022training}, demonstrations \citep{finn2016guided}, and scalar ratings \citep{warnell2018deep}. A few approaches have been explored {empirically} to personalize RLHF.  \cpedit{For example, assigning fine-grained rewards to small text segments to enhance the training process \citep{wu2024fine}, or training each human labeler's reward model with Multi-Objective Reinforcement Learning perspective \citep{jang2023personalized, hwang2023promptable} have been proposed. Moreover, \citep{li2024personalized} suggested the training of each human labeler's reward model directly using personalized feedback with human embedding obtained by the human model, and also an approach for the clustering with finding cluster embedding.}

On the theory front, the studies of RLHF  have received increasing research interest. The most related prior works are \citep{zhu2023principled,zhan2023provable,wang2024rlhf}, where 
\citep{zhu2023principled} investigated the Bradley-Terry-Luce (BTL) model \citep{bradley1952rank} within the context of a linear reward framework; while \citep{zhan2023provable} generalized the results to encompass more general classes of reward functions. Both works concern the setting with offline preference data. Additionally,  \citep{kim2024unified} provided a linear programming framework for offline reward learning.
\citep{xiong2023gibbs} provided a theoretical analysis for KL-regularized RLHF. In the online setting, \citep{wang2024rlhf} established a correlation between online preference learning and online RL through a \textit{preference-to-reward} interface.

{Yet, to the best of our knowledge, there is no prior work that has analyzed RLHF with heterogeneous feedback with theoretical guarantees {(except the recent independent works discussed in detail below).}}

\paragraph{Representation Learning.} 
Early work of \citep{baxter2000model} established a generalization bound that hinges on the concept of a task generative model within the representation learning framework. More recently,  \citep{tripuraneni2021provable, du2020few} demonstrated that, in the setup with linear representations and squared loss functions, task diversity can significantly enhance the efficiency of learning representations. \cpedit{Moreover, \citep{tripuraneni2020theory} provided a representation learning with general representation and general loss functions. Representation learning has been extended to the reinforcement learning setting as well. } \cpedit{For low-rank Markov Decision Processes, where both the reward function and the probability kernel are represented through the inner products of state and action representations with certain parameters, \citep{agarwal2020flambe, ren2022free, uehara2021representation} explored the theoretical foundations for learning these representations.} Also, \citep{ishfaq2024offline, bose2024offline} analyzed the sample complexity of multi-task offline RL. 

\cpedit{
\paragraph{Reward and Preference Aggregation.} 
Preference aggregation is the process by which multiple humans' preference orderings of various social alternatives are combined into a single, collective preference or choice \citep{list2013social}. Arrow's Impossibility Theorem demonstrates that no aggregation rule for preference orderings can simultaneously meet specific criteria essential for ensuring a fair and rational aggregation of each human user's preferences into a collective decision \citep{arrow1951alternative}. Therefore, people considered replacing preference orderings with assigning real numbers to social alternatives \citep{sen2018collective, moulin2004fair}, which is sometimes called a reward (welfare) function in social choice theory.} for each human user. \citep{skiadas2016scale, moulin2004fair} provided reward (welfare) aggregation rules which satisfy several desirable properties. \cp{is it too simple?} Furthermore, an alternative method to circumvent Arrow's impossibility theorem involved aggregating preferences via probabilistic opinion \citep{stone1961opinion, lehrer2012rational}. In this approach, opinions are represented as probability assignments to specific events or propositions of interest.

\paragraph{Comparison with Recent  Works.}
While preparing the present work, we noticed two recent independent works that are closely related. Firstly, \citep{chakraborty2024maxmin} considered the aggregation of reward models with heterogeneous preference data, 
{focusing on aligning with the Egalitarian principle in social choice theory. In contrast, we provide a framework with various aggregation rules} and also prove that {the aggregation rules we considered are} also welfare-maximizing. More importantly, we design mechanisms for human feedback providers so that they can truthfully report their preferences {even when they may be strategic}. 
Moreover, we also develop another framework to handle heterogeneous preferences: the personalization-based one. 
Finally, we establish near-optimal sample complexity analyses for the frameworks we developed. 

More recently, {\citep{zhong2024provable}, which is a concurrent work with this paper, provided a theoretical analysis of reward aggregation in RLHF, focusing primarily on linear representations. Our work, in comparison, considers {general representation functions and general relationships between reward function and preference.} Unlike \citep{zhong2024provable}, where they focused on reward aggregation, we focus on personalization for every human labeler and also employ clustering techniques for personalization. \citep{zhong2024provable} and our paper also both investigated the case that reward and preference are not related. Our paper suggested a probabilistic opinion pooling with a mechanism design to effectively elicit truthful human preferences, presuming human labelers may be {strategic}. In contrast, \citep{zhong2024provable} analyzed an algorithm for a von Neumann winner policy, where a von Neumann winner policy is a policy that has at least a 50\% chance of being preferred compared to any other policy. Moreover, \citep{zhong2024provable} also explored the Pareto efficiency of the resulting policy.}  

\paragraph{Fundamentals of Auction Theory. }
Consider the sealed-bid auction mechanism \citep{vickrey1961counterspeculation}, where each participant $i \in [N]$ privately submits a bid $b_{i}(x)$ for every possible outcome $x \in X$, whose true value is $p_i(x) \in \RR$. {An auction is termed a {Dominant Strategic Incentive-Compatible (DSIC) auction} \citep{roughgarden2010algorithmic} if revealing each participant's true valuation is a weakly dominant strategy, i.e., an individual's optimal strategy is to bid their true valuation of the item, $b_i(x) = p_i(x)$ for all $x \in X$, irrespective of the bids $b_{-i}(x)$ submitted by others for all $x \in X$. This mechanism is also called a \emph{truthful} mechanism \citep{roughgarden2010algorithmic}.  An auction has {a social-welfare-maximizing  allocation rule} \citep{roughgarden2010algorithmic} if the outcome $x$ is $\argmax_{x \in X} \sum_{i \in [N]} p_i(x)$.}  }

\neurips{\vspace{-6pt}}

\paragraph{Notation.}  
The matrix $\mathbf{O}$ denotes an all-zero matrix, while $I$ stands for an identity matrix, of proper dimensions. We use $A \succ \mathbf{O}$ to denote that matrix $A$ is a positive definite matrix. The function $\sigma$ represents the Sigmoid function, defined by $\sigma(x) = 1/(1 + \exp(-x))$. The notation $[K]$ denotes the set $\{1, 2, \dots, K\}$. $\Delta(\cA)$ refers to a probability vector in $\RR^{|\cA|}$. The term $\sigma_k^2(A)$ denotes the $k$-th largest singular value of matrix $A$. 
\arxiv{A function \( f(x) \) is categorized based on the complexity notation as follows: \( f(x) = O(g(x)) \) {if there exists $C>0$ and $x_0$ such that  $f(x) \leq C g(x) $ holds for all $x \geq x_0$, \( f(x) = \Omega(g(x)) \) if there exists $C>0$ and $x_0$ such that \( f(x) \geq C \cdot g(x) \) for all \( x \geq x_0 \), \( f(x) = o(g(x)) \) if \( \lim_{x \to \infty} \frac{f(x)}{g(x)} = 0 \), and \( f(x) = \tilde{O}(g(x)) \) if \( f(n) = O(g(x) \cdot \log^k(g(x))) \) for some finite \( k \)}. 
The vector $e_1$ is defined as the standard basis vector of proper dimension with the first component being $1$. For a finite-dimensional vector $x$, the norm $\norm{x}_1$ refers to its $\ell_1$-norm, while }$\norm{x}_2$ refers to the $\ell_2$-norm\arxiv{, unless otherwise specified}.  We also define $\norm{x}_{\Sigma} = \sqrt{x^\intercal \Sigma x}$ for 
a positive definite matrix $\Sigma$. For a matrix $M$, the norm $\norm{M}_F$ denotes the Frobenius norm of $M$.  
\arxiv{The multinomial distribution is denoted by \(\text{Multinomial}(p_1, \ldots, p_n)\), where \(p_1, \ldots, p_n\) are the probabilities of outcomes for each of the \(n\) categories, respectively, with \( \sum_{i=1}^{n} p_i = 1 \) and $p_i\geq 0$ for all $i\in[n]$. Kullback-Leibler (KL) divergence between two probability distributions \( P, Q \in \Delta(X) \)  is defined as $\sum_{x \in \text{supp}(X)} P(x) \log \left(\frac{P(x)}{Q(x)}\right) $.}

\neurips{\vspace{-12pt}}
\section{Preliminaries}
\neurips{\vspace{-9.5pt}}

Most existing RLHF processes  (for language model fine-tuning)  consist of two main stages: (1) learning a model of human rewards {(oftentimes from preference data),} 
and (2) fine-tuning with the reference policy through Reinforcement Learning algorithms, e.g., Proximal Policy Optimization (PPO) \citep{schulman2017proximal}. {It may also be possible to avoid the explicit learning of reward functions while fine-tuning the policy directly from preference data \citep{rafailov2024direct}.}  
\neurips{\vspace{-9pt}}

\paragraph{Markov Decision Processes. }
We define the state \(s\) as an element of the set of possible prompts or questions, denoted by \(\mathcal{S}\), and the set of actions \(a\), contained in \(\mathcal{A}\), as the potential answers or responses to these questions.
{Consider an RLHF setting with $N$ human labelers (or users), each of whom has their own reward function. This setting can be characterized by a Markov Decision Process (MDP) with $N$ reward functions, represented by the tuple \(M = (\mathcal{S}, \mathcal{A}, H, (P_h)_{h \in [H]}, \br = (r_{ i})_{i \in [N]})\),} where \(H\) denotes the length of the horizon, 
\(P_h: \mathcal{S} \times \mathcal{A} \mapsto \Delta(\mathcal{S})\) is the state transition probability  at step \(h\in[H]\), 
$\cT:= (\mathcal{S} \times \mathcal{A})^H$ denotes the  set of all possible trajectories, and \(r_{i}: \cT\rightarrow \mathbb{R}\) is the reward function for individual \(i\) and trajectory $\tau \in \cT$, 
representing the utility of {human user} \(i\) from a sequence of responses to a given prompt.
{We assume $-R_{\max} \leq r_i(\tau) \leq R_{\max}$ for every $\tau \in \cT$ and $i \in [N]$, for some $R_{\max}>0$.} 
This reward model also covers  the case that $r_i(\tau) = \sum_{h \in [H]} r_{h, i}(s_h, a_h)$, {where $r_{h,i}: \cS \times \cA \to \RR$ denotes the state-action reward function for each step $h$ and individual $i$,}  
and $\tau = (s_1, a_1, s_2, a_2, \dots, s_H, a_H)$.  
The MDP concludes at an absorbing termination state with zero reward after \(H\) steps. 
A policy \(\pi_h: (\mathcal{S} \times \cA)^{h-1} \times \cS \to \Delta(\mathcal{A})\) is defined as a function mapping trajectories to distributions over actions for each step \(h \in [H]\)  within the horizon \(H\). {We define the history-dependent policy class as $\Pi$.} The collection of these policies across all steps is denoted by \(\pi\)$=(\pi_h)_{h=1}^{H-1}$. The expected cumulative reward of a policy \(\pi\) is given by $J(\pi; r_i):= \mathbb{E}_{\tau, \pi}[r_i(\tau)]$ where the expectation in the formula is taken over the distribution of the trajectories under the policy $\pi$. Trajectory occupancy measures, denoted by \(d_\pi: \cT \to [0, 1]\), are defined as \(d_\pi(\tau) := \mathbb{P}_\pi (\tau)\), which denotes the probability of generating trajectory $\tau$  following policy \(\pi\). 
\neurips{\vspace{-9pt}}

\paragraph{Relationship between Preference and Reward Function. }
For the MDP with \(M = (\mathcal{S}, \mathcal{A}, H, (P_h)_{h \in [H]}, \br = (r_{i})_{i \in [N]})\),  if we compare two trajectories $\tau_0$ and $\tau_1$, we define {some random variable} $o$ such that $o = 0$ if $\tau_{0} \succ \tau_{1}$, and $o = 1$ if $\tau_{0} \prec \tau_{1}$. Here, $\tau_0 \succ \tau_1$ indicates that $\tau_0$ is preferred than  $\tau_1$.  We assume that $P_{r_i}(o = 0\mid \tau_0, \tau_1) = \Phi(r_i (\tau_0) - r_i (\tau_1))$  for all $i \in [N]$, where $\Phi: \RR \to [0,1]$ is a monotonically increasing function, which satisfy $\Phi(x) + \Phi(-x) = 1$ and {$\log \Phi(x)$ is a strongly convex function.} For example, $\Phi(x) = \sigma(x)$ indicates the BTL model (\Cref{def:PL} below), a {frequently used model for the relationship between preference and reward}. Also, we define $P_{\br}(\cdot \mid \tau_0, \tau_1) := (P_{r_1}(\cdot \mid \tau_0, \tau_1)^\intercal, \dots, P_{r_N}(\cdot \mid \tau_0, \tau_1)^\intercal)^\intercal$. We call $P_{\br}$ and $P_{r_i}$ a preference {probability vector} induced by the reward vector $\br$ and the reward $r_i$.

\neurips{\vspace{-9pt}}

\section{Provable Personalized RLHF via Representation Learning}
\neurips{\vspace{-9pt}}
\subsection{Learning Personalized Reward Model}
\label{sec:personalization}
\neurips{\vspace{-9pt}}
{In this subsection, we provide the first approach in the personalization-based framework, based on representation learning.}
\neurips{\vspace{-9pt}}

\paragraph{Reward Function Class.} 
We will assume that we have access to a pre-trained feature function $\phi:\cT \to \RR^{d}$, which encodes a trajectory of states and actions (i.e.,  questions and answers) to a $d$-dimensional feature vector. {This covers the case where feature $\phi_h: \cS \times \cA \to \RR^d$ is defined at each state-action pair, i.e., $\phi(\tau):=\sum_{h \in [H]} \phi_h(s_h, a_h)$ for trajectory  $\tau = (s_1, a_1, \dots, s_H, a_H)$.} For example, it is common to use the penultimate layer of an existing {pre-trained LLM or other pre-trained backbones} to encode a long sentence to a feature vector \citep{donahue2014decaf, gulshan2016development, tang2015effective}. 

\neurips{\vspace{-4pt}}

Our first goal is to learn {multiple}  reward models for {each human user} using preference datasets. First, we define the reward function class as $$\cG_{\br} = \left\{ (\langle \psi_\omega(\phi(\cdot)), \theta_i \rangle)_{i \in [N]} \biggiven \psi_{\omega} \in \Psi, \theta_i \in \RR^k \text{ and } \norm{\theta_i}_2 \leq B\text{ for all }i \in [N]\right\},$$
{for some $B>0$,} 
where $\Psi$ is the set of representation functions parameterized by $\omega \in \Omega$, i.e.,  $\Psi = \{\psi_\omega \mid \omega \in \Omega \}$, where $\psi_\omega: \RR^d \to \RR^k$. We assume that $d\gg k$.  We denote $\btheta = (\theta_1, \dots, \theta_N)$, and to emphasize the relationship between reward and $(\omega, \btheta)$, we will write $r_{ \omega, \theta_i}(\cdot) := \langle \psi_\omega(\phi(\cdot)), \theta_i \rangle$ {for each individual $i \in [N]$} and $\br_{\omega, \btheta}(\cdot) := (r_{\omega, \theta_1} (\cdot), \cdots, r_{\omega, \theta_N}(\cdot))^\intercal \in \RR^N$. From this section, we will write $\br^\star = (r_1^\star, \dots, r_N^\star)$ as the underlying human reward functions.

\neurips{\vspace{-2pt}}
{\begin{assumption}[Realizability]
\label{assum:real}
 We assume that the underlying true reward can be represented as $r_i^\star(\cdot) = \langle \psi^\star (\phi(\cdot)), \theta_i^\star \rangle$ {for some representation function $ \psi^\star \in \Psi$ (in other words, there exists some  $\omega^\star \in \Omega$ such that $\psi_{\omega^\star} = \psi^\star$) and $\norm{\theta_i^\star}_2 \leq B$ for each individual $i \in [N]$.} 
\end{assumption}}
\neurips{\vspace{-7pt}}
To emphasize $(\omega, \btheta)$, we define shorthand notation  $P_{\omega, \btheta}:= P_{r_{\omega, \btheta}}$ as the preference probability induced by $r_{\omega, \btheta}$. \cpedit{We also write $P_{\omega, \theta}:= P_{\langle \psi_{\omega}(\phi(\cdot)), \theta \rangle}$, which is the probability induced by $\langle \psi_{\omega}(\phi(\cdot)), \theta \rangle$.}

\neurips{\vspace{-10pt}}
\subsubsection{Algorithms}
\label{ssec:alg-sec3}
\neurips{\vspace{-7pt}}

\cpedit{We introduce our algorithm for learning personalized policy. Compared to traditional RLHF algorithms \citep{ziegler2019fine, ouyang2022training, zhu2023principled}, we consider personalized reward function  by representation learning. }

\neurips{\vspace{-3pt}}

\Cref{alg:personal} outputs a joint estimation of $\psi^\star$ and $\btheta^\star$ with maximum likelihood estimation (MLE), together with personalized policies. The input of the algorithm is $\hat{\cD} = \cup_{i \in [N]} \hat{\cD}_i$ where $\hat{\cD}_i = \{(o_i^{(j)}, \tau_{i, 0}^{(j)}, \tau_{i,1}^{(j)})_{j \in [N_{p}]}\}$. Here, \cpedit{$\tau_{i, t}^{(j)}$ is sampled from the distribution $\mu_t$ for $t = 0, 1$}, and $o_i^{(j)} \sim P_{r^\star_i}(\cdot|\tau_0^{(j)}, \tau_1^{(j)})$. First, we estimate the reward function of human users. \cpedit{After estimating the reward functions, we construct a confidence set for the reward function as follows: Confidence set (\Cref{eqn:confidenceset-alg1-1}) with $\zeta' = C_8  \left(k \frac{ \xi^2\kappa^2 \log(\cN_{\cG_{\br}}(1/(NN_p))/ \delta)}{{\eta^2 NN_p}} + \frac{\xi^2(k + \log(N/\delta))}{\eta^2 N_p} + \lambda B^2\right) $, where $C_8, \lambda>0$ are constants, $\xi := \max_{x \in [-2R_{\max}, 2R_{\max}]}\left|\frac{\Phi'(x)}{\Phi(x)}\right|$, $\kappa:= (\min_{x \in [-2R_{\text{max}}, 2R_{\text{max}}]} \Phi'(x))^{-1}$, and $\eta := \min_{x \in [-2R_{\max},2R_{\max}]}\left(\frac{\Phi'(x)^2 - \Phi''(x)\Phi(x)}{\Phi(x)^2}\right)$. In the case that $\Phi(x) = \sigma(x)$ (i.e. $\Phi$ is a Sigmoid), $\xi \leq 1$ and $\kappa = \eta = \frac{1}{2 + \exp(-2R_{\max}) + \exp(2R_{\max})}$. This confidence set will be related to \Cref{thm:diverse}. 
 Lastly, we find \cpedit{the best policy based on 
the pessimistic expected value function.} $\mu_{i, \text{ref}}$ in \Cref{alg:personal} is a known reference trajectory distribution for individual $i\in [N]$, \cpedit{and it can be set as $\mu_1$. }} \neurips{We defer \Cref{alg:new-person} which addresses a scenario where a new human user, who was not a labeler before, aims to learn their own reward models.}
\neurips{\vspace{-7pt}}
\arxiv{\begin{algorithm}[!h]
	\caption{Personalized RLHF via Representation Learning \label{alg:personal}}
	\begin{algorithmic}
    \STATE \textbf{Input:} Dataset $\hat{\cD}=\cup_{i \in [N]} \hat{\cD}_i$ where $\hat{\cD}_i = \{(o_i^{(j)}, \tau_{i, 0}^{(j)}, \tau_{i,1}^{(j)})_{j \in [N_{p}]}\}$ is the preference dataset for the $i$th individual.    
    \STATE Estimate $\omega^\star$ and $\btheta^\star$ by 
    \[
    (\hat{\omega},\hat{\btheta})\leftarrow \argmax_{\omega \in \Omega, \norm{\theta_i}_2 \leq B \text{ for all } i \in [N]}\sum_{i \in [N]}\sum_{j \in [N_{p}]}  \log P_{\omega, \theta_i} (o_{i}^{(j)} \mid \tau_{i,0}^{(j)}, \tau_{i,1}^{(j)})
    \]
    \STATE Construct a confidence set of the reward function by 
    {\small
        \begin{equation}
        \begin{aligned}
           \cR'(\hat{\cD}) \leftarrow \cap_{i \in [N]}
        \biggl\{ & \br_{\omega, \btheta} \Biggiven   \frac{1}{N_p}\sum_{j \in [N_p]}\big|(r_{ \hat{\omega}, \hat{\theta}_i}(\tau_{i,0}^{(j)}) - r_{ \hat{\omega}, \hat{\theta}_i}(\tau_{i,1}^{(j)})) - (  r_{ \omega, \theta_i} (\tau_{i,0}^{(j)}) - r_{ \omega, \theta_i}(\tau_{i,1}^{(j)})) \big|^2  \leq \zeta' \biggr\}
        \end{aligned}
        \label{eqn:confidenceset-alg1-1}
    \end{equation}}
    \STATE Compute policy with respect to $\cR(\hat{\cD})$ (or $\cR'(\hat{\cD})$) for all $i \in [N]$ by
    \begin{align}
        \hat{\pi}'_i\leftarrow \argmax_{\pi \in \Pi} \min_{\br \in \cR'(\hat{\cD})} \left(J(\pi; r_i) - \EE_{\tau \sim \mu_{i, \text{ref}}}[r_i(\tau)]\right) \label{eqn:robust-alg1}
    \end{align}
    
    \STATE \textbf{Output:} $ (\hat{\omega},\hat{\btheta}, (\hat{\pi}'_i)_{i \in [N]})$.
    \end{algorithmic}
\end{algorithm}}

\arxiv{{\Cref{alg:new-person} addresses a scenario where a new human user, who was not a labeler before, aims to learn their own reward models using representations previously learned by other human users, focusing solely on learning $\theta^\star_0$.} They leverage the learned representation $\psi_{\hat{\omega}}$ from \Cref{alg:personal}. The input of the algorithm is $\hat{\cD}_0 = \{(o_0^{(j)}, \tau_{0, 0}^{(j)}, \tau_{0,1}^{(j)})_{j \in [N_{p}]}\}$. \Cref{alg:new-person} provides an estimation of $\theta^\star_0$ with MLE using the {frozen} representation $\psi_{\hat{\omega}}$. Similarly, after estimating the reward function, we construct confidence set for the MLE estimation with $\zeta = C_8  \left(k \frac{ \xi^2\kappa^2 \log(\cN_{\cG_{\br}}(1/(NN_p))/ \delta)}{{\eta^2 NN_p}} + \frac{\xi^2(k + \log(1/\delta))}{\eta^2 N_p} + \lambda B^2\right)$ for a constant $C_8>0$. Lastly, we find {the best policy based on the pessimistic expected value function.} $\mu_{0, \text{ref}}$ in \Cref{alg:new-person} is a known reference trajectory distribution.

\begin{algorithm}[!h]
	\caption{Transferable RLHF for a New Human User via Representation Learning \label{alg:new-person}}
	\begin{algorithmic}
    \STATE \textbf{Input:} Dataset $\hat{\cD}_0 = \{(o_0^{(j)}, \tau_{0, 0}^{(j)}, \tau_{0,1}^{(j)})_{j \in [N_{p}]}\}$ and $\hat{\omega}$ from \Cref{alg:personal}.     
    \STATE Estimate $\theta_0^\star$ by 
    \[
    \hat{\theta}_0\leftarrow \argmax_{\norm{\theta_0}_2 \leq B }\sum_{j \in [N_{p}]}  \log P_{\hat{\omega}, \theta_0} (o_{0}^{(j)} \mid \tau_{0,0}^{(j)}, \tau_{0,1}^{(j)})
    \]
    \STATE Construct a confidence set of the reward function by{ \small
    \begin{equation*}
        \begin{aligned}
           \cR(\hat{\cD}) \leftarrow
        \biggl\{ & r_{ \omega, \theta_0} \biggiven  \frac{1}{N_p}\sum_{j \in [N_p]}\big|(r_{ \hat{\omega}, \hat{\theta}_0}(\tau_{0,0}^{(j)}) - r_{ \hat{\omega}, \hat{\theta}_0}(\tau_{i,1}^{(j)})) - (  r_{ \omega, \theta_0} (\tau_{0,0}^{(j)}) - r_{ \omega, \theta_0}(\tau_{0,1}^{(j)})) \big|^2  \leq \zeta \biggr\}
        \end{aligned}
    \end{equation*}}
    \STATE Compute policy with respect to $\cR(\hat{\cD})$ by
    \[\hat{\pi}_0\leftarrow \argmax_{\pi \in \Pi} \min_{r_0 \in \cR(\hat{\cD}_0)} \left(J(\pi; r_0) - \EE_{\tau \sim \mu_{0, \text{ref}}}[r_0(\tau)] \right)
    \] 
    \STATE \textbf{Output:} $(\hat{\pi}_i)_{i \in [N]}$.
    \end{algorithmic}
\end{algorithm}
\neurips{
\subsubsection{Expected Value Function Gap for a New Human User}
\label{ssec:newhuman-theorem}
We have expected value function gap for a new human user as follows: 

\begin{theorem}
\label{thm:diverse-newmodel}
 \emph{(Expected Value Function Gap for a New Human User).}
 Suppose Assumptions \ref{assum:real}, \ref{assum:task_diverse}, \ref{assum:psi-unique}, and \ref{assum:point_concen} hold. For any $\delta \in (0, 1]$ and $\lambda >0$, with probability {at least} $1-\delta$, the output $\hat{\pi}_0$ of \Cref{alg:new-person} satisfies 
\begin{align*}
        &J(\pi_{0, \text{tar}}; r^\star_0) - J(\hat{\pi}_0; r^\star_0) \\
        &\leq 
         \sqrt{c  C_{\br}(\cG_{\br}, \pi_{i, \text{tar}}, \mu_{i, \text{ref}}, i)^2 \left( k \frac{ \xi^2\kappa^2 \log(\cN_{\cG_{\br}}(1/(NN_p))/ \delta)}{{\eta^2 NN_p}} + \frac{\xi^2(k + \log(1/\delta))}{\eta^2 N_p} + \lambda B^2\right)} 
\end{align*}    
where $c >0$ is a constant. 

\end{theorem}
}}
\neurips{\vspace{-4pt}}
\subsubsection{Results and Analyses}
\neurips{\vspace{-7pt}}
For ease of analysis, we consider the case where the sizes of preference datasets for each individual $i \in \{0\} \cup [N]$ are identical, i.e., $\hat{\cD}_i = \{(o_i^{(j)}, \tau_{i, 0}^{(j)}, \tau_{i,1}^{(j)})_{j \in [N_{p}]}\}$, satisfies $|\hat{\cD}_i| = N_p$ for all  $i \in \{0\} \cup [N]$. The result in this section can also be extended to the case with $|\hat{\cD}_i| = N_{p, i}$ for each individual $i$. 
We defer all the proofs of this section to  \Cref{appendix:sec3}. 

\begin{definition}
[Concentrability Coefficient]
\label{def:concentrability-coef}
The concentrability coefficient, \arxiv{with respect to}\neurips{\textit{w.r.t}}
a reward vector class $\mathcal{G}_{\br}$, human user $i$, a target policy $\pi_{\text{tar}}$ (which policy to compete with, which potentially can be the optimal policy $\pi_i^\star$ corresponding to $r^{\star}_i$), and a reference policy $\mu_{\text{ref}}$, is defined as follows:
\neurips{\small $
C_{\br}\left(\mathcal{G}_{\br}, \pi_{\text {tar }}, \mu_{\text {ref }}, i \right):=\max \left\{0, \sup _{\br \in \mathcal{G}_{\br}} \frac{\mathbb{E}_{\tau_0 \sim \pi_{\text {tar }}, \tau_1 \sim \mu_{\text {ref }}}\left[r^{\star}_i\left(\tau_0\right)-r^{\star}_i\left(\tau_1\right)-r_i\left(\tau_0\right)+r_i\left(\tau_1\right)\right]}{\sqrt{\mathbb{E}_{\tau_0 \sim \mu_0, \tau_1 \sim \mu_1}\left[\left|r^{\star}_i\left(\tau_0\right)-r^{\star}_i\left(\tau_1\right)-r_i\left(\tau_0\right)+r_i\left(\tau_1\right)\right|^2\right]}}\right\}.
$}
\arxiv{$$
C_{\br}\left(\mathcal{G}_{\br}, \pi_{\text {tar }}, \mu_{\text {ref }}, i \right):=\max \left\{0, \sup _{\br \in \mathcal{G}_{\br}} \frac{\mathbb{E}_{\tau_0 \sim \pi_{\text {tar }}, \tau_1 \sim \mu_{\text {ref }}}\left[r^{\star}_i\left(\tau_0\right)-r^{\star}_i\left(\tau_1\right)-r_i\left(\tau_0\right)+r_i\left(\tau_1\right)\right]}{\sqrt{\mathbb{E}_{\tau_0 \sim \mu_0, \tau_1 \sim \mu_1}\left[\left|r^{\star}_i\left(\tau_0\right)-r^{\star}_i\left(\tau_1\right)-r_i\left(\tau_0\right)+r_i\left(\tau_1\right)\right|^2\right]}}\right\}.
$$}
We also define the concentrability coefficient \neurips{($C_{r}(\cG_r, \pi_{\text{tar}}, \mu_{\text{ref}})$)} of the reward scalar class in \Cref{ssec:concnetrability}\arxiv{, and we denote this as $C_{r}(\cG_r, \pi_{\text{tar}}, \mu_{\text{ref}})$.}\neurips{.} 
\end{definition}
\neurips{\vspace{-3pt}}
\citep{zhan2023provable} provides an interpretation of concentrability coefficient. For example, if $\mu_{\text {ref}}=\mu_1$, 
the value of $C_{\br}\left(\mathcal{G}_{\br}, \pi_{\text {tar }}, \mu_1, i \right) \leq \sqrt{\max_{\tau \in \cT} \frac{d_{\pi_{\text{tar}}}(\tau)}{\mu_0(\tau)}}$, so 
this reflects the concept of ``single-policy concentrability'' \citep{rashidinejad2021bridging, zanette2021provable, ozdaglar2023revisiting}, which is commonly assumed to be bounded in the offline RL literature.  
\neurips{\vspace{-12pt}}

\cpedit{We consider the case that $(\theta_i)_{i \in [N]}$ are diverse (\Cref{assum:task_diverse}), which is critical for improving the sample complexity of \Cref{alg:personal} by outputting $(\hat{\pi}_i')_{i \in [N]}$. We will additionally assume the uniqueness of the representation up to the orthonormal linear transformation (\Cref{assum:psi-unique}), and uniform concentration of covariance (\Cref{assum:point_concen}). These assumptions are commonly used in multi-task learning \citep{du2020few, tripuraneni2021provable, lu2021power}}

\begin{assumption}[Diversity]
\label{assum:task_diverse}
The matrix $\Theta^\star=[\theta_{1}^\star,\cdots,\theta_{N}^\star]\in\RR^{k\times N}$ satisfies $\sigma_{k}^2(\Theta^\star)\geq \Omega\left(N/k\right)$.
\end{assumption}
\neurips{\vspace{-5pt}}

\cpedit{\Cref{assum:task_diverse} means that $\theta_i$ is evenly distributed in $\RR^d$ space for $i \in [N]$, which indicates ``diverse'' human reward function.} 
\begin{assumption}[Uniqueness of  Representation (up to Orthonormal-Transformation)]
\label{assum:psi-unique}
For any representation functions \cpedit{$\psi,\psi'\in\Psi$  and $\epsilon>0$, if there exists $\{v_i\}_{i=1}^T,\{v_i'\}_{i=1}^T$, and a trajectory distribution $\mu$ that satisfy}
$
\frac{1}{T} \sum_{i \in [T]} \EE_{\tau \sim \mu}\|\psi(\phi(\tau))^\top v_i-\psi'(\phi(\tau))^\top v_i'\|^2\leq \epsilon,
$ where 
$W=[v_1,v_2,\cdots,v_T]\in\RR^{k\times T}$ satisfies $\sigma^2_k(W)\geq \Omega\left(T/k\right)$, and $\norm{v_i}_2 \leq B$ for all $i \in [T]$. Then, there exists a constant orthonormal matrix $P$ such that
\[
\|\psi(\phi(\tau))-P\psi'(\phi(\tau))\|^2\leq  c k\epsilon/B
\]
\cpedit{for all trajectory $\tau$ where $c>0$ is a constant.}
\end{assumption}
\neurips{\vspace{-7pt}}

\cpedit{This assumption posits that if two representation functions,  \(\psi\) and \(\psi'\), yield sufficiently small differences in expected squared norms of their inner products with corresponding vectors over trajectory distributions, then they are related by a constant orthonormal transformation.} If $\psi_\omega(\phi(s,a)):= \omega \phi(s,a)$ where $\omega$ is ${k \times d}$ orthonormal matrix, we can prove that \Cref{assum:psi-unique} holds  with non-degenerate $\phi(s,a)$ distribution (\Cref{sssec:linear}). 

\begin{definition}
Given distributions $\mu_0, \mu_1$ and two representation functions  $\psi,\psi'\in\Psi$, define the covariance between $\psi$ and $\psi'$ with respect to $\mu_0, \mu_1$ to be 
        \neurips{\vspace{-3pt}}
\[
\Sigma_{\psi, \psi'}(\mu_0, \mu_1) :=\EE_{\tau_0 \sim \mu_0, \tau_1 \sim \mu_1 }[(\psi(\phi(\tau_0)) - \psi(\phi(\tau_1)))(\psi'(\phi(\tau_0)) - \psi'(\phi(\tau_1)))^\intercal]\in\RR^{k\times k}.
\]
        \neurips{\vspace{-3pt}}
Define the symmetric covariance as
        \neurips{\vspace{-3pt}}
\[
\Lambda_{\psi,\psi'}(\mu_0, \mu_1)=
\begin{bmatrix}
\Sigma_{\psi, \psi}(\mu_0, \mu_1) &\Sigma_{\psi, \psi'}(\mu_0, \mu_1)\\
\Sigma_{\psi', \psi}(\mu_0, \mu_1) & \Sigma_{\psi, \psi'}(\mu_0, \mu_1)
\end{bmatrix}.
\]
\end{definition} 
\neurips{\vspace{-7pt}}

  We make the following assumption on the concentration property of the representation covariances.
\begin{assumption}
\label{assum:point_concen}\emph{(Uniform Concentrability).}
For any $\delta \in (0, 1]$, there exists a number $N_{\text{unif}}(\Psi, \mu_0, \mu_1, \delta)$ such that for any $n \geq N_{\text{unif}}(\Psi, \mu_0, \mu_1, \delta)$, the empirical estimation $\hat{\Lambda}_{\psi,\psi'}(\mu_0, \mu_1)$ of $\Lambda_{\psi,\psi'}(\mu_0, \mu_1)$ based on {$n$ 
 independent trajectory sample pairs from distributions  $(\mu_0, \mu_1)$}, with probability at least $1-\delta$, will satisfy   the following inequality for all $\psi, \psi' \in \Psi$:

 \neurips{\vspace{-15pt}}

\[
1.1 \Lambda_{\psi,\psi'}(\mu_0, \mu_1) \succeq \hat{\Lambda}_{\psi,\psi'}(\mu_0, \mu_1) \succeq 0.9 \Lambda_{\psi,\psi'}(\mu_0, \mu_1).
\] 
\end{assumption}

\neurips{\vspace{-7pt}}
\Cref{assum:point_concen} means that the empirical estimate $\hat{\Lambda}_{\psi,\psi'}(\mu_0, \mu_1)$ closely approximates the true $\Lambda_{\psi,\psi'}(\mu_0, \mu_1)$ with high probability. Similarly, if $\psi_\omega(\phi(\tau)):= \omega \phi(\tau)$, $N_{\text{point}}(\Psi, \mu_0, \mu_1, \delta) = \tilde{\cO}(d)$ \citep[Claim A.1]{du2020few}. {If distributions $\mu_0, \mu_1$ are clear from the context, we omit the notation $\mu_0, \mu_1$ for $\Sigma_{\psi,\psi'}(\mu_0, \mu_1)$ and $\Lambda_{\psi,\psi'}(\mu_0, \mu_1)$.} 
Moreover, we also write $\Sigma_{\psi, \psi}$ as $\Sigma_{\psi}$ for  notational convenience.
\neurips{\vspace{-4pt}}

\arxiv{
With \Cref{assum:task_diverse} and \Cref{assum:psi-unique}, {$\psi^\star$ and $\psi_{\omega}$ are close up to an orthonormal matrix transformation, as asserted below}:

\begin{restatable}{corollary}{labelcorrect}
    \label{cor:label-is-correct} \emph{(Closeness between $\psi^\star$ and $\psi_{\omega}$).}
Suppose Assumptions \ref{assum:real}, \ref{assum:task_diverse}, and \ref{assum:psi-unique} hold. For any $\delta \in (0,1]$, with probability at least $1-\delta$, if $\br_{\omega, \btheta} \in \cR'(\cD)$ as specified in \Cref{alg:personal}, then there exists an orthonormal matrix $P_\omega$ such that 
\begin{align*}
   \left[ \norm{\psi^\star(\phi(\tau_0)) - \psi^\star(\phi(\tau_1)) - P_\omega(\psi_\omega(\phi(\tau_0))- \psi_\omega(\phi(\tau_1)))}^2 \right] \leq  k \frac{c_{\text{rep}} \kappa^2 \log(\cN_{\cG_{\br}}(1/(NN_p))/ \delta)}{{NN_p B^2}}
\end{align*}
for all $\tau_0, \tau_1$, where $c_{\text{rep}} >0$ is a constant.
\end{restatable}}

\cpedit{We present the gap of the expected value function between the target policy $\pi_{i, \text{tar}}$ and the estimated policy $\hat{\pi}_i$ for each individual $i \in [N]$. Here, $\pi_{i, \text{tar}}$, which may be the optimal policy $\pi_i^\star$ over $r_i^\star$, serves as the policy that $\hat{\pi}_i$ will compare with.} 

\begin{restatable}{theorem}{diverse}
    \label{thm:diverse}
    \emph{(Expected Value Function Gap).}
Suppose Assumptions \ref{assum:real}, \ref{assum:task_diverse}, \ref{assum:psi-unique}, and \ref{assum:point_concen} hold. For any $\delta \in (0, 1]$, all $i \in [N]$ and $\lambda >0$, with probability at least $1-\delta$, the output $\hat{\pi}_i'$ of \Cref{alg:personal} satisfies 
\neurips{{\small
\begin{equation}
\begin{aligned}
        &J(\pi_{i, \text{tar}}; r^\star_i) - J(\hat{\pi}_i'; r^\star_i) 
        \\
        &\leq 
         \sqrt{c  C_{\br}(\cG_{\br}, \pi_{i, \text{tar}}, \mu_{i, \text{ref}}, i)^2 \left( k \frac{ \xi^2\kappa^2 \log(\cN_{\cG_{\br}}(1/(NN_p))/ \delta)}{{\eta^2 NN_p}} + \frac{\xi^2(k + \log(N/\delta))}{\eta^2 N_p} + \lambda B^2\right)} 
\end{aligned} 
\label{eqn:thm3.4}
\end{equation}}}
\arxiv{
\begin{equation}
\begin{aligned}
        &J(\pi_{i, \text{tar}}; r^\star_i) - J(\hat{\pi}_i'; r^\star_i) 
        \\
        &\leq 
         \sqrt{c  C_{\br}(\cG_{\br}, \pi_{i, \text{tar}}, \mu_{i, \text{ref}}, i)^2 \left( k \frac{ \xi^2\kappa^2 \log(\cN_{\cG_{\br}}(1/(NN_p))/ \delta)}{{\eta^2 NN_p}} + \frac{\xi^2(k + \log(N/\delta))}{\eta^2 N_p} + \lambda B^2\right)} 
\end{aligned} 
\label{eqn:thm3.4}
\end{equation}}
where $c >0$ is a constant. 
\end{restatable} 

        \neurips{\vspace{-5pt}}
Lastly, we can also use the learned representation for a new human user \arxiv{as follows:

\begin{theorem}
\label{thm:diverse-newmodel}
 \emph{(Expected Value Function Gap for a New Human User).}
 Suppose Assumptions \ref{assum:task_diverse}, \ref{assum:psi-unique}, and \ref{assum:point_concen} hold. For any $\delta \in (0, 1]$ and $\lambda >0$, with probability {at least} $1-\delta$, the output $\hat{\pi}_0$ of \Cref{alg:new-person} satisfies 
\begin{align*}
        &J(\pi_{0, \text{tar}}; r^\star_0) - J(\hat{\pi}_0; r^\star_0) \\
        &\leq 
         \sqrt{c  C_{\br}(\cG_{\br}, \pi_{i, \text{tar}}, \mu_{i, \text{ref}}, i)^2 \left( k \frac{ \xi^2\kappa^2 \log(\cN_{\cG_{\br}}(1/(NN_p))/ \delta)}{{\eta^2 NN_p}} + \frac{\xi^2(k + \log(1/\delta))}{\eta^2 N_p} + \lambda B^2\right)} 
\end{align*}    
where $c >0$ is a constant. 

\end{theorem}
}\neurips{in \Cref{ssec:newhuman-theorem}.}

\begin{remark}[Sample Complexity] 
\cpedit{For Theorem \ref{thm:diverse}, }if we naively learn the personalization model without representation learning, $\cN_{\cG_{\br}}(1/(NN_p))$ will be very large. For example, if we use linear representation $\phi_{\omega}(x) = \omega x$ and $\omega$ is a $d \times k$ orthonormal matrix, then $\log(\cN_{\cG_{\br}}(1/NN_p)/\delta) \leq \cO\left((dk + Nk) \log\left(R_{\max} NN_p/\delta\right)\right)$ while naive personalization with

        \neurips{\vspace{-13pt}}
\cpedit{$$\cG_{\br}' = \left\{ (\langle \phi(\cdot), \theta_i \rangle)_{i \in [N]} \rangle  \biggiven  \theta_i \in \RR^d \text{ and } \norm{\theta_i}_2 \leq B\text{ for all }i \in [N]\right\} $$  %
provides $\cN_{\cG_{\br}'}(1/(NN_p)) \leq \cO\left( Nd \log\left(R_{\max} NN_p/\delta\right)\right)$. Since $d\gg k$, the bound of \Cref{eqn:thm3.4}'s right-hand side has a significant improvement when we use representation learning.} If the representation function class is an MLP class, we can use a known bracket number by \citep{bartlett2017spectrally}. 
\end{remark}

        \neurips{\vspace{-9pt}}
        
\cpedit{We also point out that the existing technique from representation learning literature does not cover the case with general representation function learning with a log-likelihood loss function with $\cO(1/N_p)$ rate, to the best of our knowledge. The technical results are thus of independent interest.}

        \neurips{\vspace{-5pt}}

Lastly, we examine the tightness of our analysis by the theoretical lower bound of the sub-optimality gap of personalization.
        \neurips{\vspace{-2pt}}

\neurips{\begin{theorem}[Informal, Lower Bound for the Sub-Optimality Gap of Personalization] \label{thm:lower-bound-3}
    For any $k>6$ and large $N_p$, there exists a representation function $\phi(\cdot)$ and $C>0$ so that 
            \neurips{\vspace{-15pt}}

{\begin{align*}
        \min_{i \in [N]} \inf_{\hat\bpi}\sup_{Q \in {\rm CB}} \left(\max_{\pi^* \in \Pi} J(\pi^*; r_{\omega, \theta_i})-J(\hat\pi;r_{\omega, \theta_i} )\right)\geq C\cdot\sqrt{\frac{k}{N_p}},
    \end{align*}}
    where $\text{CB}$ is a family of MDP with $N$ reward functions. 
\end{theorem}}

\arxiv{\begin{restatable}{theorem}{lbpersonal}
\label{thm:lower-bound-3}
\emph{(Lower Bound for the Sub-Optimality Gap of Personalization).} For any $k>6,N_p\geq Ck\Lambda^2$ and $\Lambda\geq 2$, there exists a representation function $\phi(\cdot)$ so that 
{\begin{align*}
        \min_{i \in [N]} \inf_{\hat\bpi}\sup_{Q \in {\rm CB}(\Lambda)} \left(\max_{\pi^* \in \Pi} J(\pi^*; r_{\omega, \theta_i})-J(\hat\pi;r_{\omega, \theta_i} )\right)\geq C\Lambda\cdot\sqrt{\frac{k}{N_p}},
    \end{align*}}
    where 
    \begin{align*}
        {\rm CB}(\Lambda)\coloneqq \cbr{Q\coloneqq  \left(\cbr{\mu_0,\mu_1},\{\tau_{i, 0}^{(j)},\tau_{i,1}^{(j)}\}_{i \in [N],j \in [N_p]} ,\omega,\btheta \right) \biggiven C_{\br}'(\cG_{\br},\pi^\star,\mu_{1}, i)\leq\Lambda \text{ for all }i \in [N]}
    \end{align*}    
    is the family of MDP with $N$ reward functions and $H=1$ instances, where
    \begin{align}
        C_{\br}'(\cG_{\br},\pi^\star,\mu_{1}, i)\coloneqq \max \left\{0, \sup _{\br \in \mathcal{G}_{\br}} \frac{\mathbb{E}_{\tau_0 \sim \pi^\star, \tau_1 \sim \mu_1}\left[r^{\star}_i\left(\tau_0\right)-r^{\star}_i\left(\tau_1\right)-r_i\left(\tau_0\right)+r_i\left(\tau_1\right)\right]}{\sqrt{\frac{1}{N_p} \sum_{j=1}^{N_p}\left[\left|r^{\star}_i\left(\tau_{i,0}^{(j)}\right)-r^{\star}_i\left(\tau_{i,1}^{(j)}\right)-r_i\left(\tau_{i,0}^{(j)}\right)+r_i\left(\tau_{i,1}^{(j)}\right)\right|^2\right]}}\right\}.
        \label{eq:define-Cr-prime}
    \end{align}
\end{restatable}}
Our approach for personalized reward lower bound builds upon \citep[Theorem 3.10]{zhu2023principled}. \arxiv{Note that all results in this paper still hold for the new concentrability coefficient $C_{\br}'$.} By \Cref{thm:lower-bound-3}, for general representation function class, we establish that \Cref{alg:personal} is near-optimal for the sub-optimality of the induced personalization policy, \cpedit{as $\log(\cN_{\cG_{\br}}(1/NN_p))$ can be small so that $\sqrt{k \frac{\log(\cN_{\cG_{\br}}(1/(NN_p))/ \delta)}{NN_p} }$ can be dominated by $\sqrt{\frac{k}{N_p}}$}. Note that if $\Psi$ is a linear representation class, \arxiv{our result for personalization (\Cref{thm:diverse})}\neurips{\Cref{thm:diverse}} still has a $\sqrt{k}$ gap compared to the lower bound (\Cref{thm:lower-bound-3}). \cpedit{This gap is also observed in \citep{tripuraneni2020theory}.} {We will leave the sharpening of this $\sqrt{k}$ factor as a future work.} 

\neurips{\vspace{-7pt}}
\subsection{Personalized RLHF  via Human User Clustering}
\label{sec:k-LLM}
\neurips{\vspace{-7pt}}

{We now provide the second approach in the personalization-based framework, through human user clustering.} In particular, fine-tuning an LLM  for each individual may be impractical. We thus propose an alternative approach that segments human users into clusters and fine-tunes an LLM for \emph{each cluster}. This strategy entails deploying $K$ clustered models{, which can be smaller than the number of human users $N$}. A critical aspect of this methodology is the way to generate clusters. {This clustering-based personalization has also been studied in the federated (supervised) learning literature \citep{mansour2020three,ghosh2020efficient,sattler2020clustered}. We introduce our algorithm next, based on the algorithmic idea in \citep{mansour2020three}.} 

\neurips{\vspace{-7pt}}
\subsubsection{Algorithms}
\neurips{\vspace{-7pt}}

We {partition all the $N$ human users} into $K$ clusters and find the best parameters for each cluster\arxiv{ as follows}:
\arxiv{\begin{align*}
    \max_{(r_{(k)})_{k \in [K]}} \sum_{i \in [N]}\frac{1}{N} \max_{k \in [K]} \EE_{\cD_i}\left[\log P_{r_{(k)}} \left(o_i \mid \tau_{i,0}, \tau_{i,1}\right)\right]. 
\end{align*}
Since we only have access to the empirical data distribution, we instead  solve the following problem:}
\begin{align}
    \max_{(r_{(k)})_{k \in [K]}} \sum_{i \in [N]}\frac{1}{N} \max_{k \in [K]} \sum_{j \in [N_p]} \log P_{r_{(k)}} \left(o_i^{(j)} \biggiven  \tau_{i,0}^{(j)}, \tau_{i,1}^{(j)}\right). 
    \label{eqn:sec4-maxmax}
\end{align}

\neurips{\vspace{-8pt}}
\Cref{alg:cluster} outputs $K$ clustered policies and a map from human users to clusters. The input of the algorithm is $\hat{\cD} = \cup_{i \in [N]} \hat{\cD}_i$ where $\hat{\cD}_i = \{(o_i^{(j)}, \tau_{i, 0}^{(j)}, \tau_{i,1}^{(j)})_{j \in [N_{p}]}\}$, which is the same as \Cref{alg:personal}. After estimating the representation parameter $\hat{\omega}$, the algorithm will estimate the reward function parameters $(\hat{\theta}_{(k)})_{k \in [K]}$ with \Cref{eqn:learntheta(i)}. Lastly, we find {the best policy based on 
 the expected value function}. We defer a practical algorithm that uses DPO \citep{rafailov2024direct} (and also refer to \Cref{ssec:dpo}) and EM \citep{moon1996expectation} algorithms to solve \Cref{eqn:learntheta(i)} to \Cref{alg:clusterDPO}. 
\arxiv{
\begin{algorithm}[!h]
	\caption{Personalized RLHF via Clustering \label{alg:cluster}}
	\begin{algorithmic}
    \STATE \textbf{Input:} Dataset $\hat{\cD}=\cup_{i \in [N]} \hat{\cD}_i$ where $\hat{\cD}_i = \{(o_i^{(j)}, \tau_{i, 0}^{(j)}, \tau_{i,1}^{(j)})_{j \in [N_{p}]}\}$ is the preference dataset for the $i$th individual, and $\hat{\omega}$ form \Cref{alg:personal}.
    \STATE Learn $\theta_{(i)}$ and the clustering map $f:[N] \to [K]$ by
    \begin{align}
        &(\hat{\theta}_{(k)})_{k \in [K]} \leftarrow \argmax_{\norm{\theta_{(k)}}_2 \leq B \text{ for all } {k \in [K]}} \sum_{i \in [N]} \max_{k \in [K]} \sum_{j \in [N_p]} \log P_{\hat{\omega}, \theta_{(k)}}(o_i^{(j)} \mid \tau_{i,0}^{(j)}, \tau_{i,1}^{(j)}) \label{eqn:learntheta(i)}
        \\
        &\hat{f}(i) \leftarrow \argmax_{k \in [K]} \sum_{j \in [N_p]} \log P_{\hat{\omega}, \hat{\theta}_{(k)}}(o_i^{(j)} \mid \tau_{i,0}^{(j)}, \tau_{i,1}^{(j)}) \text{ for all } i \in [N] \nonumber
    \end{align}
    
    \STATE For each $k \in [K]$,  
    \[\hat{\pi}_{(k)}\leftarrow \argmax_{\pi \in \Pi} \left(J(\pi; r_{\hat{\omega}, \hat{\theta}_{(k)}}) - \EE_{\tau \sim \mu_1}[r_{\hat{\omega}, \hat{\theta}_{(k)}}(\tau)]\right).
    \] 
    \STATE \textbf{Output:} $((\hat{\pi}_{(k)})_{k \in [K]}, (\hat{\theta}_{(k)})_{k \in [K]}, \hat{\omega}, \hat{f})$.
    \end{algorithmic}
\end{algorithm}

}

\neurips{\vspace{-8pt}}
\subsubsection{Results and Analyses}
\neurips{\vspace{-8pt}}

{To analyze the clustering-based personalization approach, we adapt the notion of \emph{label discrepancy} in  \citep{mohri2012new} to our RLHF setting, for preference data and a given reward function class.} 
We defer all the proofs of this section to \Cref{appendix:sec4}.

\begin{definition}[Label Discrepancy]
\label{def:discrepancy}
     Label discrepancy for preference distribution $\bD_i$ and $\bD_j$, which are distributions of $(o, \tau_0, \tau_1)$, with reward function class $\cG_{r}$ 
     is defined as follows: 
\begin{align*}
    \texttt{disc}(\bD_i, \bD_j, \cG_{r}) = \max_{\br \in \cG_{r}}\Big|\EE_{\bD_i} \log P_{r}(o \mid \tau_1, \tau_0) - \EE_{\bD_j} \log P_{r}(o \mid \tau_1, \tau_0)\Big|.
\end{align*}     
\end{definition}

The discrepancy is defined as  the supremum value of the difference between the log-likelihood of the preference data when taking expectations over two human dataset distributions. This quantity will be used in the analysis to characterize the gap between the log-likelihood of the estimated parameters and the underlying parameters. A similar concept is frequently used in domain adaptation \citep{mansour2009domain} and federated learning \citep{mansour2020three}.

\arxiv{\begin{lemma}[\cite{mansour2020three}]
\label{lemma:mansour}
For any $\delta \in (0, 1]$, with probability at least $1-\delta$, the output $((\hat{\pi}_{(k)})_{k \in [K]}, (\hat{\theta}_{(k)})_{k \in [K]}, \hat{\omega}, \hat{f})$ of \Cref{alg:cluster} satisfies 
\begin{align*}
    & \max_{\norm{\theta'_i} \leq B \text{ for all }i \in [N]} \sum_{i \in [N]} \sum_{j \in [N_{p, i}]} \log \left(\frac{P_{\hat{\omega}, \theta_i'}(o_i^{(j)} \mid \tau_{i, 0}^{(j)}, \tau_{i, 1}^{(j)})}{P_{\hat{\omega}, \hat{\theta}_{\hat{f}(i)} }(o_i^{(j)} \mid  \tau_{i, 0}^{(j)}, \tau_{i, 1}^{(j)})}\right) 
    \\
    &\leq C_{\text{cluster}} NN_p \left(\sqrt{\frac{\log (2K/\delta)}{N_p}} + \sqrt{\frac{kK\log(N_p/k)}{N_p}}+ \sum_{i \in [N]} \frac{1}{N}\texttt{disc}(\cD_i, \cC_{\hat{f}(i)}, \cG_{\psi_{\hat{\omega}}}) \right),
\end{align*} 
where $\cC_k:=\cup_{\hat{f}(i) = k} \cD_i$, $C_{\text{cluster}}>0$ is a constant, and $\cG_{\psi_{{\omega}}}:=\{r_{\omega, \theta}\mid \norm{\theta} \leq B\}$ for all $\omega \in \Omega$.
\end{lemma}}

\begin{restatable}{theorem}{cluster}
\label{thm:cluster}
\emph{({Total} Expected Value Function Gap).} 
Suppose Assumptions \ref{assum:real},  \ref{assum:task_diverse}, \ref{assum:psi-unique}, and \ref{assum:point_concen} hold. Also, assume that $C_r(\cG_r, \pi, \mu_{i, \text{ref}}, i) \leq C_{\text{max}}'$ for all policy $\pi$ and $i \in [N]$. For any $\delta \in (0, 1]$, all $i \in [N]$ and $\lambda >0$, with probability at least $1-\delta$, the output $((\hat{\pi}_{(k)})_{k \in [K]}, \hat{f})$ of \Cref{alg:cluster} satisfies 
\arxiv{\begin{align*}
        &\sum_{i \in [N]} \left(J(\pi_{i, \text{tar}}; r^\star_i) - J(\hat{\pi}_{\hat{f}(i)}; r^\star_i)\right) 
       \\
      &\leq c N \kappa \Biggl(\underbrace{\frac{\log (2K/\delta)}{N_p} + {\frac{kK\log(N_p/k)}{N_p}}}_{(i)}
 +{\frac{k \xi^2 \kappa^2 \log(\cN_{\cG_r}(1/(NN_p))/ \delta)}{{NN_p}} }
 \\
 &\qquad\qquad\qquad + \underbrace{\left(\sum_{i \in [N]} \frac{1}{N}\texttt{disc}(\cD_i, \cC_{\hat{f}(i)}, \cG_{\psi^\star}))\right)^2}_{(ii)}  + \left(\frac{\log(\cN_{\cG_{\psi^\star}}(1/NN_p)/ \delta)}{NN_p}\right)^2\Biggr)^{1/4},
\end{align*}}
\neurips{{\scriptsize \begin{align*}
        &\sum_{i \in [N]} \left(J(\pi_{i, \text{tar}}; r^\star_i) - J(\hat{\pi}_{\hat{f}(i)}; r^\star_i)\right)  \leq c N \kappa \Biggl(\underbrace{{\frac{\log (2K/\delta) + kK\log(N_p/k)}{N_p}}
 +{\frac{k \xi^2 \kappa^2 \log(\cN_{\cG_r}(1/(NN_p))/ \delta)}{{NN_p}} }}_{(i)}
 \\
 &\qquad\qquad\qquad + \underbrace{\left(\sum_{i \in [N]} \frac{1}{N}\texttt{disc}(\cD_i, \cC_{\hat{f}(i)}, \cG_{\psi^\star}))\right)^2}_{(ii)}  + \left(\frac{\log(\cN_{\cG_{\psi^\star}}(1/NN_p)/ \delta)}{NN_p}\right)^2\Biggr)^{1/4},
\end{align*}}} 
where $c>0$ is a constant. 
\end{restatable}
Note that \Cref{thm:cluster} addresses the bias-variance tradeoff: as the number of clusters (\(K\)) increases, term (i) (variance) increases, while term (ii) (bias) decreases.
\arxiv{
Also, we note that due to the $\sqrt{kK/N_p}$ order on the right-hand side of \Cref{lemma:mansour}, we have a slower rate in \Cref{thm:cluster} than \Cref{thm:diverse}. This gap is mainly due to the {fact that the} analysis of \Cref{lemma:mansour} should cover uniformly for arbitrary $\hat{f}$, {and also due to a difference between $\max$ and expectation of $\max$, which is bounded using McDiarmid's inequality.}}

\arxiv{
\begin{remark}
   {In contrast to the results in \Cref{sec:personalization}},  
    we additionally assume \( C_r(\cG_r, \pi, \mu_1, i) \leq C_{\text{max}}' \) in \Cref{thm:cluster}. To adopt a pessimistic approach, constructing a confidence set for clustered reward functions across all clusters is necessary. However, the ambiguity of which human user belongs to which cluster complicates this analysis, as pessimism would need to be applied to every potential cluster. Consequently, defining a confidence set for every possible clustering scenario is required, significantly complicating the analysis of the algorithm.
\end{remark}}

\section{Reward and Preference Aggregation}
\label{ssec:reward-agg}
\neurips{\vspace{-7pt}}

{This section {adheres to the RLHF setting with} a single LLM, while handling the heterogeneous human feedback by reward/preference aggregation.}
For reward aggregation, we first estimate individual reward functions and then aggregate these functions to form a unified reward model. In comparison, for preference aggregation, we introduce a novel framework termed ``probabilistic opinion {pooling}''. {Specifically,} instead of relying on binary comparison data, human users provide feedback as \emph{probability vectors}. This approach eliminates the need to aggregate heterogeneous preferences via reward functions, allowing for the direct aggregation of probabilistic opinions provided by users.
\neurips{\vspace{-7pt}}

\subsection{Reward Aggregation}
\neurips{\vspace{-7pt}}
{We introduce the following reward aggregation rules (Equations \eqref{eqn:agg-reward} and \eqref{eqn:agg-reward-new}), {which are favorable as they satisfy several} pivotal axioms in social choice theory. These axioms -- monotonicity, symmetry, continuity, independence of unconcerned agents, translation independence, and the Pigou-Dalton transfer principle -- are crucial for ensuring fairness and consistency in the decision-making process   \citep{list2013social,skiadas2009asset, skiadas2016scale}. We present the definition of these axioms in \Cref{appendix:def-of-axiom} for completeness.}  The aggregation rules are presented as follows:
\neurips{\vspace{-12pt}}

\neurips{\scriptsize
\begin{center}
\begin{minipage}{0.49\linewidth}
\begin{align}
\text{Agg}_{\alpha}(\br) =\left\{\begin{array}{lr}\frac{1}{\alpha}\log\left(\frac{1}{N}\sum_{i \in [N]}\exp(\alpha r_i)\right) & \alpha \neq 0  \\ \frac{1}{N}\sum_{i \in [N]}r_i\ & \alpha=0 \end{array}\right. \label{eqn:agg-reward} \end{align}
\end{minipage}\begin{minipage}{0.49\linewidth}
    \begin{align}
\text{Agg}_{\alpha}'(\br) =\left\{\begin{array}{lr}\frac{1}{N\alpha}\sum_{i \in [N]} (\exp(\alpha r_i) - 1) & \alpha \neq 0  \\ \frac{1}{N}\sum_{i \in [N]}r_i\ & \alpha=0 \end{array}\right. \label{eqn:agg-reward-new}    \end{align}
    \end{minipage}
\end{center}}

\arxiv{\small
\begin{center}
\begin{minipage}{0.49\linewidth}
\begin{align}
\text{Agg}_{\alpha}(\br) =\left\{\begin{array}{lr}\frac{1}{\alpha}\log\left(\frac{1}{N}\sum_{i \in [N]}\exp(\alpha r_i)\right) & \alpha \neq 0  \\ \frac{1}{N}\sum_{i \in [N]}r_i\ & \alpha=0 \end{array}\right. \label{eqn:agg-reward} \end{align}
\end{minipage}\begin{minipage}{0.49\linewidth}
    \begin{align}
\text{Agg}_{\alpha}'(\br) =\left\{\begin{array}{lr}\frac{1}{N\alpha}\sum_{i \in [N]} (\exp(\alpha r_i) - 1) & \alpha \neq 0  \\ \frac{1}{N}\sum_{i \in [N]}r_i\ & \alpha=0 \end{array}\right. \label{eqn:agg-reward-new}    \end{align}
    \end{minipage}
\end{center}}

\neurips{\vspace{-7pt}}

where $\br = (r_1, \dots, r_N)^\intercal$ {is a reward vector with trajectory input.  Note that \Cref{eqn:agg-reward} and \Cref{eqn:agg-reward-new} are equivalent {in the sense of the associated optimal policy}, as $\log(x)$ is monotonically increasing.
We can verify that $\lim_{\alpha \to -\infty} \text{Agg}_{\alpha}(\br) = \min_{i \in [N]} r_i$ and $\lim_{\alpha \to \infty} \text{Agg}_{\alpha}(\br) = \max_{i \in [N]} r_i$. This implies that when $\alpha$ is small (or large), the reward aggregation rule emphasizes $\min_{i \in [N]} (\text{or } \max_{i \in [N]}) r_i$, respectively. When $\alpha = 0$, \Cref{eqn:agg-reward} represents  utilitarianism, and when $\alpha \to -\infty$, \Cref{eqn:agg-reward} represents a Leximin-based aggregation rule  \citep{list2013social}.}

\neurips{\vspace{-7pt}}

\subsubsection{Algorithm and Analysis}
\neurips{\vspace{-7pt}}

\Cref{alg:aggregation} outputs a joint estimation of $\psi^\star$ and $\btheta^\star$ with maximum likelihood estimation  as \Cref{alg:personal}. The procedure is overall the same as \Cref{alg:personal}, except the last step for estimating the best policy for the pessimistic expected value function associated with the aggregated reward function. 
\arxiv{
\begin{algorithm}[!h]
	\caption{RLHF with \cpedit{Reward Aggregation}}\label{alg:aggregation}
	\begin{algorithmic}
    \STATE \textbf{Input:} Dataset $\hat{\cD}=\cup_{i \in [N]} \hat{\cD}_i$ where $\hat{\cD}_i = \{(o_i^{(j)}, \tau_{i, 0}^{(j)}, \tau_{i,1}^{(j)})_{j \in [N_{p}]}\}$ is the preference dataset for the $i$th human, $\lambda > 0$, and $\hat{\omega}$ from \Cref{alg:personal}. We also use \Cref{eqn:confidenceset-alg1-1} for constructing a confidence set of reward function $\cR'(\hat{\cD})$. 
    \STATE Compute policy with respect to $\cR'(\hat{\cD})$ for all $i \in [N]$ by
    \begin{align}
        \hat{\pi} \leftarrow \argmax_{\pi \in \Pi} \min_{\br \in \cR'(\hat{\cD})} \left(J(\pi; \text{Agg}_{\alpha}(r_1, \dots, r_N)) - \EE_{\tau \sim \mu_{\text{ref}}}[\text{Agg}_{\alpha}(r_1, \dots, r_N)(\tau)]\right). \label{eqn:robust-alg-agg-1}
    \end{align}
    \STATE \textbf{Output:} $ (\hat{\omega},\hat{\btheta}, \hat{\pi})$.
    \end{algorithmic}
\end{algorithm}
}

\arxiv{Similar to the results in \Cref{thm:diverse}, we have the following theorem for \Cref{alg:aggregation}:}
\neurips{\vspace{-2pt}}

\begin{restatable}{theorem}{agg}
    \label{thm:agg}
    \emph{(Expected Value Function Gap).}
    Suppose Assumptions \ref{assum:real}, \ref{assum:task_diverse}, \ref{assum:psi-unique}, and \ref{assum:point_concen} hold. For any $\delta \in (0, 1]$, all $i \in [N]$ and $\lambda >0$, with probability at least $1-\delta$, the output $\hat{\pi}$ of \Cref{alg:aggregation} satisfies 
\begin{align*}
        &J(\pi_{\text{tar}}; \text{Agg}_{\alpha}(\br^\star)) - J(\hat{\pi}; \text{Agg}_{\alpha}(\br^\star)) 
        \\
        &\leq 
        \sqrt{c_\alpha   C_{\br}(\cG_{\br}, \pi_{\text{tar}}, \mu_{\text{ref}})^2 \left( \frac{ k\kappa^2 \log(\cN_{\cG_{\br}}(1/(NN_p))/ (\delta/N))}{{NN_p}} + \frac{\xi^2(k + \log(N/\delta))}{\eta^2 N_p} + \lambda B^2\right)}
\end{align*} 
where $c_{\alpha} > 0$ is a constant depending on $\alpha$, and other constants are defined in \Cref{ssec:alg-sec3}. 
\end{restatable}
\neurips{\vspace{-4pt}}
We defer the proof of \Cref{thm:agg} to  \Cref{appendix:reward-agg-pfthm}. 
\arxiv{
Lastly, we prove the tightness of our analysis by providing a theoretical lower bound of the sub-optimality gap of aggregation.}\neurips{Lastly, we prove the tightness and near-optimality of our analysis for \Cref{alg:aggregation} by providing a theoretical lower bound of the sub-optimality gap of aggregation, which is deferred to \Cref{appendix:defstate-lower-agg}.}

\arxiv{\begin{restatable}{theorem}{lbagg}
\label{thm:lower-bound-5} \emph{(Lower Bound for the Sub-Optimality Gap of Aggregation).}
    For any $k>6,N_p\geq Ck\Lambda^2,\Lambda\geq 2$, and $\alpha \in \RR$ there exists a representation function $\phi(\cdot)$ so that 
    \begin{align*}
\inf_{\hat\bpi}\sup_{Q\in {\rm CB}(\Lambda)} \left(\max_{\pi^* \in \Pi} J(\pi^*; \text{Agg}_{\alpha}(\br_{\omega, \btheta}))-J(\hat\pi;\text{Agg}_{\alpha}(\br_{\omega, \btheta}))\right)\geq C\Lambda\cdot\sqrt{\frac{k}{N_p}},
    \end{align*}
    where 
    \begin{align*}
        {\rm CB}(\Lambda)\coloneqq \cbr{Q\coloneqq  \left(\cbr{\mu_0,\mu_1},\{\tau_{i, 0}^{(j)},\tau_{i,1}^{(j)}\}_{i \in [N], j \in [N_p]} ,\omega,\btheta \right) \biggiven C_{\br}'(\cG_{\br},\pi^\star,\mu_{1}, i)\leq\Lambda \text{ for all }i \in [N]}
    \end{align*}   
    is the family of MDP with $N$ reward functions and $H=1$ instances. $C_{\br}'$ is defined in \Cref{eq:define-Cr-prime}.
\end{restatable}

\cpedit{By \Cref{thm:lower-bound-5}, for general representation function class, we establish that \Cref{alg:aggregation} is near-optimal for the sub-optimality of the induced personalization policy.}}

\arxiv{\begin{remark}[Comparison with Recent Independent Work \citep{zhong2024provable}]
\cpedit{\citep{zhong2024provable} also considered reward aggregation rules, adhering to the axiom of \textit{scale independence} but not \textit{translation independence} \citep{moulin2004fair}. Here, scale/translation independence of the aggregation rule means that the aggregation rule should yield the same choice even if the reward functions are scaled/translated, respectively.} Our theoretical results can also be extended to the reward aggregation rule in \citep{zhong2024provable}. \cpedit{However, since we also consider the relationship of the reward aggregation rule and preference aggregation rule by probabilistic opinion pooling, which will be presented in the next section, we only present \Cref{eqn:agg-reward} for the reward aggregation rule.} 
\arxiv{\citep{zhong2024provable} also considered Nash's bargaining \citep{nash1953two}, which maximizes $\prod_{i \in [N]}(r_i - \min r_i)$ rather than $\prod_{i \in [N]} r_i$. In this case, they can also consider translation independence. Ours can also be extended by substituting $r_i$ to $r_i/\min r_i$ in the aggregation function (\Cref{eqn:agg-reward}), but we decided not to include it in our paper since we did not have this result in our initial draft.}  
\end{remark}}

\neurips{\vspace{-7pt}}
\subsection{Preference Aggregation with Probabilistic Opinion Data} 
\label{ssec:prefagg}

\neurips{\vspace{-7pt}}
\subsubsection{RLHF with Probabilistic Opinion Feedback}
\neurips{\vspace{-7pt}}
Consider a set of questions $\{s^{(j)}\}_{j \in [N_p]}$, and for each question $s^{(j)}$, there are $K$ potential answers denoted by $\cA^{(j)} := \{a_k^{(j)}\}_{k \in [K]}$. Traditional RLHF methods involve human labelers $i \in [N]$ selecting a preferred answer from $\cA^{(j)}$. This approach limits the human feedback to a singular choice, {which, though being simple,} restricts the expressiveness of human preferences. 

\neurips{\vspace{-3pt}}
To address this, we introduce a new setting whereby human labelers provide feedback as a probability vector $q_i^{(j)} \in \Delta(\cA^{(j)})$, which is also called \textit{probabilistic opinion} in social choice theory \citep{stone1961opinion, lehrer2012rational}.  
Here, $\Delta(\cA^{(j)})$ represents the set of all possible distributions over the answers in $\cA^{(j)}$. This allows labelers to quantify their preferences across multiple answers rather than selecting only one{, and can be implemented in practice without increasing too much of overload for feedback collection.} 
\neurips{\vspace{-3pt}}

Our setup does not assume a predefined relationship between each reward function for every human labeler and their preferences. Instead, we aggregate the diverse probabilistic preferences of multiple labelers into a consensus probability distribution over the answers. We define an aggregation function (or a \emph{probabilistic opinion pooling function}), $\text{Agg-p}_{\alpha}(\bP)$, which takes a tuple of human preference distributions $\bP = (P_1, \dots, P_N) \in (\Delta(\cA))^N$ and maps it to a single probability distribution in $\Delta(\cA)$ where $\cA$ is the potential answer set. For each $a \in \cA$, 
\begin{align}
\text{Agg-p}_{\alpha}(\mathbf{P})(a) :=\left\{\begin{array}{lr} \frac{\left(\sum_{i \in [N]}(P_i(a))^\alpha\right)^{1/\alpha}}{\sum_{a' \in \cA} \left(\sum_{i \in [N]}(P_i(a'))^\alpha\right)^{1/\alpha}} & \alpha \neq 0  \\        \frac{\left(\prod_{i \in [N]}P_i(a)\right)^{1/N}}{\sum_{a' \in \cA} \left(\prod_{i \in [N]}P_i(a')\right)^{1/N}} & \alpha=0 \end{array}\right..\label{eqn:agg-pref}
\end{align}
\neurips{\vspace{-16pt}}

\arxiv{\begin{remark}
The case where \(\alpha = 0\) is referred to as the geometric pooling function \citep{mcconway1978combination}. This function is known for preserving unanimity and not being eventwise independent, while it does satisfy external Bayesianity \citep{madansky1964externally, dietrich2016probabilistic}. External Bayesianity mandates that updating the probabilities with new information should yield consistent results regardless of whether the update occurs before or after the aggregation process \citep{genest1984characterization}.  
\end{remark}} \neurips{\noindent The case where \(\alpha = 0\) is referred to as the geometric pooling function \citep{mcconway1978combination}\footnote{We refer \Cref{ssec:remark-prob-op} for the discussion of the case with $\alpha = 0$.}.} 
 Interestingly, \Cref{eqn:agg-pref}, which describes the aggregation of {probabilistic} preferences, has a connection to \Cref{eqn:agg-reward-new}, concerning reward aggregation, under the assumption of the Plackett-Luce model for the relationship between reward functions and preference models (\Cref{def:PL}). We {then} formalize the connection between the probabilistic opinion pooling in \Cref{eqn:agg-pref} and the reward aggregation rule in  \Cref{eqn:agg-reward}. We defer the proof of \Cref{prop:agg-min-prob} to \Cref{appendix:ssec:agg-equiv}. 

\neurips{\vspace{-3pt}}
\begin{definition}
\label{def:PL}
The Plackett-Luce (PL) model \citep{plackett1975analysis, luce2005individual} quantifies the likelihood that a trajectory $\tau_k$ is preferred over all other pairs in the set $\{\tau_k\}_{k \in [K]}$ by assigning it a probability defined as

\neurips{\vspace{-8pt}}
\[
{P}_{r}\left(\tau_k \succ \tau_{k'}  \forall k' \neq k \biggiven (\tau_k)_{k \in [K]}\right) = \frac{\exp(r(\tau_k))}{\sum_{k' \in [K]} \exp(r(\tau_{k'}))}
\]

\neurips{\vspace{-8pt}}
where $r$ is the reward function for a human {labeler}. In the case where $k=2$, this formulation simplifies to the Bradley-Terry-Luce (BTL) Model \citep{bradley1952rank}. 
\end{definition}
\neurips{\vspace{-7pt}}
 
\begin{restatable}{theorem}{aggequiv}
\label{prop:agg-min-prob}
\cpedit{\emph{(Relationship between Reward Aggregation and Preference Aggregation).}
Suppose human preferences are modeled by the PL model, and all human labelers share a common lower bound on their reward functions. Let $(R_i(a))_{a \in \mathcal{A}}$ represent the reward function associated with action $a \in \cA$ and $P_i \in \Delta(\mathcal{A})$ denote the corresponding probabilistic opinion for individual $i \in [N]$. Then, the preference aggregation $\text{Agg-p}_{\alpha}(\mathbf{P})$, is equivalent to the preference derived under the PL model with the aggregated rewards $(\text{Agg}_{\alpha}(\mathbf{R}(a)))_{a \in \cA}$ for any $\alpha \in [-\infty, \infty]$.}
\end{restatable}

\neurips{\vspace{-8pt}}
While we generally do not presuppose any specific relationship between probabilistic opinions and reward functions, {\Cref{prop:agg-min-prob} shows that under the classical choice model of Plackett-Luce, these two aggregation rules can coincide\neurips{.}\arxiv{ (while the  probabilistic aggregation framework may potentially handle other cases).}}
\neurips{\vspace{-9pt}}

\subsubsection{Algorithm}
\neurips{\vspace{-7pt}}

\neurips{We defer the algorithm for probabilistic opinion pooling in \Cref{appedix:human-feedback-prob-op}. We can use the aggregated probabilistic opinions to fine-tune the policy using DPO-based algorithm. The detailed description is deferred to \Cref{appedix:human-feedback-prob-op}.} 
\arxiv{We provide an algorithm that uses the feedback {in the form of} probabilistic opinions (\Cref{alg:POP-DPO}). The only difference from the DPO algorithm \citep{rafailov2024direct} is to change the deterministic answer $a_i$ to the $a_i$ sampled based on the probabilistic opinion pooling, which is in the second line in the for loop of \Cref{alg:POP-DPO}.

\begin{algorithm}[!h]
	\caption{Probabilistic Opinion Pooling  DPO (POP-DPO)\label{alg:POP-DPO}}
	\begin{algorithmic}
    \STATE \textbf{Input:} Dataset $\hat{\cD}=\cup_{i \in [N]} \hat{\cD}_i$ where $\hat{\cD}_i = \{q_i^{(j)}(s_i^{(j)}), s^{(j)}, i)\}_{j \in [N_{p}]}$ is the probabilistic opinion dataset for the $i$th individual, $q_i^{(j)} \in \Delta(\cA)$ with $|\cA| = 2$, $\beta$ is a parameter for DPO, $\alpha$ is a parameter for aggregation
    \FOR{every epoch}
    \STATE For every question $s^{(j)}$ where $j$ is in the batch, $q^{(j)}:= \text{Agg-p}_\alpha (\bq^{(j)})$. 
    \STATE Sample $a^{(j)}_0 \sim \text{Multinomial}(q^{(j)})$ and define $a^{(j)}_1$ as non-selected answer. 
    \STATE Run a few steps of optimization to update $\pi$ (for example, gradient ascent or Adam) to maximize 
    \begin{align*}
        \sum_{j \in \text{batch}} \log \sigma\left(\beta \log \frac{\pi(a_{0}^{(j)} \mid s^{(j)})}{\pi^{\text{old}}(a_{0}^{(j)} \mid s^{(j)})} - \beta \log \frac{\pi(a_{1}^{(j)} \mid s^{(j)})}{\pi^{\text{old}}(a_{1}^{(j)} \mid s^{(j)})}\right)
    \end{align*}
    \ENDFOR
    \STATE \textbf{Output:} $\pi$
    \end{algorithmic}
\end{algorithm}
}
\neurips{\vspace{-10pt}}

\section{Mechanism Design for Preference Aggregation}
\neurips{\vspace{-9pt}}

Suppose that human labeler $i$ ($i \in [N])$ provides preference data by probabilistic opinion $P_i\in \Delta(\cA)$. {We now consider the natural scenario where the labelers may be \emph{strategic} -- given they are human beings with (certain degree of) rationality. In particular, knowing the form of preference aggregation (and the fact that they may affect the process), human labelers may provide \emph{untruthful} feedback of their preference, in order to benefit more in terms of their \emph{actual} utility/preference. {In particular, the untruthful preference may \emph{bias} the aggregated preference (that LLM will be fine-tuned over) towards their own preference, and thus manipulates the LLM output.} 
\arxiv{We demonstrate the scenario more quantitatively in the following example.}} 
\neurips{\vspace{-9pt}}

\paragraph{An Example with Untruthful Feedback.} 
{Consider a set of $N$ labelers evaluating two answers, where each labeler expresses a probabilistic opinion on the answers ($a_1, a_2$). Specifically, suppose labeler $N$ believes that $a_1$ is slightly preferable to $a_2$, represented by the probability vector $P_N = (0.6, 0.4)^\intercal$. Conversely, all other labelers $i \in [N-1]$ have probabilistic opinion favoring the second answer, represented by $P_i = (0.2, 0.8)^\intercal$.}
\arxiv{

}
{We assume that the aggregation of these opinions employs the $\text{Agg-p}_{-\infty}$ rule, defined as $\text{Agg-p}_{-\infty}(\bP)(a_t) = \frac{\min_{i \in [N]} P_i(a_t)}{\min_{i \in [N]} P_i(a_1) + \min_{i \in [N]} P_i(a_2)}$ for $t= 1, 2$, where $\bP$ represents the matrix of probabilistic opinions across all labelers and answers. Under truthful reporting, the aggregated result would be calculated as $\text{Agg-p}_{-\infty}\left(\bP\right) = \left(1/3, 2/3\right)^\intercal$. However, labeler $N$ can strategically provide an untruthful probabilistic opinion to distort the aggregated result toward his original view: If labeler $N$ reports a distorted opinion of $P_N' = (13/15, 2/15)^\intercal$ instead of $(0.6, 0.4)^\intercal$, the new aggregated opinion becomes} $\text{Agg-p}_{-\infty}\left(\bP'\right) = \left(0.6, 0.4\right)^\intercal$, where $\bP' = (P_1, \dots, P_{N-1}, P_N')$, which aligns exactly with labeler $N$'s probabilistic opinion{, while further deviating from other labelers' actual preference}. 
{This example underscores the potential of strategic behavior in the aggregation of probabilistic opinions, and thus highlights the importance of  incentivizing  truthful preference reporting.}  

\neurips{\vspace{-4pt}}

\arxiv{\subsection{Setup}
\label{sssec:setup-untruthful}}
{To address the untruthful feedback issue, we resort to the ideas in mechanism design \citep{nisan1999algorithmic,borgers2015introduction, roughgarden2010algorithmic}. Specifically, we will develop mechanisms that can impose some  \emph{cost} on human labelers,  so that they do not have the incentive to report untruthful preferences.}
\arxiv{ 

}
\neurips{For a given question $s$, we assume that labeler $i$'s probabilistic opinion vector is $p_i$, and the aggregated vector being $p=\text{Agg-p}(\bP)$ where $\bP=(p_1,\cdots,p_N)$. We additionally impose cost $c_i>0$ to labeler $i$ based on the reports and aggregation results. We assume the following quasi-linear utility of labeler $i$\arxiv{:
\[
u_i(\bP):= d(p_i,\text{Agg-p}(\bP))-c_i(\bP)
\]}
\neurips{: $u_i(\bP):= d(p_i,\text{Agg-p}(\bP))-c_i(\bP)$}
where $d(\cdot,\cdot)$ measures the distance between two probability distribution.} 
\neurips{
Under this utility model, we design an incentive-compatible mechanism to elicit truthful reports.
\neurips{\vspace{-2pt}}

\begin{theorem}[Informal]
\label{thm:inf-mec}
For any distance function $d(\cdot,\cdot)$ in a given class, there exists an aggregation rule in \eqref{eqn:agg-pref} maximizes social welfare $\text{Welfare}(\bP):=\sum_{i=1}^N d(p_i,\text{Agg-p}(\bP))$. Moreover, inspired by the Vickery-Clarke-Groves mechanism \citep{vickrey1961counterspeculation, clarke1971multipart, groves1973incentives}, we can design proper cost function $c_i\geq 0$, which makes truthful reporting incentive-compatible.
\end{theorem}

\neurips{\vspace{-8pt}}
Here incentive-compatibility means for labeler $i$, whatever the other labelers' reports are, truthful reporting will always maximize her own utility function. Intuitively, our mechanism punishes labeler $i$ through cost $c_i$ for the externality she posed on other labelers, i.e., if she makes other labelers worse off based on the aggregation outcome. We defer detailed explanation of this section in \Cref{appendix:setupexplanation}.
} 
{
\neurips{\vspace{-15pt}}

\neurips{\paragraph{Imposing Cost for  Human Feedback Collection.}  Though not being enforced in most existing RLHF frameworks, we believe it is reasonable and possible to incorporate it in the feedback collection, especially in scenarios where a single reward model (and thus a single LLM) is mandated. For example, the future large models may be regulated by some administrative agency, e.g., the government. These agencies' objective is for social good, despite the heterogeneity in human preferences, and also possess the power to enforce cost to human labelers, e.g., via taxing. It may also be possible for big technology companies who train LLMs, e.g., OpenAI, to incentivize truthful feedback through personalized and strategic (negative) payment (which corresponds to the cost here) to human labelers.} 
}

\arxiv{
{In this setup, we will first prove the existence of a cost function \( c_i : \Delta(\mathcal{A})^N \to \mathbb{R} \) for all $i \in [N]$ that induces truthful reporting of probabilistic opinions from human labelers. Here, the input of $c_i$ is the probabilistic opinion of every human labeler.}
This is also called the dominant strategy incentive-compatible (DSIC) mechanism \citep{nisan1999algorithmic,borgers2015introduction,roughgarden2010algorithmic}. Here incentive-compatibility means for labeler $i$, whatever the other labelers' reports are, truthful reporting will always maximize her own utility function. Intuitively, our mechanism punishes labeler $i$ through cost $c_i$ for the externality she posed on other labelers, i.e., if she makes other labelers worse off based on the aggregation outcome. 

{Moreover, we prove that this cost function $c_i$ not only induce DSIC but also maximize social welfare. We denote each human labeler's underlying (true) probabilistic opinion as $p_i\left(s^{\left(j\right)}\right)$ for each question $s^{(j)}$.} {Accounting for such cost, w}e define the quasi-linear \emph{utility function} of individual $i$ for question $s^{(j)}$ as 
\begin{align*}
    u_i^{\left(j\right)}\left(p_i\left(s^{\left(j\right)}\right), \left(P_i\left(s^{\left(j\right)}\right)\right)_{i \in [N]}
\right) = -d\left(p_i\left(s^{\left(j\right)}\right),\text{Agg-p}\left(\left(P_i\left(s^{\left(j\right)}\right)\right)_{i \in [N]}\right)\right) -c_i\left(\left(P_i\left(s^{\left(j\right)}\right)\right)_{i \in [N]}\right).
\end{align*} 
Here, $d: \Delta(\cA) \times \Delta(\cA) \to \RR$ represents the distance between the underlying true probabilistic opinion and the aggregated preference.
Moreover, we define the \emph{welfare function} {of individual $i$} from {addressing question} $s^{(j)}$ as $\text{Wel}_i^{(j)}(O) = -d(p_i(s^{(j)}) , O)$ for {any} $O \in \Delta(\cA)$. In other words, we define utility function as subtracting cost from the welfare of individual $i$. The concepts of utility and welfare discussed in this section should not be confused with the reward used in previous sections. Here, welfare pertains to the intrinsic rewards (preferences) of individuals \textbf{and} the aggregation of heterogeneous human rewards (preferences).

{\paragraph{Imposing Cost for  Human Feedback Collection.}  Though not being enforced in most existing RLHF frameworks, we believe it is reasonable and possible to incorporate it in the feedback collection, especially in scenarios where a single reward model (and thus a single LLM) is mandated. For example, the future large models may be regulated by some administrative agency, e.g., the government. These agencies' objective is for social good, despite the heterogeneity in human preferences, and also possess the power to enforce cost to human labelers, e.g., via taxing. It may also be possible for big technology companies who train LLMs, e.g., OpenAI, to incentivize truthful feedback through personalized and strategic (negative) payment (which corresponds to the cost here) to human labelers.}

\begin{remark}[Examples of Distance Function $d$] 
\label{rmk:renyi}
    We can instantiate $d(p,q)$ as the KL-divergence. Also, we may instantiate  $d_{\alpha}(p,q) = \text{sgn}(\alpha) \frac{1}{1-\alpha}\sum_{j \in \cA}\left(1 - p_j^{\alpha}q_j^{1-\alpha}\right)$, which is a variant of \kzedit{the} $\alpha$-Renyi divergence for $\alpha \neq 0$. One can easily check that $d_{\alpha}(p, q) \geq 0$. In fact, one can also prove that $\lim_{\alpha \to 1} d_{\alpha}(p,q) = d(p,q)$ {with $d(p,q)$ being the KL-divergence} (\Cref{appendix:relationship}).
\end{remark}

\subsection{Mechanism and Guarantees}
\label{sssec:vcg}

We design a mechanism inspired by the Vickery-Clarke-Groves mechanism \citep{vickrey1961counterspeculation, clarke1971multipart, groves1973incentives}, as defined below. 

\begin{definition}[VCG Mechanism]
\label{def:vcg}
Assume that there are $n$ strategic agents and a finite set $X$ of outcome, and each individual $i$ has a private valuation $v_i$ for each outcome $x \in X$. \cpedit{The bidding $\bb = (b_1, \dots, b_N)^\intercal \in (\RR^{|X|})^N$ where $b_i \in \RR^{|X|}$ is bidding for all outcome of individual $i\in[N]$.} Define their utility function as $v_i(\bx(\bb)) - c_i(\bb)$, where $\bx: (\RR^{|X|})^N \to X$ is the allocation rule and $c_i: (\RR^{|X|})^N \to \RR$ is the cost function. \cpedit{The summation of welfare function of all agents}
is defined as $\text{Wel}(x) = \sum_{i \in [N]}v_i(x)$ for all $x \in X$. The goal is to design $\bx$ 
and {$(c_i)_{i\in[N]}$} functions to make a DSIC {and} welfare{-}maximizing mechanism. The following $\bx$ and $c_i$ for $i \in [N]$ is DSIC welfare maximizing mechanism:
\begin{align*}
    \bx(\bb) = \argmax_{x \in X} \sum_{i \in [N]} b_i(x), \qquad c_i(\bb) = \max_{x \in X}\sum_{j \neq i} b_j(x) - \sum_{j \neq i} b_j(\bx(\bb)) \text{ for all } i \in [N].
\end{align*}
\end{definition}
The private valuation $v_i$ is corresponding to our welfare function for $i$th individual. \cpedit{Unfortunately, the {classical} VCG mechanism presents certain limitations such as it cannot be solved in polynomial time in general \citep{nisan1999algorithmic, borgers2015introduction, roughgarden2010algorithmic}.}
{We here adopt certain forms of allocation rule (which corresponds to the aggregation rule in our RLHF setting) and cost functions as follows, which allow the outcome set to be a simplex (with infinitely many outcomes)}: 
\begin{align}
    \text{Agg-p}(\bP) = \argmin_{p \in \Delta(\cA) } \sum_{i \in [N]}d(\bP, p), \qquad c_i(\bP) = \sum_{j \neq i}d(P_i, \text{Agg-p}(\bP)) - \min_{p \in \Delta(\cA)}\sum_{j \neq i}d(P_i, p)
    \label{eqn:agg-p-cost}.
\end{align}

\begin{restatable}{theorem}{vcgmain}
\emph{(DSIC Welfare-Maximizing Mechanism).}
\label{thm:dsic}
    The aggregation {rule} and the cost function as in  \Cref{eqn:agg-p-cost} provide a DSIC welfare-maximizing mechanism.  
\end{restatable}

Due to the modeling, we have an advantage compared to the original VCG mechanism.  The minimization in the aggregation function can be achieved using a simple optimization method such as gradient descent, which makes our aggregation rule and cost function computation easy, which is in contrast with the original VCG mechanism.

\cpedit{Now, we connect our mechanism design with pre-defined preference aggregation function ($\text{Agg-p}_\alpha$ in \Cref{eqn:agg-pref}). \Cref{thm:aggrenyi} implies that \Cref{eqn:agg-pref} is maximizing social welfare and also we are available to construct the cost function to make human feedback truthful. 
}
\begin{restatable}{theorem}{aggprenyi}
\label{thm:aggrenyi}
    \cpedit{If we set $d$ as a variant of {the} $\alpha$-Renyi distance for $\alpha \neq 0$ (\Cref{rmk:renyi}) and define $d$ as KL-divergence for $\alpha = 0$, the DSIC welfare-maximizing aggregation rule is \Cref{eqn:agg-pref}.} \cpedit{Therefore, aggregation rule \Cref{eqn:agg-pref} is also welfare-maximizing with appropriate cost function.}
\end{restatable}

If we assume the relationship between reward and preference {follows} the PL model (\Cref{def:PL}), then \Cref{eqn:agg-reward} implies a welfare-maximizing aggregation rule, \cpedit{which connects reward aggregation and mechanism design.}  We defer all proofs for the results in \Cref{sssec:vcg} to \Cref{appendix:vcg}. 
}

\begin{remark}
An analog of \neurips{\Cref{thm:inf-mec}}\arxiv{\Cref{thm:dsic} and \Cref{thm:aggrenyi}} can also be applied to reward aggregation. Additionally, under the PL model, the mechanism design for reward aggregation and preference aggregation coincide.
\end{remark}

\neurips{\vspace{-12pt}}
\section{Experiments}
\neurips{\vspace{-12pt}}

We now conduct an empirical evaluation of our methods' performance on a text summarization task, using the Reddit TL;DR summarization dataset and the Reddit TL;DR human feedback dataset \neurips{(\texttt{comparison} and \texttt{axes evals})} \citep{stiennon2020learning}. 
\arxiv{For the Reddit TL;DR summarization dataset, \citep{stiennon2020learning} filtered the TL;DR summarization dataset \citep{volske2017tl} to ensure quality. The Reddit TL;DR human feedback dataset is constructed with two components: \texttt{comparison} and \texttt{axes evals}. The \texttt{comparison} component contains labeled comparisons between pairs of summaries with workers identified by unique IDs, while the \texttt{axes evals} component contains ratings of summaries along three axes: accuracy, coverage, and coherence.} We used GPT-J 6B \citep{wang2021gpt} 
and LLaMA3 8B models \citep{metallama3} 
in our experiments. 

\arxiv{
\vspace{5pt}
\noindent
\begin{minipage}{0.48\textwidth} %
    \centering
    \includegraphics[width=0.9\textwidth]{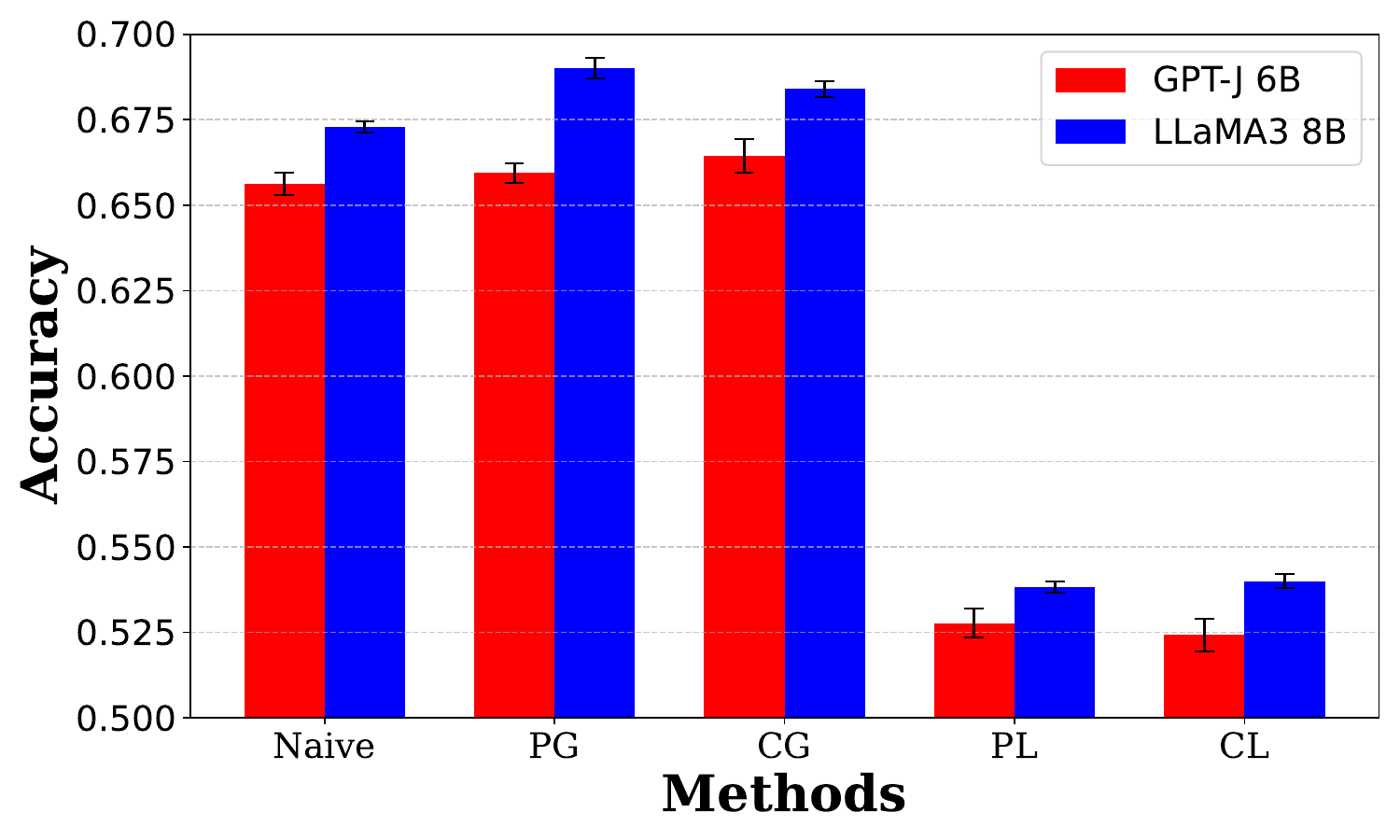}
    \captionof{figure}{Accuracy of different methods with 3 times experiments. P: Personalized, C: Clustered, G(L): General (Linear) representation. Naive RLHF: original training method.}
    \vspace{10pt}
    \label{fig:accuracy}
\end{minipage}
\hspace{7pt}
\begin{minipage}{0.49\textwidth}
    In \Cref{ssec:reward-model-performance}, we fine-tuned the personalized reward model with \Cref{alg:personal} and \Cref{alg:cluster}, without pessimism. We ranked workers based on the number of annotated comparisons in the training split of the dataset and included the top 5 workers for training. To balance the number of samples for each worker, we took the worker with the fewest samples among the top 5 as the baseline. We then randomly sampled the same number of comparisons from the other workers so that each worker had 5,373 comparison samples, resulting in a total of 26,865 samples for training. Similarly, for the validation set, we applied the same method. We randomly sampled the same number of comparisons as the worker with the fewest samples from the top 5 workers used in training. Each worker had 1,238 samples for validation, resulting in a total of 6,190 samples for validation. 
    \vspace{3pt}
\end{minipage}} 
\arxiv{In \Cref{ssec:examples}, we fine-tuned the personalized reward model using \Cref{alg:aggregation}, without incorporating pessimism. We considered three types of reward functions: accuracy-reward, coverage-reward, and coherence-reward. Since this dataset is only publicly available for the validation set (with 8,585 samples) and the test set (with 6,313 samples), we used the validation set for fine-tuning the training set of our model and validated it with the samples in the test set. } We defer \arxiv{other} details in \Cref{appendix:exp-detail}.
 
\neurips{\vspace{-7pt}}

\subsection{Experiment 1: Reward Model Performance with \texttt{comparison} Dataset}
\label{ssec:reward-model-performance}
\neurips{\begin{wrapfigure}{r}{0.4\textwidth}
    \centering         \includegraphics[width=0.4\textwidth]{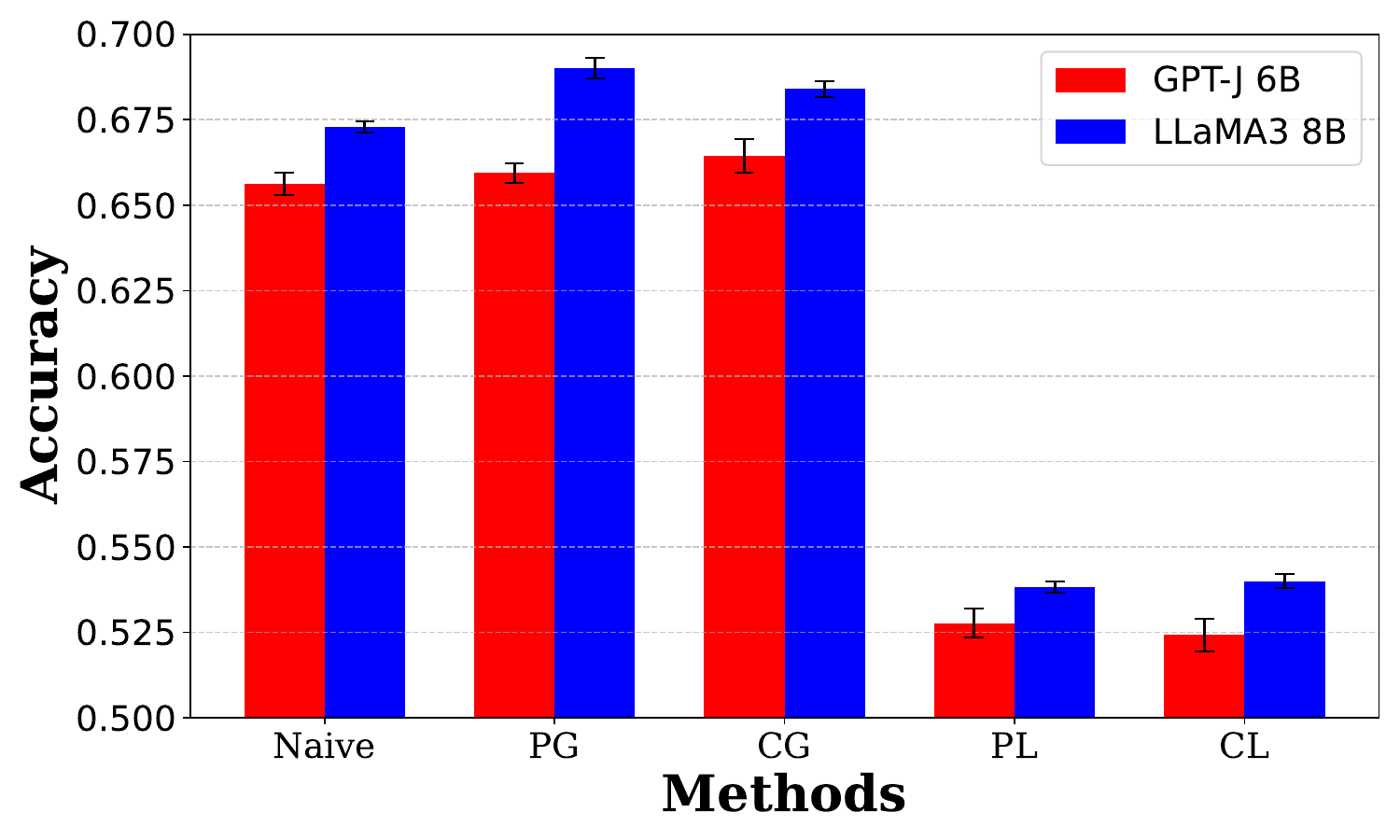}
            \neurips{\vspace{-10pt}}
        \caption{Accuracy of different methods with 3 times experiments. P: Personalized, C: Clustered, G(L): General (Linear) representation. Naive RLHF: original training method.}
        \label{fig:accuracy}
        \neurips{\vspace{-10pt}}
\end{wrapfigure}}
In Experiment 1, we compared our \Cref{alg:personal} and \Cref{alg:cluster} with naive RLHF methods. We constructed a reward model using a supervised fine-tuned language model and added a linear layer to represent individual reward functions{, as in our model in \Cref{sec:personalization}}.
\arxiv{For the reward model structure of the general representation function, we froze the first 70\% of the language model's layers, using the outputs of these layers as the representation. For the linear representation function, we froze the entire language model and only trained the additional final layer.}\neurips{We provide a detailed discussion of the reward model structure for general representation and linear representation in \Cref{appendix:exp-detail}.} We used two clusters for the personalized reward model with user clustering. We evaluated the reward models based on their accuracy in correctly assigning higher rewards to chosen summaries over rejected summaries in the validation set. Our results, shown in \Cref{fig:accuracy}, indicate that clustering methods can efficiently learn the personalized reward model. Furthermore, personalization with general representation learning is necessary, as indicated by the performance gap compared to personalization with linear representation learning. Notably, for LLaMA3 8B, the performance differences between the Naive method and both PG and CG are statistically significant by t-test ($p<0.006$). 

\neurips{\vspace{-10pt}}

\subsection{Experiment 2: Output Examples of Reward Aggregation with \texttt{axes} Dataset}
\neurips{\vspace{-10pt}}

\label{ssec:examples}
In Experiment 2, we aggregated three axes rewards using \Cref{eqn:agg-reward} with $\alpha = -\infty, -1, 0, 1, \infty$. We included representative outputs from these aggregated results in \Cref{ssec:detailed-ex}.

\arxiv{\section*{Acknowledgement}
The authors would like to thank 
Jisu Jang for providing helpful feedback for \Cref{fig:intro-fig}. }
\bibliographystyle{ims}
\bibliography{cparkbib}
\clearpage
\appendix

~\\
\centerline{{\fontsize{18}{18}\selectfont \textbf{Supplementary Materials for}}}

\vspace{6pt}
\centerline{\fontsize{13.5}{13.5}\selectfont \textbf{
	  ``RLHF from Heterogeneous Feedback via}}

   \centerline{\fontsize{13.5}{13.5}\selectfont \textbf{Personalization and Preference Aggregation''}}

\tableofcontents
\clearpage
\section{Societal Impact}\label{aox:societal}

Our work is mainly theoretical, and aimed at better understanding RLHF  with heterogeneous feedback, with principles, algorithms, and analyses. As such, we do not anticipate any direct positive or negative societal impact from this research.

\section{Limitations}\label{sec:conclude}
Our works provided overall theoretical analysis and experimental validation. However, due to the computational issue, we experimented on the 6B and 8B models, and also we did not calculate the penalty for the pessimism in our Algorithms. 
\section{Table of Notation}

\begin{table}[h]
\begin{tabular}{|c|c|}
\hline
Notation                                                                        & Definition                                                                                                                                                                                                                                                 \\ \hline
N                                                                               & Number of Individuals                                                                                                                                                                                                                                      \\ \hline
$\mathcal{S}$                                                                   & State Space                                                                                                                                                                                                                                                \\ \hline
$\mathcal{A}$                                                                   & Action Set                                                                                                                                                                                                                                                 \\ \hline
$H$                                                                             & Horizon Length                                                                                                                                                                                                                                             \\ \hline
$P_h$                                                                           & Transition Probability at Horizon $h$                                                                                                                                                                                                                      \\ \hline
$\br$                                                                           & Reward                                                                                                                                                                                                                                                     \\ \hline
$\mathcal{T}$                                                                   & Trajectory Set                                                                                                                                                                                                                                             \\ \hline
$\tau$                                                                          & Trajectory                                                                                                                                                                                                                                                 \\ \hline
$J(\pi;r_i)$                                                                    & $\EE_{\tau,\pi}\sbr{r_i(\tau)}$                                                                                                                                                                                                                            \\ \hline
$d_{\pi}(\tau)$                                                                 & Occupancy Measure: $\mathbb{P}_{\pi}(\tau)$                                                                                                                                                                                                                \\ \hline
$\Phi\colon \RR\to [0, 1]$                                                      & Strongly Convex Function Mapping Reward to Preference                                                                                                                                                                                                      \\ \hline
$\sigma(x)$                                                                     & Sigmoid Function: $\frac{e^x}{1+e^x}$                                                                                                                                                                                                                      \\ \hline
$P_{r_i}(o = 0\mid \tau_0, \tau_1)$                                          & $\Phi(r_i (\tau_0) - r_i (\tau_1))$                                                                                                                                                                                                                        \\ \hline
$\psi_\omega: \RR^d \to \RR^k$                                                  & Representation Function                                                                                                                                                                                                                                    \\ \hline
$\Psi$                                                                          & $\{\psi_\omega \mid \omega \in \Omega \}$                                                                                                                                                                                                                  \\ \hline
$\mathcal{G}_{\br}$                                                             & \begin{tabular}[c]{@{}c@{}}Set of Reward Functions: \\ $\{ (\langle \psi_\omega(\phi(\cdot)), \theta_i \rangle)_{i \in [N]} \mid \psi_{\omega} \in \Psi, \theta_i \in \RR^k \text{ and } \norm{\theta_i}_2 \leq B\text{ for all }i \in [N]\}$\end{tabular} \\ \hline
$\mathcal{N}_{\mathcal{G}_{\br}}(\epsilon)$                                     & Bracket Number of $\mathcal{G}_{\br}$ Associated with $\epsilon$                                                                                                                                                                                            \\ \hline
$r_{\omega, \theta_j}(\cdot)$                                                & $\langle \psi_\omega(\phi(\cdot)), \theta_j \rangle$                                                                                                                                                                                                       \\ \hline
$\br_{\omega, \btheta}(\cdot)$                                                  & $(r_{ \omega, \theta_1} (\cdot), \cdots, r_{\omega, \theta_N}(\cdot)) \in \RR^N$                                                                                                                                                                      \\ \hline
$r_i^\star(\cdot) $                                                           & Ground-truth Reward: $\langle \psi^\star (\phi(\cdot)), \theta_i^\star \rangle$                                                                                                                                                                            \\ \hline
$\psi^\star(=\psi_{\omega^\star})$                                              & Ground-truth Representation Function                                                                                                                                                                                                                       \\ \hline
$R_{\max}$                                                      & $-R_{\max} \leq r^\star_i(\tau) \leq R_{\max}$                                                                                                                                                                                                             \\ \hline
$\hat{\cD} $                                                                    & $\cup_{i \in [N]} \hat{\cD}_i$                                                                                                                                                                                                                             \\ \hline
$\hat{\cD}_i $                                                                  & $ \{(o_i^{(j)}, \tau_{i, 0}^{(j)}, \tau_{i,1}^{(j)})_{j \in [N_{p}]}\}$                                                                                                                                                                                    \\ \hline
$N_p$                                                                           & $N_p=|\hat{\cD}_1|=|\hat{\cD}_2|=...=|\hat{\cD}_N| $                                                                                                                                                                                                                                           \\ \hline
$C_r\left(\mathcal{G}_{\br}, \pi_{\text {tar }}, \mu_{\text {ref }}, i \right)$ & Defined in \Cref{def:concentrability-coef}                                                                                                                                                                                              \\ \hline
$C_{\br}'(\cG_{\br},\pi^\star,\mu_{1}, i)$ & Defined in \Cref{eq:define-Cr-prime}                                                                                                                                                                                              \\ \hline

$\text{Agg}_{\alpha}(\br)$                                                      & Defined in \Cref{eqn:agg-reward-new}                                                                                                                                                                                                      \\ \hline
$\text{Agg-p}_{\alpha}(\bp)(a)$                                                 & Defined in \Cref{eqn:agg-pref}                                                                                                                                                                                                            \\ \hline
\end{tabular}
\end{table}

\section{Deferred Definition}
\label{sec:defdef}

\paragraph{Bracketing Number. }
We modify and adopt the definition of the bracketing number of preferences introduced by \citep{zhan2023provable}, with some adjustments. Consider $\mathcal{G}_{\br}$ as the class of functions representing sets of reward vectors, where each reward vector is denoted by $(r_i)_{i \in [N]}$. 
Assume $g_1$ and $g_2$ maps $(\tau_0, \tau_1) \in \cT \times \cT$ to $2N$-dimensional vectors. 
A pair $(g_1, g_2)$ constitutes an $\epsilon$-bracket if for every pair of trajectories $(\tau_0, \tau_1)$ and for each $i \in [N]$, it holds that $g_1(\cdot \mid \tau_0, \tau_1) \leq g_2\left(\cdot \mid \tau_0, \tau_1\right)$ and $\norm{g_1\left(\cdot \mid \tau_0, \tau_1\right)-g_2\left(\cdot \mid \tau_0, \tau_1\right)}_1 \leq \epsilon$. The $\epsilon$-bracketing number of $\mathcal{G}_{\br}$, denoted by $\mathcal{N}_{\mathcal{G}_{\br}}(\epsilon)$, is defined as the minimum number of $\epsilon$-brackets \cpedit{$\left(g_{b, 1}, g_{b, 2}\right)_{b \in [\cN_{\cG_{\br}(\epsilon)}]}$} required such that for any reward vector $\br \in \mathcal{G}_{\br}$, there exists at least one bracket $b \in [\cN_{\cG_{\br}(\epsilon)}]$ such that for all pairs of trajectories $(\tau_0, \tau_1)$, $g_{b, 1}(\cdot \mid \tau_0, \tau_1) \leq P_{\br}(\cdot \mid \tau_0, \tau_1) \leq g_{b, 2}(\cdot \mid \tau_0, \tau_1)$ holds. 

\paragraph{Concentrability Coefficient for a Reward Scalar Class}
\label{ssec:concnetrability}
This definition is exactly the same with the concentrability coefficient of preference as outlined by \citep{zhan2023provable}.
\begin{definition}[\cite{zhan2023provable}]
The concentrability coefficient, with a reward vector class $\mathcal{G}_{r}$, a target policy $\pi_{\text{tar}}$ (which policy to compete with (potentially optimal policy $\pi^\star$)), and a reference policy $\mu_{\text{ref}}$, is defined as follows:
$$
C_r\left(\mathcal{G}_{r}, \pi_{\text {tar }}, \mu_{\text {ref }}\right):=\max \left\{0, \sup _{r \in \mathcal{G}_r} \frac{\mathbb{E}_{\tau_0 \sim \pi_{\text {tar }}, \tau_1 \sim \mu_{\text {ref }}}\left[r^{\star}\left(\tau_0\right)-r^{\star}\left(\tau_1\right)-r\left(\tau_0\right)+r\left(\tau_1\right)\right]}{\sqrt{\mathbb{E}_{\tau_0 \sim \mu_0, \tau_1 \sim \mu_1}\left[\left|r^{\star}\left(\tau_0\right)-r^{\star}\left(\tau_1\right)-r\left(\tau_0\right)+r\left(\tau_1\right)\right|^2\right]}}\right\}.
$$ 
\end{definition}

\paragraph{Direct Preference Optimization (DPO) \citep{rafailov2024direct}.} \label{ssec:dpo}
{Consider the case with Markovian reward and policy, i.e., the reward $r:\cS \times \cA \to \RR$ is a function of state $s$ and action $a$, and the policy $\pi:\cS \to \Delta(\cA)$ is also depending only on the state $s$. Also, assume that we compare actions for each state rather than the whole trajectories.} In the fine-tuning phase using RL, when {KL-regularization with the reference policy $\pi^{\text{old}}$ is employed}, the optimal policy  is given by:
\[
\pi(a \mid s) = \frac{1}{Z(s)} \pi^{\text{old}}(a \mid s) \exp\left(\frac{r(s, a)}{\beta}\right),
\]
where $Z(s)$ serves as a normalization factor that is independent of the answer $a$, and $\beta$ represents the coefficient for KL regularization. Integrating the BTL model into this framework yields:
\[
\pi^{\text{RLHF}} = \underset{\pi}{\arg\min} -\mathbb{E}_{(s, a_0) \succ (s, a_1)}\left[\log \sigma\left(\beta \log \frac{\pi(a_0 \mid s)}{\pi^{\text{old}}(a_0 \mid s)} - \beta \log \frac{\pi(a_1 \mid s)}{\pi^{\text{old}}(a_1 \mid s)}\right)\right],
\]
where $\sigma$ denotes the Sigmoid function \citep{rafailov2024direct}. This formulation bypasses the step of explicitly estimating the reward function. 

\neurips{
\section{Related Work}

}
\newpage
\section{Deferred Pseudocode of Algorithms}
\label{alg:deferred-alg-pseudo}
\neurips{\begin{algorithm}[!h]
	\caption{Personalized RLHF via Representation Learning \label{alg:personal}}
	\begin{algorithmic}
    \STATE \textbf{Input:} Dataset $\hat{\cD}=\cup_{i \in [N]} \hat{\cD}_i$ where $\hat{\cD}_i = \{(o_i^{(j)}, \tau_{i, 0}^{(j)}, \tau_{i,1}^{(j)})_{j \in [N_{p}]}\}$ is the preference dataset for the $i$th individual.    
    \STATE Estimate $\omega^\star$ and $\btheta^\star$ by 
    \[
    (\hat{\omega},\hat{\btheta})\leftarrow \argmax_{\omega \in \Omega, \norm{\theta_i}_2 \leq B \text{ for all } i \in [N]}\sum_{i \in [N]}\sum_{j \in [N_{p}]}  \log P_{\omega, \theta_i} (o_{i}^{(j)} \mid \tau_{i,0}^{(j)}, \tau_{i,1}^{(j)})
    \]
    \STATE Construct a confidence set of the reward function by 
    {\small
        \begin{equation}
        \begin{aligned}
           \cR'(\hat{\cD}) \leftarrow \cap_{i \in [N]}
        \biggl\{ & \br_{\omega, \btheta} \Biggiven   \frac{1}{N_p}\sum_{j \in [N_p]}\big|(r_{ \hat{\omega}, \hat{\theta}_i}(\tau_{i,0}^{(j)}) - r_{ \hat{\omega}, \hat{\theta}_i}(\tau_{i,1}^{(j)})) - (  r_{ \omega, \theta_i} (\tau_{i,0}^{(j)}) - r_{ \omega, \theta_i}(\tau_{i,1}^{(j)})) \big|^2  \leq \zeta' \biggr\}
        \end{aligned}
        \label{eqn:confidenceset-alg1-1}
    \end{equation}}
    \STATE Compute policy with respect to $\cR(\hat{\cD})$ (or $\cR'(\hat{\cD})$) for all $i \in [N]$ by
    \begin{align}
        \hat{\pi}'_i\leftarrow \argmax_{\pi \in \Pi} \min_{\br \in \cR'(\hat{\cD})} \left(J(\pi; r_i) - \EE_{\tau \sim \mu_{i, \text{ref}}}[r_i(\tau)]\right) \label{eqn:robust-alg1}
    \end{align}
    
    \STATE \textbf{Output:} $ (\hat{\omega},\hat{\btheta}, (\hat{\pi}'_i)_{i \in [N]})$.
    \end{algorithmic}
\end{algorithm}
}
\neurips{
\begin{algorithm}[!h]
	\caption{Personalized RLHF via Clustering \label{alg:cluster}}
	\begin{algorithmic}
    \STATE \textbf{Input:} Dataset $\hat{\cD}=\cup_{i \in [N]} \hat{\cD}_i$ where $\hat{\cD}_i = \{(o_i^{(j)}, \tau_{i, 0}^{(j)}, \tau_{i,1}^{(j)})_{j \in [N_{p}]}\}$ is the preference dataset for the $i$th individual, and $\hat{\omega}$ form \Cref{alg:personal}.
    \STATE Learn $\theta_{(i)}$ and the clustering map $f:[N] \to [K]$ by
    \begin{align}
        &(\hat{\theta}_{(k)})_{k \in [K]} \leftarrow \argmax_{\norm{\theta_{(k)}}_2 \leq B \text{ for all } {k \in [K]}} \sum_{i \in [N]} \max_{k \in [K]} \sum_{j \in [N_p]} \log P_{\hat{\omega}, \theta_{(k)}}(o_i^{(j)} \mid \tau_{i,0}^{(j)}, \tau_{i,1}^{(j)}) \label{eqn:learntheta(i)}
        \\
        &\hat{f}(i) \leftarrow \argmax_{k \in [K]} \sum_{j \in [N_p]} \log P_{\hat{\omega}, \hat{\theta}_{(k)}}(o_i^{(j)} \mid \tau_{i,0}^{(j)}, \tau_{i,1}^{(j)}) \text{ for all } i \in [N] \nonumber
    \end{align}
    
    \STATE For each $k \in [K]$,  
    \[\hat{\pi}_{(k)}\leftarrow \argmax_{\pi \in \Pi} \left(J(\pi; r_{\hat{\omega}, \hat{\theta}_{(k)}}) - \EE_{\tau \sim \mu_1}[r_{\hat{\omega}, \hat{\theta}_{(k)}}(\tau)]\right).
    \] 
    \STATE \textbf{Output:} $((\hat{\pi}_{(k)})_{k \in [K]}, (\hat{\theta}_{(k)})_{k \in [K]}, \hat{\omega}, \hat{f})$.
    \end{algorithmic}
\end{algorithm}}

\begin{algorithm}[!h]
	\caption{ClusterDPO: Learning $K$ clustered policies by DPO \label{alg:clusterDPO}}
	\begin{algorithmic}
    \STATE \textbf{Input:} Dataset $\hat{\cD}=\cup_{i \in [N]} \hat{\cD}_i$ where $\hat{\cD}_i = \{a_{i, 0}^{(j)}\succ a_{i,1}^{(j)}, s_i^{(j)})_{j \in [N_{p}]}\}$ is the preference dataset for the $i$th individual, $\beta$ is a parameter for DPO    
    \STATE Randomly select $K$ human users  $p_1, \dots, p_K$ and initialize $\pi_{(k)}^0$ for all $k \in [K]$ as  
    \begin{align*}
         \pi_{(k)}^0 \leftarrow \underset{\pi \in \Pi}{\arg\max}  \sum_{j \in [N_p]} \log \sigma\left(\beta \log \frac{\pi(a_{p_k, 0}^{(j)} \mid s_{p_k}^{(j)})}{\pi^{\text{old}}(a_{p_k, 0}^{(j)} \mid s_{p_k}^{(j)})} - \beta \log \frac{\pi(a_{p_k, 1}^{(j)} \mid s_{p_k}^{(j)})}{\pi^{\text{old}}(a_{p_k, 1}^{(j)} \mid s_{p_k}^{(j)})}\right)
    \end{align*}
    \STATE Randomly initialize $f^0(i)$ for $i \notin \{p_1, \dots, p_K\}$
    \FOR{$t \in [T]$}
    \STATE Randomly select $K$ human users $p_1, \dots, p_K$.
    \FOR{$i \in [N]$}
    \IF{$i \notin \{p_1, \dots, p_K\}$}
    \STATE Define $f^t(i) \leftarrow  f^{t-1}(i) $
    \ENDIF
    \ENDFOR
    \STATE Assign $f^t(p_k)$ for all $k \in [K]$ as 
    \begin{align}
         f^t(p_k) \leftarrow \underset{s \in [K]}{\arg\max}  \sum_{j \in [N_p]} \log \sigma\left(\beta \log \frac{\pi_{(s)}^{t-1}(a_{p_k, 0}^{(j)} \mid s_{p_k}^{(j)})}{\pi^{\text{old}}(a_{p_k, 0}^{(j)} \mid s_{p_k}^{(j)})} - \beta \log \frac{\pi_{(s)}^{t-1}(a_{p_k, 1}^{(j)} \mid s_{p_k}^{(j)})}{\pi^{\text{old}}(a_{p_k, 1}^{(j)} \mid s_{p_k}^{(j)})}\right) \label{eqn:alg5dpoloss}
    \end{align}
    \STATE Run a few steps of optimization to update $\pi_{(s)}^{t-1}$ for all $s \in [K]$ (for example, gradient ascent or Adam) to maximize 
    \begin{align*}
        \sum_{f(p_k) = s}\sum_{j \in [N_p]} \log \sigma\left(\beta \log \frac{\pi(a_{p_k, 0}^{(j)} \mid s_{p_k}^{(j)})}{\pi^{\text{old}}(a_{p_k, 0}^{(j)} \mid s_{p_k}^{(j)})} - \beta \log \frac{\pi(a_{p_k, 1}^{(j)} \mid s_{p_k}^{(j)})}{\pi^{\text{old}}(a_{p_k, 1}^{(j)} \mid s_{p_k}^{(j)})}\right)
    \end{align*}
    and obtain     $\pi_{(s)}^{t}$ for all $s \in [K]$.
    \ENDFOR
    \STATE Assign $f^{T+1}(i)$ for all $i\in [N]$ as 
    \begin{align*}
         f^{T+1}(i) \leftarrow \underset{s \in [K]}{\arg\max}  \sum_{j \in [N_p]} \log \sigma\left(\beta \log \frac{\pi_{(s)}^{T}(a_{i, 0}^{(j)} \mid s_{i}^{(j)})}{\pi^{\text{old}}(a_{i, 0}^{(j)} \mid s_{i}^{(j)})} - \beta \log \frac{\pi_{(s)}^T(a_{i, 1}^{(j)} \mid s_{i}^{(j)})}{\pi^{\text{old}}(a_{i, 1}^{(j)} \mid s_{i}^{(j)})}\right)
    \end{align*}
    \STATE \textbf{Output:} $({\pi}_{(k)}^T)_{k \in [K]}$ and $f^{T+1}$
    \end{algorithmic}
\end{algorithm}

\neurips{
\begin{algorithm}[!h]
	\caption{RLHF with \cpedit{Reward Aggregation}}\label{alg:aggregation}
	\begin{algorithmic}
    \STATE \textbf{Input:} Dataset $\hat{\cD}=\cup_{i \in [N]} \hat{\cD}_i$ where $\hat{\cD}_i = \{(o_i^{(j)}, \tau_{i, 0}^{(j)}, \tau_{i,1}^{(j)})_{j \in [N_{p}]}\}$ is the preference dataset for the $i$th human, $\lambda > 0$, and $\hat{\omega}$ from \Cref{alg:personal}. We also use \Cref{eqn:confidenceset-alg1-1} for constructing a confidence set of reward function $\cR'(\hat{\cD})$. 
    \STATE Compute policy with respect to $\cR'(\hat{\cD})$ for all $i \in [N]$ by
    \begin{align}
        \hat{\pi} \leftarrow \argmax_{\pi \in \Pi} \min_{\br \in \cR'(\hat{\cD})} \left(J(\pi; \text{Agg}_{\alpha}(r_1, \dots, r_N)) - \EE_{\tau \sim \mu_{\text{ref}}}[\text{Agg}_{\alpha}(r_1, \dots, r_N)(\tau)]\right). \label{eqn:robust-alg-agg-1}
    \end{align}
    \STATE \textbf{Output:} $ (\hat{\omega},\hat{\btheta}, \hat{\pi})$.
    \end{algorithmic}
\end{algorithm}
}
\newpage

\section{Deferred Proofs in \Cref{sec:personalization}}
\label{appendix:sec3}

\neurips{\subsection{Deferred Explanation of \Cref{alg:new-person} for a New Human User}
\label{appendix:alg-newhuman}
}

\neurips{ We defer this proof to \Cref{ssec:proof-of-diverse}. }

{\subsection{Expected Value Function Gap without Diversity Assumption}
Firstly, we provide an algorithm for each reward function learning without Assumptions~\ref{assum:task_diverse}, \ref{assum:psi-unique}, and \ref{assum:point_concen}.
\begin{algorithm}[!h]
	\caption{Personalized RLHF via Representation Learning - without Diversity Assumption \label{alg:personal-nodiverse}}
	\begin{algorithmic}
    \STATE \textbf{Input:} Dataset $\hat{\cD}=\cup_{i \in [N]} \hat{\cD}_i$ where $\hat{\cD}_i = \{(o_i^{(j)}, \tau_{i, 0}^{(j)}, \tau_{i,1}^{(j)})_{j \in [N_{p}]}\}$ is the preference dataset for the $i$th individual.    
    \STATE Estimate $\omega^\star$ and $\btheta^\star$ by 
    \[
    (\hat{\omega},\hat{\btheta})\leftarrow \argmax_{\omega \in \Omega, \norm{\theta_i}_2 \leq B \text{ for all } i \in [N]}\sum_{i \in [N]}\sum_{j \in [N_{p}]}  \log P_{\omega, \theta_i} (o_{i}^{(j)} \mid \tau_{i,0}^{(j)}, \tau_{i,1}^{(j)})
    \]
    \STATE Construct a confidence set of the reward function by { \small
    \begin{equation}
        \begin{aligned}
           \cR(\hat{\cD}) \leftarrow
        \biggl\{ & \br_{\omega, \btheta} \bigggiven  \sum_{i \in [N]}\sum_{j \in [N_{p}]}  \log P_{\omega, \theta_i} (o_{i}^{(j)} \mid \tau_{i,0}^{(j)}, \tau_{i,1}^{(j)})\geq  \sum_{i \in [N]}\sum_{j \in [N_{p}]}  \log P_{\hat{\omega}, \hat{\theta}_i} (o_{i}^{(j)} \mid \tau_{i,0}^{(j)}, \tau_{i,1}^{(j)})  - \zeta \biggr\} 
        \end{aligned}
        \label{eqn:confidenceset-alg1}
    \end{equation}}
    \STATE Compute policy with respect to $\cR(\hat{\cD})$ (or $\cR'(\hat{\cD})$) for all $i \in [N]$ by
    \begin{align}
        \hat{\pi}_i\leftarrow \argmax_{\pi \in \Pi} \min_{\br \in \cR(\hat{\cD})} \left(J(\pi; r_i) - \EE_{\tau \sim \mu_{i, \text{ref}}}[r_i(\tau)]\right) \label{eqn:robust-alg1}
    \end{align}
    \STATE \textbf{Output:} $ (\hat{\omega},\hat{\btheta}, (\hat{\pi}_i)_{i \in [N]})$.
    \end{algorithmic}
\end{algorithm}
{\begin{itemize}
    \item Confidence set (\Cref{eqn:confidenceset-alg1})
    for the MLE estimation as \citep{liu2022partially}, which is also used in \citep{liu2023optimistic, zhan2023provable, wang2024rlhf, zhan2022pac}, with $\zeta = C_1 \log(\cN_{\cG_{\br}}(1/(NN_p))/ \delta)$ for a constant $C_1>0$, which will be related to \Cref{thm:nodiverse}. the definition of bracketing number ($\cN_{\cG_{\br}}$) is deferred to \Cref{sec:defdef}. 
\end{itemize}
}

We will provide the expected value function gap of the output of \Cref{alg:personal-nodiverse} and the reference policy.

\begin{restatable}{theorem}{nodiverse}
    \label{thm:nodiverse}
    \emph{(Total Expected Value Function Gap).}
Suppose \Cref{assum:real} holds. For any $\delta \in (0, 1]$, with probability at least $1-\delta$, the output $(\hat{\pi}_i)_{i \in [N]}$ of \Cref{alg:personal} satisfies
\begin{align*}
        &\sum_{i \in [N]} \left(J(\pi_{i, \text{tar}}; r^\star_i) - J(\hat{\pi}_i; r^\star_i)\right) \leq \sqrt{\frac{c \kappa^2 N C_{\max}^2 \log(\cN_{\cG_{\br}}(1/NN_p)/ \delta)}{N_p }},
\end{align*}
where $C_{\max}:= \max_{i \in [N]} C_{\br}(\cG_{\br}, \pi_{i, \text{tar}}, \mu_{i, \text{ref}}, i)$ and $c>0$ is a constant. 
\end{restatable}

\begin{restatable}{corollary}{cornodiverse}
    \label{cor:nodiverse}
    \emph{(Expected Value Function Gap).}
Suppose \Cref{assum:real} holds. For any $\delta \in (0, 1]$ and all $i \in [N]$, with probability at least $1-\delta$, the output $\hat{\pi}_i$ of \Cref{alg:personal} satisfies
\begin{align*}
        &J(\pi_{i, \text{tar}}; r^\star_i) - J(\hat{\pi}_i; r^\star_i) \leq  \sqrt{\frac{c \kappa^2 C_{\br}(\cG_{\br}, \pi_{i, \text{tar}}, \mu_{i, \text{ref}}, i)^2 \log(\cN_{\cG_{\br}}(1/NN_p)/ \delta)}{N_p }},
\end{align*}
where $c>0$ is a constant.
\end{restatable}

Note that the results above do not need any assumption on $(\theta_i^\star)_{i\in[N]}$. Still, as $N_p \to \infty$, $\hat{\pi}_i$ {has comparable or better performance than the comparator policy $\pi_{i, \text{tar}}$, which approaches the optimal policy if $\pi_{i, \text{tar}}=\pi_i^\star$.}  
\textbf{We will leverage the proof of \Cref{thm:nodiverse} to prove \Cref{thm:diverse}.} To be specific, we will improve the bound for \Cref{cor:nodiverse}, as the gap of the expected value function does not decay with $N$, which is the number of human users. We defer the proofs of  \Cref{thm:nodiverse} and \Cref{cor:nodiverse} to  \Cref{ssec:proofthm1thm2}.
}

\subsection{Proof of \Cref{thm:nodiverse} and \Cref{cor:nodiverse}}
\label{ssec:proofthm1thm2}
\nodiverse*
\cornodiverse*
Before having a proof of \Cref{thm:nodiverse} and \Cref{cor:nodiverse}, we provide two general properties of MLE estimates, which is a slightly modified version of \citep{zhan2023provable} and \citep{liu2022partially}.

\begin{lemma}[(\citet{zhan2023provable}, Lemma 1, reward vector version)]
\label{lemma:mle}
For any $\delta \in (0,1]$, if $\br \in \cG_{\br}$, with dataset $\hat{\cD} = \cup_{i \in [N]} \hat{\cD}_i$ where $\hat{\cD}_i = \{(o_i^{(j)}, \tau_{i, 0}^{(j)}, \tau_{i,1}^{(j)})_{j \in [N_{p}]}\}$, $\tau_{i, 0}^{(j)} \sim \mu_0$, $\tau_{i, 1}^{(j)} \sim \mu_1$, and $o_i^{(j)} \sim P_{r^\star_i}(\cdot|\tau_0^{(j)}, \tau_1^{(j)})$, there exist $C_1 > 0$ such that 
\begin{align*}
    \sum_{i \in [N]} \sum_{j \in [N_{p}]} \log \left(\frac{P_{r_i}( o_{i}^{(j)} \mid \tau_{i, 0}^{(j)}, \tau_{i,1}^{(j)})}{P_{r^\star_i}(o_{i}^{(j)} \mid \tau_{i, 0}^{(j)}, \tau_{i,1}^{(j)})} \right) \leq C_1 \log(\cN_{\cG_{\br}}(1/(NN_p))/ \delta)
\end{align*} 
holds.
\end{lemma}

\begin{lemma}[(\citet{liu2022partially}, Proposition 14, scalar version)]
For any $\delta \in (0,1]$, with probability at least $1-\delta$, if $r \in \cG_{r}'$, with dataset $\hat{\cD} =  \{(o^{(j)}, \tau_{ 0}^{(j)}, \tau_{1}^{(j)})_{j \in [M]}\}$ where $\tau_0^{(j)} \sim \mu_0$, $\tau_1^{(j)} \sim \mu_1$, and $o^{(j)} \sim P_{r^\star}(\cdot|\tau_0^{(j)}, \tau_1^{(j)})$,  
\label{lemma:l2distance-scalar}
    \begin{align*}
        &\EE_{\mu_0, \mu_1}\left[ \norm{P_{r} ( \cdot \mid \tau_{ 0}^{(j)}, \tau_{1}^{(j)}) - P_{r^\star} ( \cdot \mid \tau_{0}^{(j)}, \tau_{1}^{(j)})}_1^2 \right] \leq \frac{C_2}{M} \left(    \sum_{j \in [M]} \log \left(\frac{P_{r^\star}( o^{(j)} \mid \tau_{ 0}^{(j)}, \tau_{1}^{(j)})}{P_{r}( o^{(j)} \mid \tau_{ 0}^{(j)}, \tau_{1}^{(j)})}\right) + \log(\cN_{\cG_{r}'}(1/M)/ \delta)\right)
    \end{align*}
holds where $C_2>0$ is a constant.
\end{lemma}

\begin{lemma}[(\citet{liu2022partially}, Proposition 14, vector version)]
For any $\delta \in (0,1]$, with probability at least $1-\delta$, if $\br \in \cG_{\br}'$, with dataset $\hat{\cD} = \cup_{i \in [N]} \hat{\cD}_i$ where $\hat{\cD}_i = \{(o_i^{(j)}, \tau_{i, 0}^{(j)}, \tau_{i,1}^{(j)})_{j \in [N_{p}]}\}$, $\tau_{i, 0}^{(j)} \sim \mu_0$, $\tau_{i, 1}^{(j)} \sim \mu_1$, and $o_i^{(j)} \sim P_{r^\star_i}(\cdot|\tau_0^{(j)}, \tau_1^{(j)})$, 
\label{lemma:l2distance}
    \begin{align*}
        &\frac{1}{N} \sum_{i \in [N]} \EE_{\mu_0, \mu_1}\left[ \norm{P_{r_i} ( \cdot \mid \tau_{ 0}^{(j)}, \tau_{1}^{(j)}) - P_{r^\star_i} ( \cdot \mid \tau_{0}^{(j)}, \tau_{1}^{(j)})}_1^2 \right] 
        \\
        &\qquad \leq \frac{C_2}{NN_p} \left(   \sum_{i \in [N]} \sum_{j \in [N_p]} \log \left(\frac{P_{r^\star_i}( o^{(j)} \mid \tau_{ 0}^{(j)}, \tau_{1}^{(j)})}{P_{r_i}( o^{(j)} \mid \tau_{ 0}^{(j)}, \tau_{1}^{(j)})}\right) + \log(\cN_{\cG_{\br}'}(1/(NN_p))/ \delta)\right)
    \end{align*}
holds where $C_2 > 0$ is a constant. 
\end{lemma}
Note that $\br^\star$ does not need to be in $\cG_{\br}'$ for the above lemmas. \Cref{lemma:mle} states that the log-likelihood $\log P_{\br}$ for a preference dataset generated by the reward model $\br^\star$ cannot exceed the log-likelihood $\log P_{\br^\star}$ for a preference dataset generated by the reward model $\br^\star$, with a gap related to the bracket number of $\cG_{\br}$. \Cref{lemma:l2distance} states that the $\ell_1$ distance between likelihood function $P_{\br^\star}$ and $P_{\br}$ for all $\br \in \cG_{\br}'$ can be bounded with the difference between log-likelihood $\log P_{\br^\star}$ and $\log P_{\br}$ for a preference dataset generated by the reward model $\br^\star$ with a gap related to the bracket number of $\cG_{\br}'$. 

\begin{proof}[Proof of \Cref{thm:nodiverse} and \Cref{cor:nodiverse}]

We define the event $\cE_1, \cE_2$ as satisfying (\Cref{lemma:mle}, \Cref{lemma:l2distance}) with $\delta \leftarrow \delta/2$, respectively, so we have $\PP(\cE_1\cap \cE_2) > 1- \delta$. We will only consider the under event $\cE_1\cap \cE_2$. Then, we can guarantee that 
\begin{align*}
    &\sum_{i \in [N]} \sum_{j \in [N_{p}]} \log P_{\hat{\omega}, \hat{\theta}_i}( o_{i}^{(j)} \mid \tau_{i, 0}^{(j)}, \tau_{i,1}^{(j)}) 
    \\
    &\qquad \leq \sum_{i \in [N]} \sum_{j \in [N_{p}]}
    \log P_{\omega^\star, \theta^\star_i}(o_i^{(j)} \mid \tau_{i, 0}^{(j)}, \tau_{i,1}^{(j)})+  C_1 \log(\cN_{\cG_{\br}}(1/(NN_p))/ \delta),
\end{align*}
which indicates that $\br^\star(= \br_{\omega^\star, \btheta^\star}) \in \cR(\hat{\cD})$. Moreover, by the definition of \Cref{eqn:confidenceset-alg1}, if $\br_{\omega, \btheta}, \br_{\omega', \btheta'} \in \cR(\hat{\cD})$,  
\begin{align*}
    &\big| \sum_{i \in [N]} \sum_{j \in [N_{p}]} \log P_{\omega, \theta_i}( o_i^{(j)} \mid \tau_{i, 0}^{(j)}, \tau_{i,1}^{(j)}) - \sum_{i \in [N]} \sum_{j \in [N_{p}]}
    \log P_{\omega', \theta_i'}(o_i^{(j)} \mid \tau_{i, 0}^{(j)}, \tau_{i,1}^{(j)})\big| \\
    &\qquad \leq C_1 \log(\cN_{\cG_{\br}}(1/(NN_p))/ \delta)
\end{align*}
holds, since $\sum_{i \in [N]} \sum_{j \in [N_{p}]} \log P_{\omega, \theta_i}( o_i^{(j)} \mid \tau_{i, 0}^{(j)}, \tau_{i,1}^{(j)})$ is bounded by $\sum_{i \in [N]} \sum_{j \in [N_{p}]} \log P_{\hat{\omega}, \hat{\theta}_i}( o_{i}^{(j)} \mid \tau_{i, 0}^{(j)}, \tau_{i,1}^{(j)})$ by definition of $\hat{\omega}, \hat{\btheta}$ if $r_{\omega, \btheta} \in \cG_{\br}$. Therefore, by \Cref{lemma:l2distance}, we have 
\begin{align*}
        &\frac{1}{N} \sum_{i \in [N]} \EE_{\mu_0, \mu_1}\left[ \norm{P_{\omega, \theta_i} ( \cdot \mid \tau_{i, 0}^{(j)}, \tau_{i,1}^{(j)}) - P_{\omega^\star, \theta^\star_i} ( \cdot \mid \tau_{i, 0}^{(j)}, \tau_{i,1}^{(j)})}_1^2 \right] 
        \\
        &\leq \frac{C_2}{NN_p} \left(   \sum_{i \in [N]} \sum_{j \in [N_{p}]} \log \left(\frac{P_{\omega^\star, \theta^\star_i}( o_i^{(j)} \mid \tau_{i, 0}^{(j)}, \tau_{i,1}^{(j)})}{P_{\omega, \theta_i}( o_i^{(j)} \mid \tau_{i, 0}^{(j)}, \tau_{i,1}^{(j)})}\right) + \log(\cN_{\cG_{\br}}(1/(NN_p))/ \delta)\right)
        \\
        &\leq \frac{C_2}{NN_p} \left( C_1 \log(\cN_{\cG_{\br}}(1/(NN_p))/ \delta) + \log(\cN_{\cG_{\br}}(1/(NN_p))/ \delta)\right)
        \\
        &= \frac{C_3}{NN_p} \log(\cN_{\cG_{\br}}(1/(NN_p))/ \delta)
\end{align*}
for any $\br_{\omega, \btheta} \in \cR(\hat{\cD})$, where $C_3 = C_2(C_1 + 1)$. Then, by the mean value theorem, for any $\br_{\omega, \btheta} \in \cR(\hat{\cD})$, we have 
\begin{equation}
\begin{aligned}
        &\frac{1}{N}\sum_{i \in [N]} \EE_{\mu_0, \mu_1} \left[ \left|(r_{\omega, \theta_i}(\tau_{i, 0}) - r_{\omega, \theta_i}(\tau_{i, 1})) - (r_{i}^\star(\tau_{i, 0}) - r_{i}^\star(\tau_{i, 1}))\right|^2 \right] 
        \\
        &\leq\frac{ \kappa^2 }{N}\sum_{i \in [N]} \EE_{\mu_0, \mu_1}\left[ \norm{P_{\omega, \btheta} ( \cdot \mid \tau_{i, 0}^{(j)}, \tau_{i,1}^{(j)}, i) - P_{\omega^\star, \btheta^\star} ( \cdot \mid \tau_{i, 0}^{(j)}, \tau_{i,1}^{(j)}, i)}_1^2 \right] 
        \\
        &\leq \frac{C_3 \kappa^2}{NN_p} \log(\cN_{\cG_{\br}}(1/(NN_p))/ \delta).
\end{aligned}
\label{eqn:reward-bound-thm1}    
\end{equation}
Now, we define for all policy $\pi$, 
    $$r_{\pi}^{i, \text{inf}} := \argmin_{\br \in \cR(\cD)} \left(J(\pi, r_i) - \EE_{\tau \sim \mu_{i, \text{ref}}}[r_i(\tau)]\right).$$
Then, we can bound the difference of the expected cumulative reward of a policy \(\pi_{i, \text{tar}}\) and $\hat{\pi}_i$ by  
\begin{equation}
    \begin{aligned}
        &J(\pi_{i, \text{tar}}; r^\star_i) - J(\hat{\pi}_i; r^\star_i) 
        \\
        &=(J(\pi_{i, \text{tar}}; r^\star_i) - \EE_{\tau \sim \mu_{i, \text{ref}}}[r^\star_i(\tau)]) - (J(\hat{\pi}_i; r^\star_i) -\EE_{\tau \sim \mu_{i, \text{ref}}}[r^\star_i(\tau)])
        \\
        &\underset{(i)}{\leq} (J(\pi_{i, \text{tar}}; r^\star_i) - \EE_{\tau \sim \mu_{i, \text{ref}}}[r^\star_i(\tau)]) 
        \\
        &\qquad  - (J(\pi_{i, \text{tar}}; r_{ \pi_{i, \text{tar}}}^{i, \text{inf}}) - \EE_{\tau \sim \mu_{i, \text{ref}}}[r_{\pi_{i, \text{tar}}}^{i, \text{inf}}(\tau)]) + (J(\hat{\pi}_{j}; r_{\hat{\pi}_i}^{i, \text{inf}}) - \EE_{\tau \sim \mu_{i, \text{ref}}}(r_{\hat{\pi}_i}^{i, \text{inf}}(\tau)))
        \\ 
        & \qquad - (J(\hat{\pi}_i; r^\star_i) -\EE_{\tau \sim \mu_{i, \text{ref}}}[r^\star_i(\tau)])
        \\
        &\underset{(ii)}{\leq} (J(\pi_{i, \text{tar}}; r^\star_i) - \EE_{\tau \sim \mu_{i, \text{ref}}}[r^\star_i(\tau)])  - (J(\pi_{i, \text{tar}}; r_{ \pi_{i, \text{tar}}}^{i, \text{inf}}) - \EE_{\tau \sim \mu_{i, \text{ref}}}[r_{\pi_{i, \text{tar}}}^{i, \text{inf}}(\tau)])
        \\
        &= \EE_{\tau_{i, 0} \sim \pi_{i, \text{tar}}, \tau_{i, 1} \sim \mu_{i, \text{ref}}}[ {(r_{i}^\star(\tau_{i, 1}) -r_{i}^\star(\tau_{i, 0})) - (r_{ \pi_{i, \text{tar}}}^{i, \text{inf}}(\tau_{i, 1}) -r_{ \pi_{i, \text{tar}}}^{i, \text{inf}}(\tau_{i, 0}))}]
        \\
        &\leq C_{\br}(\cG_{\br}, \pi_{i, \text{tar}}, \mu_{i, \text{ref}}, i) \sqrt{\EE_{\mu_0, \mu_1}\left[ \big| {(r_{i}^\star(\tau_{i, 1}) -r_{i}^\star(\tau_{i, 0})) - (r_{ \pi_{i, \text{tar}}}^{i, \text{inf}}(\tau_{i, 1}) -r_{ \pi_{i, \text{tar}}}^{i, \text{inf}}(\tau_{i, 0}))} \big|^2\right]}
    \end{aligned} 
    \label{eqn:J-individual}
\end{equation}

Here, $(i)$ holds since $\hat{\pi}_j$ is a distributional robust policy for $\cR(\hat{\cD})$ (\Cref{eqn:confidenceset-alg1}) and $(ii)$ holds due to the definition of $r_{\hat{\pi}_i}^{i, \text{inf}}$. Therefore, if we sum \Cref{eqn:J-individual} over $i \in [N]$, we have 
    \begin{align*}
        &\sum_{i \in [N]} \left(J(\pi_{i, \text{tar}}; r^\star_i) - J(\hat{\pi}_i; r^\star_i)\right) 
        \\
        &\leq  C_{\text{max}} \sum_{i \in [N]} \sqrt{\EE_{\mu_0, \mu_1}\left[ \big| {(r_{i}^\star(\tau_{i, 1}) -r_{i}^\star(\tau_{i, 0})) - (r_{ \pi_{i, \text{tar}}}^{i, \text{inf}}(\tau_{i, 1}) -r_{ \pi_{i, \text{tar}}}^{i, \text{inf}}(\tau_{i, 0}))} \big|^2\right]}
        \\
        &\leq C_{\text{max}}  \sqrt{N \sum_{i \in [N]} \EE_{\mu_0, \mu_1}\left[\left| {(r_{i}^\star(\tau_{i, 1}) -r_{i}^\star(\tau_{i, 0})) - (r_{ \pi_{i, \text{tar}}}^{i, \text{inf}}(\tau_{i, 1}) -r_{ \pi_{i, \text{tar}}}^{i, \text{inf}}(\tau_{i, 0}))}\right|^2\right]}
        \\
        &\leq C_{\text{max}}\sqrt{\frac{C_3N \kappa^2 \log(\cN_{\cG_{\br}}(1/NN_p)/ \delta)}{N_p }},
    \end{align*}
    which proves \Cref{thm:nodiverse}. Moreover, we have 
    \begin{equation*}
    \begin{aligned}
        &J(\pi_{i, \text{tar}}; r^\star_i) - J(\hat{\pi}_i; r^\star_i) 
        \\
        &\leq C_{\br}(\cG_{\br}, \pi_{i, \text{tar}}, \mu_{i, \text{ref}}, i) \sqrt{\EE_{\mu_0, \mu_1}\left[\left|{(r_{i}^\star(\tau_{i, 1}) -r_{i}^\star(\tau_{i, 0})) - (r_{ \pi_{i, \text{tar}}}^{i, \text{inf}}(\tau_{i, 1}) -r_{ \pi_{i, \text{tar}}}^{i, \text{inf}}(\tau_{i, 0}))}\right|^2\right]}
        \\
        &\leq C_{\br}(\cG_{\br}, \pi_{i, \text{tar}}, \mu_{i, \text{ref}}, i) \sqrt{\sum_{i \in [N]}\EE_{\mu_0, \mu_1}\left[\left|{(r_{i}^\star(\tau_{i, 1}) -r_{i}^\star(\tau_{i, 0})) - (r_{ \pi_{i, \text{tar}}}^{i, \text{inf}}(\tau_{i, 1}) -r_{ \pi_{i, \text{tar}}}^{i, \text{inf}}(\tau_{i, 0}))}\right|^2\right]}
       \\
        &\leq C_{\br}(\cG_{\br}, \pi_{i, \text{tar}}, \mu_{i, \text{ref}}, i) \sqrt{\frac{C_3 \kappa^2 \log(\cN_{\cG_{\br}}(1/NN_p)/ \delta)}{N_p }}
    \end{aligned} 
\end{equation*}
which proves \Cref{cor:nodiverse}. 
\end{proof}
\subsection{Discussion on \Cref{assum:psi-unique}}
\subsubsection{Comparing with \citep[Assumption 6.4]{lu2021power}}
\begin{assumption}[(\citet{lu2021power}, Assumption 6.4)]
\label{assum:psi-unique-lu}
For any representation functions $\psi,\psi'\in\Psi$ and $\epsilon>0$, if there exists $v, v' \in \RR^d$ that satisfy
\[
 \EE\|\psi(x)^\top v-\psi'(x)^\top v'\|^2\leq \epsilon  
\]
Then there exists a constant invertible matrix $P$ such that
\[
\|\psi(x)-P\psi'(x)\|^2\leq o(\epsilon/\norm{v}^2) = o(\epsilon/\norm{v'}^2).
\]
for all $x$. 
\end{assumption}

\Cref{assum:psi-unique} bears similarity to \Cref{assum:psi-unique-lu}; however, the latter is notably more stringent. For instance, consider the case where $v = v' = e_1$ without loss of generality. If it holds that $\EE \norm{\psi_1(x) - \psi'_1(x)}^2 \leq \epsilon$, then it implies $\psi \sim P\psi'$. In this context, $\psi_1$ and $\psi'_1$ represent the first coordinates of $\psi$ and $\psi'$, respectively. The assumption that similarity in the first coordinate necessitates equivalence of the entire representations ($\psi \sim P\psi'$) is a strong assumption. 

\subsubsection{Case Study (Linear Representation): $\psi_\omega(x) = \omega x$ and $\omega$ is an Orthonormal Matrix }
\label{sssec:linear}
\begin{proposition}
\label{prop:feature_unique_linear}
Assume that $\psi_\omega(\phi(\tau)) = \omega \phi(\tau)$ where $\omega$ is a ${k \times d}$ orthornormal matrix. For any representation functions $\psi_\omega,\psi_{\omega'}\in\Psi$ and $\epsilon>0$, if there exists $\{v_i\}_{i=1}^T,\{v_i'\}_{i=1}^T$, and a trajectory distribution $\mu$ that satisfy
\begin{align}
\frac{1}{T} \sum_{i \in [T]} \EE_{\tau \sim \mu}\|\psi_\omega(\phi(\tau))^\top v_i-\psi_{\omega'}(\phi(\tau))^\top v_i'\|^2\leq \epsilon    \label{eqn:prop4}
\end{align}
dand $V=[v_1,v_2,\cdots,v_T]\in\RR^{k\times T}$ satisfies $\sigma^2_k(W)\geq \Omega\left(T/k\right)$, and $\norm{v_i}_2 \leq B$ for all $i \in [T]$. If $\Sigma:= \EE_\mu[\phi(\tau)\phi(\tau)^\intercal] \succ \mathbf{O}$, then there exists a constant invertible matrix $P$ such that
\[
 \|\psi_\omega(\phi(\tau))-P\psi_{\omega'
}(\phi(\tau))\|^2\leq ck\epsilon/B
\]
where $c>0$ is a constant.
\end{proposition}
\begin{proof}
    By \Cref{eqn:prop4}, we have 
    \begin{align*}
        (\omega^\intercal V - (\omega')^\intercal V')^\intercal \Sigma (\omega^\intercal V - (\omega')^\intercal V') \leq T\epsilon,  
    \end{align*}
    where $V' = [v_1', \dots, v_T'] \in \RR^{k \times T}$. Since $\Sigma \succ \mathbf{O}$, we have 
    \begin{align*}
        \norm{\omega^\intercal V - (\omega')^\intercal V'}^2 \leq T\epsilon.   
    \end{align*}
    By \citep[Theorem 4]{yu2015useful}, there exist an orthonormal matrix $P$ such that 
    \begin{align*}
        \norm{\omega - P(\omega')^\intercal}^2 \leq ck\epsilon
    \end{align*}
    where $c>0$ is a constant,  which concludes \Cref{prop:feature_unique_linear}.
\end{proof}

\subsection{Proof of \Cref{cor:label-is-correct}}
\label{ssec:proofoflabelcorrect}
\neurips{
With \Cref{assum:task_diverse} and \Cref{assum:psi-unique}, {$\psi^\star$ and $\psi_{\omega}$ are close up to an orthonormal matrix transformation, as asserted below}:

\begin{restatable}{corollary}{labelcorrect}
    \label{cor:label-is-correct} \emph{(Closeness between $\psi^\star$ and $\psi_{\omega}$).}
Suppose Assumptions \ref{assum:real}, \ref{assum:task_diverse}, and \ref{assum:psi-unique} hold. For any $\delta \in (0,1]$, with probability at least $1-\delta$, if $\br_{\omega, \btheta} \in \cR'(\cD)$ as specified in \Cref{alg:personal}, then there exists an orthonormal matrix $P_\omega$ such that 
\begin{align*}
   \left[ \norm{\psi^\star(\phi(\tau_0)) - \psi^\star(\phi(\tau_1)) - P_\omega(\psi_\omega(\phi(\tau_0))- \psi_\omega(\phi(\tau_1)))}^2 \right] \leq  k \frac{c_{\text{rep}} \kappa^2 \log(\cN_{\cG_{\br}}(1/(NN_p))/ \delta)}{{NN_p B^2}}
\end{align*}
for all $\tau_0, \tau_1$, where $c_{\text{rep}} >0$ is a constant.
\end{restatable}}
\arxiv{\labelcorrect*}
\begin{proof}
By \Cref{eqn:reward-bound-thm1}, if we use \Cref{assum:psi-unique} with $\Theta^\star / B$, we can find an orthonormal matrix $P_{\omega}$ such that 
\begin{align*}
   \left[ \norm{\psi^\star(\phi(\tau_0)) - \psi^\star(\phi(\tau_1)) - P_\omega(\psi_\omega(\phi(\tau_0))- \psi_\omega(\phi(\tau_1)))}^2 \right] \leq  k \frac{c_{\text{rep}} \kappa^2 \log(\cN_{\cG_{\br}}(1/(NN_p))/ \delta)}{{NN_pB^2}}
\end{align*}
for all $\tau_0, \tau_1$, where $c_{\text{rep}} >0$ is a constant.
\end{proof}

\subsection{Proof of \Cref{thm:diverse}}
\label{ssec:proof-of-diverse}
\begin{lemma}
\label{lem:high-confidence-set-alg1-2}
     Suppose Assumptions \ref{assum:real}, \ref{assum:task_diverse} and \ref{assum:psi-unique} hold. For any $\delta \in (0,1]$ and $\lambda > 0$, with probability at least $1-\delta$, $\br^\star \in \cR'(\hat{\cD})$, i.e., the underlying reward functions are an element of \Cref{eqn:confidenceset-alg1-1}.
\end{lemma}
\begin{proof}
    Assume that \Cref{cor:label-is-correct} holds with  probability $1- \delta/{2}$ for $\hat{\omega}$, i.e.,  
\begin{align}
   \left[ \norm{\psi^\star(\phi(\tau_0)) - \psi^\star(\phi(\tau_1)) - P_{\hat{\omega}}(\psi_{\hat{\omega}}(\phi(\tau_0))- \psi_{\hat{\omega}}(\phi(\tau_1)))}^2 \right] \leq  k \frac{c_{\text{rep}} \kappa^2 \log(\cN_{\cG_{\br}}(1/(NN_p))/ \delta)}{{NN_p B^2} }. \label{eqn:highprob-rep}
\end{align}
We only consider the event that \Cref{eqn:highprob-rep} holds. We will use this $P_{\hat{\omega}}$ for the proof of \Cref{thm:diverse}. We will approach similarly with the proof of \citep{zhu2023principled}. Consider the following optimization problem:
\begin{align*}
    \underset{\norm{\theta}_i \leq B}{\text{maximize}} f(\theta_i) :=  \frac{1}{N_p} \sum_{j \in [N_p]} \log P_{\hat{\omega}, \theta_i}(o_{i}^{(j)} \mid  \tau_{i, 0}^{(j)}, \tau_{i, 1}^{(j)}). 
\end{align*}
Then, we have $\hat{\theta}_i = \underset{\norm{\theta}_i \leq B}{\argmax} f(\theta_i)$ and 
\begin{align*}
    \nabla f(\theta_i) &= \frac{1}{N_p} \sum_{j \in [N_p]} \biggl(\frac{\Phi'(\langle \psi_{\hat{\omega}}(\phi(\tau_{i, 0}^{(j)})) - \psi_{\hat{\omega}}(\phi(\tau_{i, 1}^{(j)})), \theta_i \rangle) }{\Phi(\psi_{\hat{\omega}}(\phi(\tau_{i, 0}^{(j)})) - \psi_{\hat{\omega}}(\phi(\tau_{i, 1}^{(j)})), \theta_i \rangle)} \pmb{1}(o_i^{(j)} = 0)  
    \\
    &\qquad - \frac{\Phi'(\langle \psi_{\hat{\omega}}(\phi(\tau_{i, 1}^{(j)})) - \psi_{\hat{\omega}}(\phi(\tau_{i, 0}^{(j)})), \theta_i \rangle) }{\Phi(\psi_{\hat{\omega}}(\phi(\tau_{i, 1}^{(j)})) - \psi_{\hat{\omega}}(\phi(\tau_{i, 0}^{(j)})), \theta_i \rangle)} \pmb{1}(o_i^{(j)} = 1)  \biggr)\left(\psi_{\hat{\omega}}(\phi(\tau_{i, 0}^{(j)})) - \psi_{\hat{\omega}}(\phi(\tau_{i, 1}^{(j)}))\right)
    \\
    \nabla^2f(\theta_i) &= \frac{1}{N_p} \sum_{j \in [N_p]} \frac{\Phi''(x_i^{(j)})\Phi(x_i^{(j)}) - \Phi'(x_i^{(j)})^2 }{\Phi(x_i^{(j)})^2} \left(\psi_{\hat{\omega}}(\phi(\tau_{i, 0}^{(j)})) - \psi_{\hat{\omega}}(\phi(\tau_{i, 1}^{(j)}))\right) \left(\psi_{\hat{\omega}}(\phi(\tau_{i, 0}^{(j)})) - \psi_{\hat{\omega}}(\phi(\tau_{i, 1}^{(j)}))\right)^\intercal
\end{align*}
where $x_i^{(j)}= \langle \psi_{\hat{\omega}}(\phi(\tau_{i, 0}^{(j)})) - \psi_{\hat{\omega}}(\phi(\tau_{i, 1}^{(j)})), \theta_i \rangle$. Here, we also define $\psi_{\hat{\omega}}(\hat{\cD}_i) \in \RR^{N_p \times k}$ such as every $j \in[N_p]$th row is $\left(\psi_{\omega}(\phi(\tau_{i, 0}^{(j)})) - \psi_{\omega}(\phi(\tau_{i, 1}^{(j)}))\right)$. 

Then, we have $$\nabla^2 f(\theta_i) \preceq -\eta \hat{\Sigma}_{\psi_{\hat{\omega}}}:= -\frac{\eta}{N_p}\sum_{j \in [N_p]} \left(\psi_{\hat{\omega}}(\phi(\tau_{i, 0}^{(j)})) - \psi_{\hat{\omega}}(\phi(\tau_{i, 1}^{(j)}))\right)\left(\psi_{\hat{\omega}}(\phi(\tau_{i, 0}^{(j)})) - \psi_{\hat{\omega}}(\phi(\tau_{i, 1}^{(j)}))\right)^\intercal$$
where $\eta := \min_{x \in [-2R_{\max},2R_{\max}]}\left(\frac{\Phi'(x)^2 - \Phi''(x)\Phi(x)}{\Phi(x)^2}\right)$. For example, if $\Phi(x) = \sigma(x)$, then $\eta =  \frac{1}{2 + \exp(-2R_{\max}) + \exp(2R_{\max})}$. 

Then, by the Taylor expansion of $f$, we have 
\begin{align*}
    f(\hat{\theta}_i) - f(P_{\hat{\omega}}^\intercal \theta_i^\star) - \langle \nabla f(P_{\hat{\omega}}^\intercal\theta_i^\star), \hat{\theta}_i - P_{\hat{\omega}}^\intercal\theta_i^\star \rangle \leq  -\frac{\eta}{2} \norm{\hat{\theta}_i - P_{\hat{\omega}}^\intercal\theta_i^\star}^2_{\hat{\Sigma}_{\psi_{\hat{\omega}}}}. 
\end{align*}
Since $\hat{\theta}_i =\underset{\norm{\theta}_i \leq B}{\argmax}f(\theta_i)$, for any $\lambda > 0$, we have 
\begin{align}
     \norm{\nabla f(P_{\hat{\omega}}^\intercal \theta_i^\star)}_{(\hat{\Sigma}_{\psi_{\hat{\omega}}} + \lambda I)^{-1}} \norm{\hat{\theta}_i - P_{\hat{\omega}}^\intercal \theta_i^\star}_{\hat{\Sigma}_{\psi_{\hat{\omega}}} + \lambda I} \geq \langle \nabla f(P_{\hat{\omega}}^\intercal \theta_i^\star), \hat{\theta}_i - P_{\hat{\omega}}^\intercal\theta_i^\star \rangle \geq  \frac{\eta}{2} \norm{\hat{\theta}_i - P_{\hat{\omega}}^\intercal\theta_i^\star}^2_{\hat{\Sigma}_{\psi_{\hat{\omega}}}}. \label{eqn-diversse-eqn3}
\end{align}

We define a random vector $V \in \RR^{N_p}$ as follows:
$$
V_j=\left\{\begin{array}{lll}
\frac{\Phi'(\langle \psi^\star(\phi(\tau_{i, 0}^{(j)})) - \psi^\star(\phi(\tau_{i, 1}^{(j)})), \theta_i^\star \rangle) }{\Phi(\psi^\star(\phi(\tau_{i, 0}^{(j)})) - \psi^\star(\phi(\tau_{i, 1}^{(j)})), \theta_i^\star \rangle)} & \text { w.p. } & {\Phi(\psi^\star(\phi(\tau_{i, 0}^{(j)})) - \psi^\star(\phi(\tau_{i, 1}^{(j)})), \theta_i^\star \rangle)} \\
 - \frac{\Phi'(\langle \psi^\star(\phi(\tau_{i, 1}^{(j)})) - \psi^\star(\phi(\tau_{i, 0}^{(j)})), \theta_i^\star \rangle) }{\Phi(\psi^\star(\phi(\tau_{i, 1}^{(j)})) - \psi^\star(\phi(\tau_{i, 0}^{(j)})), \theta_i^\star \rangle)} & \text { w.p. } & {\Phi(\psi^\star(\phi(\tau_{i, 1}^{(j)})) - \psi^\star(\phi(\tau_{i, 0}^{(j)})), \theta_i^\star \rangle)}
\end{array}\right.
$$
for all $j \in [N_p]$. Define $\xi = \max_{x \in [-2R_{\max}, 2R_{\max}]}\left|\frac{\Phi'(x)}{\Phi(x)}\right|$. If $\Phi(x) = \sigma(x)$, $\xi \leq 1.$ Then, we can verify that $\EE[V] = 0$ and $|V_j| \leq \xi$ for all $j \in [N_p]$.

Also, define $V' \in \RR^{N_p}$ as follows:
$$
V_j'=\left\{\begin{array}{lll}
\frac{\Phi'(\langle \psi_{\hat{\omega}}(\phi(\tau_{i, 0}^{(j)})) - \psi_{\hat{\omega}}(\phi(\tau_{i, 1}^{(j)})), P_{\hat{\omega}}^\intercal\theta_i^\star \rangle) }{\Phi(\psi_{\hat{\omega}}(\phi(\tau_{i, 0}^{(j)})) - \psi_{\hat{\omega}}(\phi(\tau_{i, 1}^{(j)})), P_{\hat{\omega}}^\intercal\theta_i^\star \rangle)} & \text { w.p. } & {\Phi(\psi^\star(\phi(\tau_{i, 0}^{(j)})) - \psi^\star(\phi(\tau_{i, 1}^{(j)})), \theta_i^\star \rangle)} \\
 - \frac{\Phi'(\langle \psi_{\hat{\omega}}(\phi(\tau_{i, 1}^{(j)})) - \psi_{\hat{\omega}}(\phi(\tau_{i, 0}^{(j)})), P_{\hat{\omega}}^\intercal\theta_i^\star \rangle) }{\Phi(\psi_{\hat{\omega}}(\phi(\tau_{i, 1}^{(j)})) - \psi_{\hat{\omega}}(\phi(\tau_{i, 0}^{(j)})), P_{\hat{\omega}}^\intercal\theta_i^\star \rangle)} & \text { w.p. } & {\Phi(\psi^\star(\phi(\tau_{i, 1}^{(j)})) - \psi^\star(\phi(\tau_{i, 0}^{(j)})), \theta_i^\star \rangle)}
\end{array}\right.
$$
for all $j \in [N_p]$.   $\nabla f(P_{\hat{\omega}}^\intercal\theta_i^\star)$ can be written as 
\begin{align*}
    \nabla f(P_{\hat{\omega}}^\intercal\theta_i^\star) &= \frac{1}{N_p} \psi_{\hat{\omega}}(\hat{\cD}_i)^\intercal V_i' =  \frac{1}{N_p} \psi_{\hat{\omega}}(\hat{\cD}_i)^\intercal V_i +  \frac{1}{N_p} \psi_{\hat{\omega}}(\hat{\cD}_i)^\intercal (V_i' - V_i). 
\end{align*}
Therefore, we can bound $\norm{\nabla f(P_{\hat{\omega}}^\intercal\theta_i^\star)}_{(\hat{\Sigma}_{\psi_{\hat{\omega}}} + \lambda I)^{-1}}$ by
\begin{align*}
    \norm{\nabla f(P_{\hat{\omega}}^\intercal\theta_i^\star)}_{(\hat{\Sigma}_{\psi_{\hat{\omega}}} + \lambda I)^{-1}} &\leq   \underbrace{\norm{\frac{1}{N_p} \psi_{\hat{\omega}}(\hat{\cD}_i)^\intercal V_i}_{(\hat{\Sigma}_{\psi_{\hat{\omega}}} + \lambda I)^{-1}}}_{(i)} +  \underbrace{\norm{\frac{1}{N_p} \psi_{\hat{\omega}}(\hat{\cD}_i)^\intercal (V_i' - V_i)}_{(\hat{\Sigma}_{\psi_{\hat{\omega}}} + \lambda I)^{-1}}}_{(ii)}. 
\end{align*}

\noindent{\color{blue}\textbf{Step 1: Bounding (i).}}
\\
Define $M = \frac{1}{N_p^2} \psi_{\hat{\omega}}(\hat{\cD}_i) (\hat{\Sigma}_{\psi_{\hat{\omega}}} + \lambda I)^{-1} \psi_{\hat{\omega}}(\hat{\cD}_i)^\intercal$, then we have 
\begin{align*}
    \norm{\frac{1}{N_p} \psi_{\hat{\omega}}(\hat{\cD}_i)^\intercal V_i}_{(\hat{\Sigma}_{\psi_{\hat{\omega}}} + \lambda I)^{-1}} = V^\intercal M V.
\end{align*}
We can check 
$$
\begin{aligned}
\operatorname{Tr}(M) \leq \frac{k}{N_p}, \qquad \operatorname{Tr}\left(M^2\right)  \leq \frac{k}{N_p^2}, \qquad 
\|M\|_{F} & =\sigma_{1}(M) \leq \frac{1}{N_p}
\end{aligned}
$$
in the same way with \citep[Page 19]{zhu2023principled}. Therefore, as $V$'s components are bounded, independent, and $\EE V = \pmb{0}$, we can use Bernstein's inequality in quadratic form (for example, \cite[Theorem 2.1]{hsu2012tail} and  \citep[Page 19]{zhu2023principled}), so we have 
\begin{align}
    \norm{\frac{1}{N_p} \psi_{\hat{\omega}}(\hat{\cD}_i)^\intercal V_i}_{(\hat{\Sigma}_{\psi_{\hat{\omega}}} + \lambda I)^{-1}} \leq \xi C_4 \sqrt{\frac{k + \log(N/\delta)}{N_p}} \label{eqn-diverser-eqn1}
\end{align}
for a constant $C_4 > 0$ with probability at least $1- \delta/(2N)$. 

\noindent{\color{blue}\textbf{Step 2: Bounding (ii).}}
\\
We have $\big| \frac{\Phi'(x)}{\Phi(x)} - \frac{\Phi'(y)}{\Phi(y)} \big| \leq \xi |x-y| $
by the mean value theorem if $x, y \in [-2R_{\max}, 2R_{\max}]$, so
\begin{align*}
|V_i - V_i'| &\leq \max_{\tau_0, \tau_1} \xi |\langle (\psi^\star(\phi(\tau_0))  - \psi^\star(\phi(\tau_1))) - (P_{\hat{\omega}} \psi_{\hat{\omega}}(\phi(\tau_0))- P_{\hat{\omega}} \psi_{\hat{\omega}}(\phi(\tau_1))), \theta_i^\star \rangle |
\\
&\leq \xi \sqrt{k \frac{c_{\text{rep}} \kappa^2 \log(\cN_{\cG_{\br}}(1/(NN_p))/ \delta)}{{NN_p} }}. 
\end{align*}
Therefore, we have 
\begin{equation}
\begin{aligned}
    {\norm{\frac{1}{N_p} \psi_{\hat{\omega}}(\hat{\cD}_i)^\intercal (V_i' - V_i)}_{(\hat{\Sigma}_{\psi_{\hat{\omega}}} + \lambda I)^{-1}}} \leq  \frac{\xi C_5}{\sqrt{N_p}}\sqrt{k \frac{ \kappa^2 \log(\cN_{\cG_{\br}}(1/(NN_p))/ \delta)}{{NN_p}}} 
\end{aligned}
    \label{eqn-diverser-eqn2}
\end{equation}
where $C_5 >0 $ is a constant. 

\noindent{\color{blue}\textbf{Step 3: Combining (i) and (ii).}}

Combining \Cref{eqn-diverser-eqn1} and \Cref{eqn-diverser-eqn2}, we have 
\begin{align*}
    \norm{\nabla f(P_{\hat{\omega}}^\intercal\theta_i^\star)}_{(\hat{\Sigma}_{\psi_{\hat{\omega}}} + \lambda I)^{-1}} &\leq   \frac{\xi C_5}{\sqrt{N_p}}\sqrt{k \frac{ \kappa^2 \log(\cN_{\cG_{\br}}(1/(NN_p))/ \delta)}{{NN_p}}} + \xi C_4 \sqrt{\frac{k + \log(N/\delta)}{N_p}}
    \\
    &\leq C_6 \sqrt{k \frac{ \xi^2\kappa^2 \log(\cN_{\cG_{\br}}(1/(NN_p))/ \delta)}{{NN_p}} + \frac{\xi^2(k + \log(N/\delta))}{N_p}} 
\end{align*}
for a constant $C_6 > 0$ with probability at least $1-\delta/N$ and \Cref{eqn-diversse-eqn3} provides 
\begin{align*}
     \norm{\hat{\theta}_i - P_{\hat{\omega}}^\intercal\theta_i^\star}_{\hat{\Sigma}_{\psi_{\hat{\omega}}}} \leq 
     C_7 \sqrt{k \frac{ \xi^2\kappa^2 \log(\cN_{\cG_{\br}}(1/(NN_p))/ \delta)}{{\eta^2 NN_p}} + \frac{\xi^2(k + \log(N/\delta))}{\eta^2 N_p} + \lambda B^2}, 
\end{align*}
which is equivalent to 
\begin{align*}
    \frac{1}{N_p}\sum_{j \in [N_p]}& \big|\langle(\psi_{\hat{\omega}}( \phi(\tau_{i, 0}^{(j)})) -\psi_{\hat{\omega}}( \phi(\tau_{i, 1}^{(j)}))), \hat{\theta}_i - P^\intercal_{\hat{\omega}}\theta_i^\star\rangle\big|^2
    \\
    &\leq  C_7^2  \left(k \frac{ \xi^2\kappa^2 \log(\cN_{\cG_{\br}}(1/(NN_p))/ \delta)}{{\eta^2 NN_p}} + \frac{\xi^2(k + \log(N/\delta))}{\eta^2 N_p} + \lambda B^2\right),
\end{align*}
with probability at least $1-\delta/N$.

Now, we will bound $\frac{1}{N_p}\sum_{j \in [N_p]}\big|(r_{\hat{\omega}, \hat{\theta}_i}(\tau_{i,0}^{(j)}) - r_{\hat{\omega}, \hat{\theta}_i}(\tau_{i,1}^{(j)})) - (  r_i^\star(\tau_{i,0}^{(j)}) - r_i^\star(\tau_{i,1}^{(j)})) \big|^2$:
\begin{align*}
    &\frac{1}{N_p}\sum_{j \in [N_p]}\big|(r_{\hat{\omega}, \hat{\theta}_i}(\tau_{i,0}^{(j)}) - r_{\hat{\omega}, \hat{\theta}_i}(\tau_{i,1}^{(j)})) - (  r_i^\star(\tau_{i,0}^{(j)}) - r_i^\star(\tau_{i,1}^{(j)})) \big|^2 
    \\&
    = \frac{1}{N_p} \sum_{j \in [N_p]}  \big|\langle\psi_{\hat{\omega}}( \phi(\tau_{i, 0}^{(j)})) -\psi_{\hat{\omega}}( \phi(\tau_{i, 1}^{(j)})), \hat{\theta_i}\rangle - \langle\psi^\star( \phi(\tau_{i, 0}^{(j)})) -\psi^\star( \phi(\tau_{i, 1}^{(j)})), \theta_i^\star\rangle \big|^2
    \\
    &\leq \frac{2}{N_p}\sum_{j \in [N_p]} \big|\langle(\psi_{\hat{\omega}}( \phi(\tau_{i, 0}^{(j)})) -\psi_{\hat{\omega}}( \phi(\tau_{i, 1}^{(j)}))), \hat{\theta}_i - P^\intercal_{\hat{\omega}}\theta_i^\star\rangle\big|^2
    \\
    &\qquad + \frac{2}{N_p} \sum_{j \in [N_p]}  \big|\langle\psi_{\hat{\omega}}( \phi(\tau_{i, 0}^{(j)})) -\psi_{\hat{\omega}}( \phi(\tau_{i, 1}^{(j)})) - P_{\hat{\omega}}( \psi^\star( \phi(\tau_{i, 0}^{(j)})) -\psi^\star( \phi(\tau_{i, 1}^{(j)}))), \theta_i^\star\rangle \big|^2 
    \\
    &\leq 2C_7 \left(k \frac{ \xi^2\kappa^2 \log(\cN_{\cG_{\br}}(1/(NN_p))/ \delta)}{{\eta^2 NN_p}} + \frac{\xi^2(k + \log(N/\delta))}{\eta^2 N_p} + \lambda B^2\right)
    \\
    &\qquad + \frac{2}{N_p} N_p  k \frac{c_{\text{rep}} \kappa^2 \log(\cN_{\cG_{\br}}(1/(NN_p))/ \delta)}{{NN_p}} 
    \\
    &\leq C_8  \left(k \frac{ \xi^2\kappa^2 \log(\cN_{\cG_{\br}}(1/(NN_p))/ \delta)}{{\eta^2 NN_p}} + \frac{\xi^2(k + \log(N/\delta))}{\eta^2 N_p} + \lambda B^2\right) 
\end{align*}
for a constant $C_8 > 0$. Combining this result for all $i \in [N]$, \Cref{lem:high-confidence-set-alg1-2} holds.
\end{proof}
\begin{lemma}
\label{lem:boundexp}
    Suppose Assumptions \ref{assum:real}, \ref{assum:task_diverse}, \ref{assum:psi-unique}, and \ref{assum:point_concen} hold. For any $\delta \in (0, 1]$, with probability at least $1- \delta$, for any $\br_{\omega, \btheta} \in \cR'(\hat{\cD})$, 
    \begin{align*}
        &\EE_{\mu_0, \mu_1} \left[ \left|(r_{{\omega}, \theta_i}(\tau_{i, 0}) - r_{\omega, \theta_i}(\tau_{i, 1})) - (r_{i}^\star(\tau_{i, 0}) - r_{i}^\star(\tau_{i, 1}))\right|^2 \right] 
        \\
        &\qquad \leq C_9  \left(k \frac{ \xi^2\kappa^2 \log(\cN_{\cG_{\br}}(1/(NN_p))/ \delta)}{{\eta^2 NN_p}} + \frac{\xi^2(k + \log(N/\delta))}{\eta^2 N_p} + \lambda B^2\right) 
    \end{align*}
    where $C_9 > 0$ is a constant. 
\end{lemma}
\begin{proof}
For any $\tau_0, \tau_1$, by \Cref{assum:point_concen}, with large $N_p \geq N_{\text{unif}}(\Psi, \mu_0, \mu_1, \delta)$, we have the analog of \Cref{eqn:reward-bound-thm1}: 
\begin{align*}
&\EE_{\mu_0, \mu_1} \left[ \left|(r_{{\omega}, \theta_i}(\tau_{i, 0}) - r_{\omega, \theta_i}(\tau_{i, 1})) - (r_{i}^\star(\tau_{i, 0}) - r_{i}^\star(\tau_{i, 1}))\right|^2 \right] 
\\
        &= \left[\begin{array}{c}\theta_i \\ -\theta_i^\star \end{array}\right]^{\intercal } \Lambda_{\phi_{\omega}, \phi_{\psi^\star}}(\mu_0, \mu_1) \left[\begin{array}{c}\theta_i \\ -\theta_i^\star\end{array}\right] \leq 1.1 \left[\begin{array}{c}\theta_i \\ -\theta_i^\star \end{array}\right]^{\intercal } \hat{\Lambda}_{\phi_{\omega}, \phi_{\psi^\star}}(\mu_0, \mu_1) \left[\begin{array}{c}\theta_i \\ -\theta_i^\star\end{array}\right] 
        \\
        & \leq 1.1 C_8  \left(k \frac{ \xi^2\kappa^2 \log(\cN_{\cG_{\br}}(1/(NN_p))/ \delta)}{{\eta^2 NN_p}} + \frac{\xi^2(k + \log(N/\delta))}{\eta^2 N_p} + \lambda B^2\right)
\end{align*}
which concludes the proof. 
\end{proof}
\diverse*
\begin{proof}
We have 
    \begin{equation*}
    \begin{aligned}
        &J(\pi_{i, \text{tar}}; r^\star_i) - J(\hat{\pi}_i'; r^\star_i) 
        \\
        &\leq C_{\br}(\cG_{\br}, \pi_{i, \text{tar}}, \mu_{i, \text{ref}}, i) \sqrt{\EE_{\mu_0, \mu_1}\left[ \big| {(r_{i}^\star(\tau_{i, 1}) -r_{i}^\star(\tau_{i, 0})) - (r_{ \pi_{i, \text{tar}}}^{i, \text{inf}}(\tau_{i, 1}) -r_{ \pi_{i, \text{tar}}}^{i, \text{inf}}(\tau_{i, 0}))}\big|^2\right]}
       \\
        &\leq \sqrt{c  C_{\br}(\cG_{\br}, \pi_{i, \text{tar}}, \mu_{i, \text{ref}}, i)^2 \left( \frac{ k\kappa^2 \log(\cN_{\cG_{\br}}(1/(NN_p))/ \delta)}{{NN_p}} + \frac{\xi^2(k + \log(N/\delta))}{\eta^2 N_p} + \lambda B^2\right)} 
    \end{aligned} 
\end{equation*}
where $c>0$ is a constant, which is similar to the proof of \Cref{thm:nodiverse}. 
\end{proof}
In the exactly same way, we can prove \Cref{thm:diverse-newmodel}, so we omit the proof of \Cref{thm:diverse-newmodel}

\subsection{Proof of \Cref{thm:lower-bound-3}}
\neurips{We present the formal version of \Cref{thm:lower-bound-3}. 
{\begin{restatable}{theorem}{lbpersonal}
\emph{(Lower Bound for the Sub-Optimality Gap of Personalization).} For any $k>6,N_p\geq Ck\Lambda^2$ and $\Lambda\geq 2$, there exists a representation function $\phi(\cdot)$ so that 
{\begin{align*}
        \min_{i \in [N]} \inf_{\hat\bpi}\sup_{Q \in {\rm CB}(\Lambda)} \left(\max_{\pi^* \in \Pi} J(\pi^*; r_{\omega, \theta_i})-J(\hat\pi;r_{\omega, \theta_i} )\right)\geq C\Lambda\cdot\sqrt{\frac{k}{N_p}},
    \end{align*}}
    where 
    \begin{align*}
        {\rm CB}(\Lambda)\coloneqq \cbr{Q\coloneqq  \left(\cbr{\mu_0,\mu_1},\{\tau_{i, 0}^{(j)},\tau_{i,1}^{(j)}\}_{i \in [N],j \in [N_p]} ,\omega,\btheta \right) \biggiven C_{\br}'(\cG_{\br},\pi^\star,\mu_{1}, i)\leq\Lambda \text{ for all }i \in [N]}
    \end{align*}    
    is the family of MDP with $N$ reward functions and $H=1$ instances, where
    \begin{align}
        C_{\br}'(\cG_{\br},\pi^\star,\mu_{1}, i)\coloneqq \max \left\{0, \sup _{\br \in \mathcal{G}_{\br}} \frac{\mathbb{E}_{\tau_0 \sim \pi^\star, \tau_1 \sim \mu_1}\left[r^{\star}_i\left(\tau_0\right)-r^{\star}_i\left(\tau_1\right)-r_i\left(\tau_0\right)+r_i\left(\tau_1\right)\right]}{\sqrt{\frac{1}{N_p} \sum_{j=1}^{N_p}\left[\left|r^{\star}_i\left(\tau_{i,0}^{(j)}\right)-r^{\star}_i\left(\tau_{i,1}^{(j)}\right)-r_i\left(\tau_{i,0}^{(j)}\right)+r_i\left(\tau_{i,1}^{(j)}\right)\right|^2\right]}}\right\}.
        \label{eq:define-Cr-prime}
    \end{align}
\end{restatable}}
}
\neurips{All results in this paper still hold for the new concentrability coefficient $C_{\br}'$. }
\arxiv{\lbpersonal*}
\begin{proof}[Proof of \Cref{thm:lower-bound-3}]
    We follow the construction in Theorem 3.10 of \citep{zhu2023principled}. 

    We will only consider $H = 1$ case. Assume $k$ can be divided by 3 without loss of generality. Let $\cS\coloneqq \cbr{0,1,...,k/3-1}$ and $\cA\coloneqq \cbr{a_1,a_2,a_3}$. Let $\psi_{\omega}(\phi(s,a_1))=e_{3s}$, $\psi_{\omega}(\phi(s,a_2))=e_{3s+1}$, and $\psi_{\omega}(\phi(s,a_3))=0$. Also, let $v_{-1}\coloneqq \cbr{1/d,1/d+\Delta,-2/d-\Delta}$ and $v_1\coloneqq \cbr{1/d+2\Delta,1/d+\Delta,-2/d-3\Delta}$. We construct $2^{|\cS|}$ instances in $CB$. Let $w\in\cbr{\pm 1}^{|\cS|}$ and $\theta_{w}\coloneqq [v_{w_1},v_{w_2},...,v_{w_{|\cS|}}]$. Let $\mu_0(s,a_1)=\frac{1-2\Lambda^2}{|\cS|},\mu_0(s,a_2)=\frac{2\Lambda^2}{|\cS|}$, and $\mu_1(s,a_3)=1$ for any $s\in\cS$. 
    
    According to \citep{zhu2023principled}, $\nbr{\Sigma_{\cD}^{-1/2} \EE_{s\sim \rho}\sbr{\psi_{\omega}(\phi(s,\pi^\star(s)))}}_2\leq \Lambda$, where $\rho$ is the uniform distribution over $\cS$. At the same time, for any $\theta_{w}$ we have $\nbr{\theta_w}_2\in\Theta_B$ when taking $B=1$, $d>6$ and $\Delta<1/(6d)$.

    Next, we will show that $C_{\br}'(\cG_{\br},\pi^\star,\mu_1, i)\leq\Lambda$. By definition, we have 
    \begin{align*}
    \nbr{\Sigma_{\cD}^{-1/2} \EE_{s\sim \rho}\sbr{\psi_{\omega}(\phi(s,\pi^\star(s)))}}_2=\nbr{\Sigma_{\cD}^{-1/2} \EE_{s\sim\rho,a\sim\pi^\star(\cdot\given s),(s',a')\sim \mu_1}\sbr{\psi_{\omega}(\phi(s,\pi^\star(s)))-\psi_{\omega}(\phi(s',a')}}_2,
    \end{align*}
    since $a'\equiv a_3$ by definition of $\mu_1$ and $\psi_{\omega}(\phi(\cdot,a_3))\equiv 0$. Then, by Section D.1. of \citep{zhan2023provable}, we have $C_{\br}'(\cG_{\br},\pi^\star,\mu_1, i)\leq \nbr{\Sigma_{\cD}^{-1/2} \EE_{s\sim \rho}\sbr{\psi_{\omega}(\phi(s,\pi^\star(s)))}}_2\leq \Lambda$. Therefore, combined with Theorem 3.10 of \citep{zhu2023principled}, we finished the proof.
\end{proof}
\section{Proof of \Cref{sec:k-LLM}}
\label{appendix:sec4}
\Cref{cor:label-is-correct} holds with 
probability $1- \delta/3$, so we have
\begin{align*}
       \max_{\tau_0, \tau_1}{\norm{(\psi^\star(\phi(\tau_0)) - \psi^\star(\phi(\tau_1))) - P_{\hat{\omega}}(\psi_{\hat{\omega}}(\phi(\tau_0))- \psi_{\hat{\omega}}(\phi(\tau_1)))}^2} \leq  k \frac{C_3 \kappa^2 \log(\cN_{\cG_r}(1/(NN_p))/ \delta)}{{NN_pB^2}}
\end{align*}
\neurips{\begin{lemma}[\cite{mansour2020three}]
\label{lemma:mansour}
For any $\delta \in (0, 1]$, with probability at least $1-\delta$, the output $((\hat{\pi}_{(k)})_{k \in [K]}, (\hat{\theta}_{(k)})_{k \in [K]}, \hat{\omega}, \hat{f})$ of \Cref{alg:cluster} satisfies 
\begin{align*}
    & \max_{\norm{\theta'_i} \leq B \text{ for all }i \in [N]} \sum_{i \in [N]} \sum_{j \in [N_{p, i}]} \log \left(\frac{P_{\hat{\omega}, \theta_i'}(o_i^{(j)} \mid \tau_{i, 0}^{(j)}, \tau_{i, 1}^{(j)})}{P_{\hat{\omega}, \hat{\theta}_{\hat{f}(i)} }(o_i^{(j)} \mid  \tau_{i, 0}^{(j)}, \tau_{i, 1}^{(j)})}\right) 
    \\
    &\leq C_{\text{cluster}} NN_p \left(\sqrt{\frac{\log (2K/\delta)}{N_p}} + \sqrt{\frac{kK\log(N_p/k)}{N_p}}+ \sum_{i \in [N]} \frac{1}{N}\texttt{disc}(\cD_i, \cC_{\hat{f}(i)}, \cG_{\psi_{\hat{\omega}}}) \right),
\end{align*} 
where $\cC_k:=\cup_{\hat{f}(i) = k} \cD_i$, $C_{\text{cluster}}>0$ is a constant, and $\cG_{\psi_{{\omega}}}:=\{r_{\omega, \theta}\mid \norm{\theta} \leq B\}$ for all $\omega \in \Omega$.
\end{lemma}}
\begin{claim}
\label{claim:diff}
For any $\delta \in (0, 1]$, with probability at least $1-\delta$, for arbitrary $\cD_i$ and $\cD_j$, the gap between label discrepency with reward function class $\cG_{\psi_{\hat{\omega}}}$ and $\cG_{\psi^\star}$ is bounded as follows:
\begin{align*}
    \big|\texttt{disc}(\bD_i, \bD_j, \cG_{\psi_{\hat{\omega}}}) -     \texttt{disc}(\bD_i, \bD_j, \cG_{\psi^\star})\big| \leq 2C_{10}\sqrt{\frac{k \xi^2 \kappa^2 \log(\cN_{\cG_{\br}}(1/(NN_p))/ \delta)}{{NN_p}} }
\end{align*}
for $i, j \in [N]$ where $C_{10}>0$ is a constant.  We recall the definition of $\cG_{\psi_w} = \{\langle \psi_w, \theta \rangle \mid \norm{\theta}_2 \leq B \}$.
\end{claim}
\begin{proof}
\begin{align*}
    &\bigg|\EE_{\bD_i} \log P_{\hat{\omega}, P^\intercal_{\hat{\omega}}\theta}(o \mid \tau_1, \tau_0) - \EE_{\bD_i} \log P_{\omega^\star, \theta}(o \mid \tau_1, \tau_0)
    \bigg| 
    \\
    &\leq \EE_{\bD_i}\bigg|\log P_{\hat{\omega}, P^\intercal_{\hat{\omega}}\theta}(o \mid \tau_1, \tau_0)  - \log P_{\omega^\star, \theta}(o \mid \tau_1, \tau_0) \bigg|  
    \\
    &\leq \xi \EE_{\bD_i} \big| \langle P_{\hat{\omega}}(\psi_{\hat{\omega}}(\phi(\tau_1)) - \psi_{\hat{\omega}}(\phi(\tau_0)))  - (\psi^\star(\phi(\tau_1)) - \psi^\star(\phi(\tau_0))) , \theta \rangle \big| 
    \\
    &\leq C_{10}\sqrt{\frac{k \xi^2 \kappa^2 \log(\cN_{\cG_r}(1/(NN_p))/ \delta)}{{NN_p}} }
\end{align*}
where $\xi := \max_{x \in [-R_{\max}, R_{\max}]}\left|\frac{\Phi'(x)}{\Phi(x)}\right|$, which is also defined in \Cref{appendix:sec3}. 
\end{proof}
\cluster*
\neurips{
We note that due to the $\sqrt{kK/N_p}$ order on the right-hand side of \Cref{lemma:mansour}, we have a slower rate in \Cref{thm:cluster} than \Cref{thm:diverse}. This gap is mainly due to the {fact that the} analysis of \Cref{lemma:mansour} should cover uniformly for arbitrary $\hat{f}$, {and also due to a difference between $\max$ and expectation of $\max$, which is bounded using McDiarmid's inequality.}}
\begin{proof}
By \Cref{claim:diff} with \Cref{lemma:mansour}, we have 
\begingroup
\allowdisplaybreaks
\begin{align*}
 &\sum_{i \in [N]} \sum_{j \in [N_{p, i}]} \log \left(\frac{P_{\omega^\star, \theta^\star_i}(o_i^{(j)}\mid \tau_{i, 0}^{(j)}, \tau_{i, 1}^{(j)})}{P_{\hat{\omega}, \hat{\theta}_{\hat{f}(i)} }( o_i^{(j)} \mid \tau_{i, 0}^{(j)}, \tau_{i, 1}^{(j)})}\right)
 \\
 &\leq\max_{\norm{\theta'_i} \leq B \text{ for all }i \in [N]} \sum_{i \in [N]} \sum_{j \in [N_{p, i}]} \log \left(\frac{P_{\omega^\star, \theta_i'}(o_i^{(j)}\mid \tau_{i, 0}^{(j)}, \tau_{i, 1}^{(j)})}{P_{\hat{\omega}, \hat{\theta}_{\hat{f}(i)} }( o_i^{(j)} \mid \tau_{i, 0}^{(j)}, \tau_{i, 1}^{(j)})}\right)
 \\
    &\underset{(i)}{\leq} \max_{\norm{\theta'_i} \leq B \text{ for all }i \in [N]} \sum_{i \in [N]} \sum_{j \in [N_{p, i}]} \log \left(\frac{P_{\hat{\omega}, \theta'_i}(o_i^{(j)}\mid \tau_{i, 0}^{(j)}, \tau_{i, 1}^{(j)})}{P_{\hat{\omega}, \hat{\theta}_{\hat{f}(i)} }( o_i^{(j)} \mid \tau_{i, 0}^{(j)}, \tau_{i, 1}^{(j)})}\right) 
    \\
    &\qquad +  NN_p C_{10}\sqrt{\frac{k \xi^2 \kappa^2 \log(\cN_{\cG_r}(1/(NN_p))/ \delta)}{{NN_p}} }
    \\
    &\leq C_{\text{cluster}} NN_p \biggl(\sqrt{\frac{\log (2K/\delta)}{N_p}} + \sqrt{\frac{kK\log(N_p/k)}{N_p}} + \sqrt{\frac{k \xi^2 \kappa^2 \log(\cN_{\cG_r}(1/(NN_p))/ \delta)}{{NN_p}} }
    \\
    &\qquad \qquad \qquad \qquad + \sum_{i \in [N]} \frac{1}{N}\texttt{disc}(\cD_i, \cC_{\hat{f}(i)}, \cG_{\psi_{\hat{w}}})\biggr)
        \\
    & \leq  C_{11} NN_p \biggl(\sqrt{\frac{\log (2K/\delta)}{N_p}} + \sqrt{\frac{kK\log(N_p/k)}{N_p}}
 +\sqrt{\frac{k \xi^2 \kappa^2 \log(\cN_{\cG_r}(1/(NN_p))/ \delta)}{{NN_p}} }
 \\
 &\qquad\qquad\qquad + \sum_{i \in [N]} \frac{1}{N}\texttt{disc}(\cD_i, \cC_{\hat{f}(i)}, \cG_{\psi^\star})) \biggr),
\end{align*} 
\endgroup
where $\hat{\omega}$ is a learned parameter from the representation learning, and $C_{11}>0$ is a constant. Here, $(i)$ came from the same reason with \Cref{claim:diff}.  Therefore, by \Cref{lemma:l2distance-scalar}, we have 
\begin{align*}
    &\EE_{\mu_0, \mu_1}\left[ \norm{P_{\hat{\omega}, \hat{\theta}_{f(i)}} ( \cdot \mid \tau_{i, 0}^{(j)}, \tau_{i, 1}^{(j)}) - P_{w^\star, \theta^\star} ( \cdot \mid \tau_{i, 0}^{(j)}, \tau_{i, 1}^{(j)})}_1^2 \right] 
    \\
    & \leq  C_{11} \biggl(\sqrt{\frac{\log (2K/\delta)}{N_p}} + \sqrt{\frac{kK\log(N_p/k)}{N_p}}
 +\sqrt{\frac{k \xi^2 \kappa^2 \log(\cN_{\cG_r}(1/(NN_p))/ \delta)}{{NN_p}} }
 \\
 &\qquad\qquad\qquad + \sum_{i \in [N]} \frac{1}{N}\texttt{disc}(\cD_i, \cC_{\hat{f}(i)}, \cG_{\psi^\star}))  + \frac{\log(\cN_{\cG_{\psi_{\hat{\omega}}}}(1/NN_p)/ \delta)}{NN_p}\biggr). 
\end{align*}
Here, we used $\cN_{\cG_{\psi_{\hat{\omega}}}}(1/NN_p) = \cN_{\cG_{\psi^\star}}(1/NN_p)$.
Now, we get the similar bound with \Cref{eqn:reward-bound-thm1}:
\begin{equation*}
\begin{aligned}
&\frac{1}{N} \sum_{i \in [N]} \EE_{\cD_i} \left[ \left|(r_{\hat{\omega}, \hat{\theta}_{\hat{f}(i)}}(\tau_{i, 0}) - r_{\hat{\omega}, \hat{\theta}_{\hat{f}(i)}}(\tau_{i, 1})) - (r_{i}^\star(\tau_{i, 0}) - r_{i}^\star(\tau_{i, 1}))\right|^2 \right] 
\\\ &\leq C_{11} \kappa^2 \biggl(\sqrt{\frac{\log (2K/\delta)}{N_p}} + \sqrt{\frac{kK\log(N_p/k)}{N_p}}
 +\sqrt{\frac{k \xi^2 \kappa^2 \log(\cN_{\cG_r}(1/(NN_p))/ \delta)}{{NN_p}} }
 \\
 &\qquad\qquad\qquad + \sum_{i \in [N]} \frac{1}{N}\texttt{disc}(\cD_i, \cC_{\hat{f}(i)}, \cG_{\psi^\star}))  + \frac{\log(\cN_{\cG_{\psi^\star}}(1/NN_p)/ \delta)}{NN_p}\biggr).
\end{aligned}
\end{equation*}
Lastly, we use the following: 
\begin{align*}
        &J(\pi_{i, \text{tar}}; r^\star_i) - J(\hat{\pi}_i; r^\star_i) 
        \\
        &=(J(\pi_{i, \text{tar}}; r^\star_i) - \EE_{\tau \sim \mu_{i, \text{ref}}}(r^\star_i(\tau))) - (J(\hat{\pi}_i; r^\star_i) -\EE_{\tau \sim \mu_{i, \text{ref}}}(r^\star_i(\tau)))
        \\
        &=(J(\pi_{i, \text{tar}}; r^\star_i) - \EE_{\tau \sim \mu_{i, \text{ref}}}(r^\star_i(\tau))) - (J({\pi}_{i, \text{tar}}; \hat{r}_i) -\EE_{\tau \sim \mu_{i, \text{ref}}}(\hat{r}_i(\tau))) 
        \\
        &\qquad + (J({\pi}_{i, \text{tar}}; \hat{r}_i) -\EE_{\tau \sim \mu_{i, \text{ref}}}(\hat{r}_i(\tau))) - (J({\hat{\pi}}_{j}; \hat{r}_i) -\EE_{\tau \sim \mu_{i, \text{ref}}}(\hat{r}_i(\tau)))
        \\
        &\qquad + (J({\hat{\pi}}_{i}; \hat{r}_i) -\EE_{\tau \sim \mu_{i, \text{ref}}}(\hat{r}_i(\tau))) - (J(\hat{\pi}_i; r^\star_i) -\EE_{\tau \sim \mu_{i, \text{ref}}}(r^\star_i(\tau)))
        \\
        &\leq 2 C_{\max}' \sqrt{\EE_{\mu_0, \mu_1}\left[ \big|{(r_{i}^\star(\tau_{i, 0}) -r_{i}^\star(\tau_{i, 1})) - (\hat{r}_{i}(\tau_{i, 0}) -\hat{r}_{i}(\tau_{i, 1}))}\big|^2\right]}
    \end{align*}
    where the last inequality came from the fact that $\hat{\pi}_i$ is the best policy with respect to $\hat{r}_{f(i)}$. Therefore, summing the above relationship with $i \in [N]$ provides 
    \begin{align*}
        &\sum_{i \in [N]}\left(J(\pi_{i, \text{tar}}; r^\star_i) - J(\hat{\pi}_i; r^\star_i) \right)
        \\
        &\leq C_{12} N \kappa \Biggl({\frac{\log (2K/\delta)}{N_p}} + {\frac{kK\log(N_p/k)}{N_p}}
 +{\frac{k \xi^2 \kappa^2 \log(\cN_{\cG_r}(1/(NN_p))/ \delta)}{{NN_p}} }
 \\
 &\qquad\qquad\qquad + \left(\sum_{i \in [N]} \frac{1}{N}\texttt{disc}(\cD_i, \cC_{\hat{f}(i)}, \cG_{\psi^\star}))\right)^2  + \left(\frac{\log(\cN_{\cG_{\psi^\star}}(1/NN_p)/ \delta)}{NN_p}\right)^2\Biggr)^{1/4}.
    \end{align*}
\end{proof}
\neurips{
\begin{remark}
   {In contrast to the results in \Cref{sec:personalization}},  
    we additionally assume \( C_r(\cG_r, \pi, \mu_1, i) \leq C_{\text{max}}' \) in \Cref{thm:cluster}. To adopt a pessimistic approach, constructing a confidence set for clustered reward functions across all clusters is necessary. However, the ambiguity of which human user belongs to which cluster complicates this analysis, as pessimism would need to be applied to every potential cluster. Consequently, defining a confidence set for every possible clustering scenario is required, significantly complicating the analysis of the algorithm.
\end{remark}}
\subsection{Why do We Provide \Cref{alg:clusterDPO}?}
Given the inherent complexity of this hierarchical optimization problem, which presents more challenges than standard optimization tasks \citep{anandalingam1992hierarchical}, we propose a novel algorithm that circumvents the need for explicit reward function estimation in \Cref{alg:clusterDPO}. Our approach begins by randomly assigning each human user to a cluster. Subsequently, we reassign random human users to the cluster where the policy most effectively maximizes their empirical DPO loss (\Cref{eqn:alg5dpoloss}). Finally, we refine our solution by optimizing the DPO loss function for the selected human users within each cluster, thereby enhancing the overall policy effectiveness. 

\section{Proof of \Cref{ssec:reward-agg}}
\label{appendix:reward-agg}
\subsection{Six Pivotal Axioms for Reward Aggregation}
\label{appendix:def-of-axiom}
For the completeness of the paper, we introduce six pivotal axioms for reward aggregation \citep{moulin2004fair}. 
\begin{itemize}
    \item \textbf{Monotonicity}: For two reward vectors, \(\br = (r_1, \dots, r_N)^\top\) and \(\br'= (r_1', \dots, r_N')^\top\) such that \(r_i = r_i'\) for \(i \neq j\) and \(r_j > r_j'\) for some \(j \in [N]\), then \(\br \succ \br'\). This is related to Pareto optimality, indicating that if one vector is strictly better than another in at least one dimension and no worse in any other, it is considered superior. 
    \item \textbf{Symmetry}:  The reward aggregation function should treat all individuals equally. The outcome should not depend on the identities of the individuals but only on their rewards.
    \item \textbf{Independence of Unconcerned Agents}: If for an individual \( j \in [N] \), \( r_j = r_j' \), then the magnitude of \( r_j \) does not influence the comparison between \(\br\) and \(\br'\).
    \item \textbf{The Pigou-Dalton Transfer Principle}:     If \( r_i < r_j \) and \( r_i' + r_j = r_j' + r_i \) for a pair \( (i,j) \in [N] \times [N] \) and $r_k = r_k'$ for all $k \neq i, j\in [N]$, then \(\br' \succ \br\). This condition implies that, all else being equal, a social welfare function should favor allocations that are more equitable, reflecting a preference for balancing the rewards between individuals \(i\) and \(j\).
    \item \textbf{Translation Independence}: If $\br \succ \br'$, then $\br + c \succ \br' + c$ for $c \in \RR^N$. 
    \item \textbf{Continuity}: In the context of social choice with a continuous preference scale, continuity means that small changes in the individual preferences should not lead to abrupt changes in the collective decision. 
\end{itemize}
\Cref{eqn:agg-reward-new} and its monotonically increasing transformation is only reward aggregation that satisfying the above six axioms. In \citep{zhong2024provable}, the consider \textit{Scale Independence} rather than \textit{Translation Independence}, which is defined as follows:
\begin{itemize}
    \item \textbf{Scale Independence}: If $\br \succ \br'$, then $\lambda \cdot \br \succ \lambda \cdot \br'$ for $\lambda > 0$. 
\end{itemize}
In this case, the reward aggregations that satisfying six axioms are
\begin{align*}
\text{Agg}_{\alpha}(\br) =\left\{\begin{array}{lr}\frac{1}{N\alpha}\sum_{i \in [N]} r_i^\alpha & \alpha \neq 0  \\ \prod_{i \in [N]}r_i\ & \alpha=0 \end{array}\right.
\end{align*}
for $\alpha \in [-\infty, \infty]$. 

\neurips{\subsection{Deferred Statement of Lower Bound for the Sub-Optimality Gap of Aggregation}
\label{appendix:defstate-lower-agg}
\begin{restatable}{theorem}{lbagg}
\label{thm:lower-bound-5} \emph{(Lower Bound for the Sub-Optimality Gap of Aggregation).}
    For any $k>6,N_p\geq Ck\Lambda^2,\Lambda\geq 2$, and $\alpha \in \RR$ there exists a representation function $\phi(\cdot)$ so that 
    \begin{align*}
\inf_{\hat\bpi}\sup_{Q\in {\rm CB}(\Lambda)} \left(\max_{\pi^* \in \Pi} J(\pi^*; \text{Agg}_{\alpha}(\br_{\omega, \btheta}))-J(\hat\pi;\text{Agg}_{\alpha}(\br_{\omega, \btheta}))\right)\geq C\Lambda\cdot\sqrt{\frac{k}{N_p}},
    \end{align*}
    where 
    \begin{align*}
        {\rm CB}(\Lambda)\coloneqq \cbr{Q\coloneqq  \left(\cbr{\mu_0,\mu_1},\{\tau_{i, 0}^{(j)},\tau_{i,1}^{(j)}\}_{i \in [N], j \in [N_p]} ,\omega,\btheta \right) \biggiven C_{\br}'(\cG_{\br},\pi^\star,\mu_{1}, i)\leq\Lambda \text{ for all }i \in [N]}
    \end{align*}   
    is the family of MDP with $N$ reward functions and $H=1$ instances. $C_{\br}'$ is defined in \Cref{eq:define-Cr-prime}.
\end{restatable}}
\subsection{Proof of \Cref{thm:agg}}
\label{appendix:reward-agg-pfthm}
\agg*
\begin{proof}
    Define $C_\alpha:= \max_{x, y, z, w \in [-R_{\max}, R_{\max}]} \frac{|(\exp(\alpha x) - \exp(\alpha y)) - (\exp(\alpha z) - \exp(\alpha w))|}{\alpha|(x-y) - (z-w)|}$ for $\alpha \neq 0$ and $C_\alpha = 1$ for $\alpha = 0$. Then we know that $C_\alpha < \infty$. Now, in the same way of proof of \Cref{thm:diverse}, we have 
        \begin{equation*}
    \begin{aligned}
        &J(\pi_{\text{tar}}; \text{Agg}_{\alpha}(r_1^\star, \dots, r_N^\star)) - J(\hat{\pi}; \text{Agg}_{\alpha}(r_1^\star, \dots, r_N^\star))  
        \\
        &\leq C_{\br}(\cG_{\br}, \pi_{\text{tar}}, \mu_{\text{ref}}) \sqrt{\EE_{\mu_0, \mu_1}\left[ \big| {(
        \text{Agg}_{\alpha}(\br^\star(\tau_{1})) -\text{Agg}_{\alpha}(\br^\star(\tau_{0}))) - (\text{Agg}_{\alpha}(\br_{\pi_{\text{tar}}}^{\text{inf}}(\tau_{1})) -\text{Agg}_{\alpha}(\br_{\pi_{\text{tar}}}^{\text{inf}}(\tau_{0})))}\big|^2\right]}
       \\
       &\leq C_{\br}(\cG_{\br}, \pi_{\text{tar}}, \mu_{\text{ref}}) \sqrt{C_{\alpha}^2 \EE_{\mu_0, \mu_1}\left[ \frac{1}{N}\sum_{i \in [N]} \big| (r^\star_i(\tau_1) - r^\star_i(\tau_0))  - (r_{\pi_{\text{tar}}}^{\text{inf}}(\tau_1 ) - r_{\pi_{\text{tar}}}^{\text{inf}}(\tau_0))\big|^2\right]}
       \\
        &\leq  \sqrt{c_{\alpha}  \left( \frac{ k\kappa^2 \log(\cN_{\cG_{\br}}(1/(NN_p))/ (\delta/N))}{{NN_p^2}} + \frac{\xi^2(k + \log(N/\delta))}{\eta^2 N_p} + \lambda B^2\right)}. 
    \end{aligned} 
\end{equation*}
where the last line is from \Cref{lem:boundexp}, which conclude the proof. 
\end{proof}
\subsection{Proof of \Cref{thm:lower-bound-5}}
\lbagg*
\begin{proof}
We start with the same setting and the same instances that achieve the lower bounds with \Cref{thm:lower-bound-3}.
    Since 
    \begin{align*}
        \EE_s[\text{Agg}_{\alpha}(\br)(s, \pi^\star) - \text{Agg}_{\alpha}(\br)(s, \pi')] \geq \Omega\left(\EE_s[\sum_{i \in [N]} (r_i(s, \pi^\star) - r_i(s, \pi'))]\right) \geq \Omega\left(C\Lambda\cdot\sqrt{\frac{k}{N_p}}\right)
    \end{align*}
    We can finish the proof for all $\alpha \in \RR$. The first inequality holds by definition when $\alpha=0$. When $\alpha\not=0$, for any $i\in[N]$, we have $\exp(r_i(s, \pi^\star))-\exp(r_i(s, \pi'))\geq \exp(-R_{\max})\abr{r_i(s, \pi^\star) - r_i(s, \pi')}\geq \Omega\left(C\Lambda\cdot\sqrt{\frac{k}{N_p}}\right)$.
\end{proof}

\neurips{\subsection{Remark on Aggregation of Probabilistic Opinion (\Cref{eqn:agg-pref})}
\label{ssec:remark-prob-op}
\begin{remark}
The case where \(\alpha = 0\) is referred to as the geometric pooling function \citep{mcconway1978combination}. This function is known for preserving unanimity and not being eventwise independent, while it does satisfy external Bayesianity \citep{madansky1964externally, dietrich2016probabilistic}. External Bayesianity mandates that updating the probabilities with new information should yield consistent results regardless of whether the update occurs before or after the aggregation process \citep{genest1984characterization}.  
\end{remark}}

\subsection{Proof of \Cref{prop:agg-min-prob}}
\label{appendix:ssec:agg-equiv}
\aggequiv* 
\begin{proof}
By the PL modeling, we have 
\begin{align}
    P_i(a) = \frac{\exp(R_i(a))}{\sum_{a' \in \cA} \exp(R_i(a'))} \label{eqn:PL-modeling}.
\end{align}
We divide \Cref{eqn:PL-modeling} by $P_i(a_{\text{fix}})$, we have 
\begin{align}
    R_i(a) = \log P_i(a) - (\log P_i(s, a_{\text{fix}}) - R_i(a_{\text{fix}})):= \log  P_i( a) - C_i \label{eqn:reward-pref-relationship}
\end{align}
where $C_i:= \log P_i(a_{\text{fix}}) - R_i(a_{\text{fix}})$. Since $R_i(a)$ have upper bound as $C_i$, and we assumed that every reward $R_i(a)$ have the same upper bound, we can assume $C_i = C$ for every $i$. Therefore, plugging \Cref{eqn:reward-pref-relationship} provides the equivalence between $\text{Agg}_{\alpha}(\bR)$ and $\text{Agg-p}_{\alpha}(\bP)$. 
\end{proof}

\neurips{
\subsection{Deferred Algorithm for Human Feedback with Probabilistic Opinions}
\label{appedix:human-feedback-prob-op}
Now, we provide an algorithm that uses the feedback {in the form of} probabilistic opinions (\Cref{alg:POP-DPO}). The only difference from the DPO algorithm \citep{rafailov2024direct} is to change the deterministic answer $a_i$ to the $a_i$ sampled based on the probabilistic opinion pooling, which is in the second line in the for loop of \Cref{alg:POP-DPO}.

\begin{algorithm}[!h]
	\caption{Probabilistic Opinion Pooling  DPO (POP-DPO)\label{alg:POP-DPO}}
	\begin{algorithmic}
    \STATE \textbf{Input:} Dataset $\hat{\cD}=\cup_{i \in [N]} \hat{\cD}_i$ where $\hat{\cD}_i = \{q_i^{(j)}(s_i^{(j)}), s^{(j)}, i)\}_{j \in [N_{p}]}$ is the probabilistic opinion dataset for the $i$th individual, $q_i^{(j)} \in \Delta(\cA)$ with $|\cA| = 2$, $\beta$ is a parameter for DPO, $\alpha$ is a parameter for aggregation
    \FOR{every epoch}
    \STATE For every question $s^{(j)}$ where $j$ is in the batch, $q^{(j)}:= \text{Agg-p}_\alpha (\bq^{(j)})$. 
    \STATE Sample $a^{(j)}_0 \sim \text{Multinomial}(q^{(j)})$ and define $a^{(j)}_1$ as non-selected answer. 
    \STATE Run a few steps of optimization to update $\pi$ (for example, gradient ascent or Adam) to maximize 
    \begin{align*}
        \sum_{j \in \text{batch}} \log \sigma\left(\beta \log \frac{\pi(a_{0}^{(j)} \mid s^{(j)})}{\pi^{\text{old}}(a_{0}^{(j)} \mid s^{(j)})} - \beta \log \frac{\pi(a_{1}^{(j)} \mid s^{(j)})}{\pi^{\text{old}}(a_{1}^{(j)} \mid s^{(j)})}\right)
    \end{align*}
    \ENDFOR
    \STATE \textbf{Output:} $\pi$
    \end{algorithmic}
\end{algorithm}
}

\subsection{Relationship between KL divergence and variant of $\alpha$-Renyi divergence.}
\label{appendix:relationship}
By L'Hôpital's rule, we have 
  \begin{align*}
        \lim_{\alpha \to 1} &\frac{1}{1-\alpha}\left(1- \sum_{j \in \cA}p_{ij} \left(\frac{q_{ij}}{p_{ij}}\right)^{1-\alpha}\right)=  \sum_{j \in \cA} \lim_{\beta \to 0} \left( - p_{ij} \log\left(\frac{q_{ij}}{p_{ij}}\right) \left(\frac{q_{ij}}{p_{ij}}\right)^{\beta} \right)  = \texttt{KL}(p, q). 
    \end{align*}
\neurips{\subsection{Deferred Explanation of Mechanism Design for RLHF}
\label{appendix:setupexplanation}
{In this setup, we will first prove the existence of a cost function \( c_i : \Delta(\mathcal{A})^N \to \mathbb{R} \) for all $i \in [N]$ that induces truthful reporting of probabilistic opinions from human labelers. Here, the input of $c_i$ is the probabilistic opinion of every human labeler.}
This is also called the dominant strategy incentive-compatible (DSIC) mechanism \citep{nisan1999algorithmic,borgers2015introduction,roughgarden2010algorithmic}.  {Then, we prove that {there exists} an aggregation rule and cost function that induce DSIC,  and also maximize social welfare. We denote each human labeler's underlying (true) probabilistic opinion as $p_i\left(s^{\left(j\right)}\right)$ for each question $s^{(j)}$.} {Accounting for such cost, w}e define the \emph{utility function} of  individual $i$ for question $s^{(j)}$ as 
\begin{align*}
    u_i^{\left(j\right)}\left(p_i\left(s^{\left(j\right)}\right), \left(P_i\left(s^{\left(j\right)}\right)\right)_{i \in [N]}
\right) = -d\left(p_i\left(s^{\left(j\right)}\right),\text{Agg-p}\left(\left(P_i\left(s^{\left(j\right)}\right)\right)_{i \in [N]}\right)\right) -c_i\left(\left(P_i\left(s^{\left(j\right)}\right)\right)_{i \in [N]}\right).
\end{align*} 
Here, $d: \Delta(\cA) \times \Delta(\cA) \to \RR$ represents the distance between the underlying true probabilistic opinion and the aggregated preference.
Moreover, we define the \emph{welfare function} {of individual $i$} from {addressing question} $s^{(j)}$ as $\text{Wel}_i^{(j)}(O) = -d(p_i(s^{(j)}) , O)$ for {any} $O \in \Delta(\cA)$.

\begin{remark}[Examples of Distance Function $d$] 
\label{rmk:renyi}
    We can instantiate $d(p,q)$ as the KL-divergence. Also, we may instantiate  $d_{\alpha}(p,q) = \text{sgn}(\alpha) \frac{1}{1-\alpha}\sum_{j \in \cA}\left(1 - p_j^{\alpha}q_j^{1-\alpha}\right)$, which is a variant of \kzedit{the} $\alpha$-Renyi divergence for $\alpha \neq 0$. One can easily check that $d_{\alpha}(p, q) \geq 0$. In fact, one can also prove that $\lim_{\alpha \to 1} d_{\alpha}(p,q) = d(p,q)$ {with $d(p,q)$ being the KL-divergence} (\Cref{appendix:relationship}).
\end{remark}

\subsubsection{Mechanism and Guarantees}
\label{sssec:vcg}

We design a mechanism inspired by the Vickery-Clarke-Groves mechanism \citep{vickrey1961counterspeculation, clarke1971multipart, groves1973incentives}, as defined below. 

\begin{definition}[VCG Mechanism]
\label{def:vcg}
Assume that there are $n$ strategic agents and a finite set $X$ of outcome, and each individual $i$ has a private valuation $v_i$ for each outcome $x \in X$. \cpedit{The bidding $\bb = (b_1, \dots, b_N)^\intercal \in (\RR^{|X|})^N$ where $b_i \in \RR^{|X|}$ is bidding for all outcome of individual $i\in[N]$.} Define their utility function as $v_i(\bx(\bb)) - c_i(\bb)$, where $\bx: (\RR^{|X|})^N \to X$ is the allocation rule and $c_i: (\RR^{|X|})^N \to \RR$ is the cost function. \cpedit{The summation of welfare function of all agents}
is defined as $\text{Wel}(x) = \sum_{i \in [N]}v_i(x)$ for all $x \in X$. The goal is to design $\bx$ 
and {$(c_i)_{i\in[N]}$} functions to make a DSIC {and} welfare{-}maximizing mechanism. The following $\bx$ and $c_i$ for $i \in [N]$ is DSIC welfare maximizing mechanism:
\begin{align*}
    \bx(\bb) = \argmax_{x \in X} \sum_{i \in [N]} b_i(x), \qquad c_i(\bb) = \max_{x \in X}\sum_{j \neq i} b_j(x) - \sum_{j \neq i} b_j(\bx(\bb)) \text{ for all } i \in [N].
\end{align*}
\end{definition}

\cpedit{Unfortunately, the {classical} VCG mechanism presents certain limitations such as it cannot be solved in polynomial time in general \citep{nisan1999algorithmic, borgers2015introduction, roughgarden2010algorithmic}.}
{We here adopt certain forms of allocation rule (which corresponds to the aggregation rule in our RLHF setting) and cost functions as follows, which allow the outcome set to be a simplex (with infinitely many outcomes)}: 
\begin{align}
    \text{Agg-p}(\bP) = \argmin_{p \in \Delta(\cA) } \sum_{i \in [N]}d(\bP, p), \qquad c_i(\bP) = \sum_{j \neq i}d(P_i, \text{Agg-p}(\bP)) - \min_{p \in \Delta(\cA)}\sum_{j \neq i}d(P_i, p)
    \label{eqn:agg-p-cost}.
\end{align}

\begin{restatable}{theorem}{vcgmain}
\emph{(DSIC Welfare-Maximizing Mechanism).}
\label{thm:dsic}
    The aggregation {rule} and the cost function as in  \Cref{eqn:agg-p-cost} provide a DSIC welfare-maximizing mechanism.  
\end{restatable}

Due to the modeling, we have an advantage compared to the original VCG mechanism.  The minimization in the aggregation function can be achieved using a simple optimization method such as gradient descent, which makes our aggregation rule and cost function computation easy, which is in contrast with the original VCG mechanism.

\cpedit{Now, we connect our mechanism design with pre-defined preference aggregation function ($\text{Agg-p}_\alpha$ in \Cref{eqn:agg-pref}). \Cref{thm:aggrenyi} implies that \Cref{eqn:agg-pref} is maximizing social welfare and also we are available to construct the cost function to make human feedback truthful. 
}
\begin{restatable}{theorem}{aggprenyi}
\label{thm:aggrenyi}
    \cpedit{If we set $d$ as a variant of {the} $\alpha$-Renyi distance for $\alpha \neq 0$ (\Cref{rmk:renyi}) and define $d$ as KL-divergence for $\alpha = 0$, the DSIC welfare-maximizing aggregation rule is \Cref{eqn:agg-pref}.} \cpedit{Therefore, aggregation rule \Cref{eqn:agg-pref} is also welfare-maximizing with appropriate cost function.}
\end{restatable}

If we assume the relationship between reward and preference {follows} the PL model (\Cref{def:PL}), then \Cref{eqn:agg-reward} implies a welfare-maximizing aggregation rule, \cpedit{which connects reward aggregation and mechanism design.}  We defer all proofs for the results in \Cref{sssec:vcg} to \Cref{appendix:vcg}. 
}
\neurips{\subsubsection{Proof of \Cref{sssec:vcg}}}
\arxiv{\subsection{Proof of \Cref{sssec:vcg}}}
\label{appendix:vcg}
The proof of \Cref{thm:dsic} is exactly the same as the proof of the fact that the VCG mechanism is DSIC welfare-maximizing. The difference with the proof of the original VCG mechanism's property is the parametrization of bidding, which will be explained in this section.  
\vcgmain*
\begin{proof}
    The aggregated result space $ \Delta(\cA)$ corresponds to the output space $X$ of \Cref{def:vcg}. We can interpret the bidding part, $b_j(x)$, of \Cref{def:vcg} as  $-d(P_j, p)$. So, instead of bidding on every output without any rule, we can interpret the bidding as the minus distance function between their own probabilistic opinion and aggregated probabilistic opinion. The underlying value function therefore corresponds to $-d(p_j, p)$. This interpretation provides the same line of proof of the VCG mechanism's property.
\end{proof}
By good parametrization of the VCG mechanism, we can also achieve the computational efficiency of our cost function computation. 
\aggprenyi*

\begin{proof}
    We solve the optimization problem as follows:
    \begin{align}
        \argmin_{p \in \Delta(\cA)}\sum_{i \in [N]} d_{\alpha}(P_i,p) \label{eqn:opt-renyi}
    \end{align}
    where $d_{\alpha}(p,q) = \text{sgn}(\alpha) \frac{1}{1-\alpha}\sum_{j \in \cA}\left(1 - p_j^{\alpha}q_j^{1-\alpha}\right).$ 
    We can check that $d_{\alpha}(p, q)$ is a convex function with respect to $q$, as 
    \begin{align*}
        \frac{d^2}{dq_j^2}d_{\alpha}(p,q) = \alpha \text{sgn}(\alpha) q_j^{-\alpha - 1} \geq 0.
    \end{align*}
    Therefore, \Cref{eqn:opt-renyi} can be solved with first-order condition:
    \begin{align*}
        \sum_{i \in [N]}\left(\frac{P_{ij}}{p_j}\right)^{\alpha} = \lambda \qquad \text{for all $j \in \cA$}
    \end{align*}
    which provides \Cref{eqn:agg-pref}. 
\end{proof}

\section{Experiment Details}
\label{appendix:exp-detail}
\neurips{For the Reddit TL;DR summarization dataset, \citep{stiennon2020learning} filtered the TL;DR summarization dataset \citep{volske2017tl} to ensure quality. The Reddit TL;DR human feedback dataset is constructed with two components: \texttt{comparison} and \texttt{axes evals}. The \texttt{comparison} component contains labeled comparisons between pairs of summaries with workers identified by unique IDs, while the \texttt{axes evals} component contains ratings of summaries along three axes: accuracy, coverage, and coherence.}

\neurips{In \Cref{ssec:reward-model-performance}, we fine-tuned the personalized reward model with \Cref{alg:personal} and \Cref{alg:cluster}, without pessimism. We ranked workers based on the number of annotated comparisons in the training split of the dataset and included the top 5 workers for training. To balance the number of samples for each worker, we took the worker with the fewest samples among the top 5 as the baseline. We then randomly sampled the same number of comparisons from the other workers so that each worker had 5,373 comparison samples, resulting in a total of 26,865 samples for training. Similarly, for the validation set, we applied the same method. We randomly sampled the same number of comparisons as the worker with the fewest samples from the top 5 workers used in training. Each worker had 1,238 samples for validation, resulting in a total of 6,190 samples for validation. In \Cref{ssec:examples}, we fine-tuned the personalized reward model using \Cref{alg:aggregation}, without incorporating pessimism. We considered three types of reward functions: accuracy-reward, coverage-reward, and coherence-reward. Since this dataset is only publicly available for the validation set (with 8,585 samples) and the test set (with 6,313 samples), we used the validation set for fine-tuning the training set of our model and validated it with the samples in the test set. }

For reward model training, we used the AdamW optimizer \citep{loshchilov2018fixing} with a learning rate of 1e-6 and a batch size of 8 for 1 epoch. The learning rate was linearly warmed up from 0 to 1e-6 over 150 steps. For fine-tuning the language model with the trained reward model, we used the AdamW optimizer with a learning rate of 5e-6 and a batch size of 4. We employed Proximal Policy Optimization (PPO) \citep{schulman2017proximal} with 128 rollouts, which is the default setting in the TRLX library \citep{havrilla2023trlx}. For SFT, for the GPT-J 6B model, we initialized a personalized language model using an open-source SFT by CarperAI \citep{havrilla2023trlx}, which fine-tuned the GPT-J 6B model \citep{wang2021gpt} with the Reddit TL;DR summarization dataset using the TRLX library \citep{havrilla2023trlx}. For the LLaMA 3, as there is no fine-tuned open model for the Reddit TL;DR summarization, we trained it with \citep{stiennon2020learning}'s dataset. We used NVIDIA A100 80GB for every experiment.

\neurips{For the reward model structure of the general representation function, we froze the first 70\% of the language model's layers, using the outputs of these layers as the representation. For the linear representation function, we froze the entire language model and only trained the additional final layer.}
\clearpage
\subsection{Detailed Aggregated Results in \Cref{ssec:examples}}
\label{ssec:detailed-ex}
We provide the first 5 examples in the test set, and compare the results. The examples might have undesirable vocabulary or topics.

\noindent\makebox[\linewidth]{\rule{0.6\paperwidth}{0.4pt}}

{\color{red}\textbf{Prompt 1}}
\\\texttt{SUBREDDIT: r/AskReddit\\TITLE: How do you get someone out of your head?\\POST: Hi,\\I'm 22, and I have been with my girlfriend for 5 years now. We recently moved together. We've always loved each other intensely.\\\\Problem, I recently started to have feelings for an other person (a friend). This person has had a boyfriend for now 3 years, and has absolutely no ideas. Those feelings were so strong, it was hard to hide them. After 2 months of me being distant and really sad, my girlfriend forced me to say what was bothering me. I'm not a good liar, and now she knows.\\\\We decided to give us a week alone, I went to my parents.~\\\\Now, I'm completely lost. I keep on thinking about this person, and I hate that. I would like for those feelings to go away, to leave me alone. But I can't.~~\\\\What do I do? It's been 3 months now, and I'm just desperate.\\TL;DR:}

{\color{red}\textbf{Aggregation with $\alpha = -\infty$}}
\\
\texttt{I'm in love with my girlfriend, but recently started to have feelings for another person (a friend). How do you get someone out of your head?}

{\color{red}\textbf{Aggregation with $\alpha = -1$}}
\\
\texttt{
 I don't know how to get out of my head, to get over my feelings for someone else, without cheating, without breaking up. Help.}
 
{\color{red}\textbf{Aggregation with $\alpha = 0$}}
\\
\texttt{
 Have in couple for 5 years, moved in together. It has the tendency to fall in love with other random people. How can I get those feelings to leave me alone?}

{\color{red}\textbf{Aggregation with $\alpha = 1$}}
\\
\texttt{
 I have a girlfriend I love, and also a friend that I have feelings for. I don't know how to get rid of those feelings. Please help.}
 
{\color{red}\textbf{Aggregation with $\alpha = \infty$}}
\\
\texttt{I'm with my girlfriend for 5 years, I have feelings for an other person and I need to get those feelings gone.}
\clearpage 
\noindent\makebox[\linewidth]{\rule{0.6\paperwidth}{0.4pt}}

{\color{red}\textbf{Prompt 2}}
\\\texttt{SUBREDDIT: r/pettyrevenge\\TITLE: So, my mom woke me up with a loud TV.\\POST: She was in her living room, watching TV. This was at about 8:30 in the morning, and she was exercising. She turned the TV up extra loud to hear it over her excercycle, and woke me up. I went in there asking for her to turn it down. She said she didn't have to; I explained that I always used headphones so she didn't have to deal with my noise and that she should give me a little more respect, given that I paid rent at the time.\\\\She disagreed. I went back to my room, rather pissed off at the lack of equality. I had no lock on my door; but I had a dresser right next to it, so I pulled one of the drawers out enough so that it caused the door to not be openable. Then, I turned my speakers up really loud and blasted Gangnam Style on repeat, with the bass cranked up as high as it could go.\\\\If you hate Gangnam Style for being overplayed, you will see why I chose that particular song. I personally don't mind it. But here's the thing about my bass; it vibrates the walls, making one hell of a lot of noise. Needless to say, my mom was not pleased and shut off the internet. But it was oh so worth it.\\TL;DR:}

{\color{red}\textbf{Aggregation with $\alpha = -\infty$}}
\\
\texttt{I make a racket when I sleep, then I play Gangnam Style at full volume to "wake me up" and piss my mom off.}

{\color{red}\textbf{Aggregation with $\alpha = -1$}}
\\
\texttt{ Moms wake me up by watching me play video games. She calls it disrespect for her having to listen to my music. I play Gangnam Style on repeat until she shuts off the Internet.}
  
{\color{red}\textbf{Aggregation with $\alpha = 0$}}
\\
\texttt{
  Mom was awake early and didn't have to deal with my noise, so I played Gangnam Style and left it at that.}

{\color{red}\textbf{Aggregation with $\alpha = 1$}}
\\
\texttt{
 Mom wakes me up to background noise even when she does not have headphones on. I crank up Gangnam style on repeat enough to cause some structural damage to her house. }
 
{\color{red}\textbf{Aggregation with $\alpha = \infty$}}
\\
\texttt{So, my mom wakes me up with my TV; I pay rent, but don't always have to bow to her wishes. Fuck yes.  }

\clearpage
\noindent\makebox[\linewidth]{\rule{0.6\paperwidth}{0.4pt}}

{\color{red}\textbf{Prompt 3}}
\\\texttt{SUBREDDIT: r/relationships\\TITLE: My girlfriend (20f) of two years cheated on me (20m) by kissing two guys at a Halloween party.\\POST: Lately her and I have been having a few problems, and these problems have been brought up before a few times. One problem being that I don't show enough affection. I don't tell her she's pretty very often or don't compliment her much. I feel terrible about it, but this time I was really trying to change for her.\\\\For Halloween she went to visit her step brother at a college and I got drunk with my friends and watched movies. Last night (11/1) we got in a huge fight about me not changing and how our relationship won't work out and basically broke up over the phone. So in an effort to try and fix it I drove to her house. She told me how at the parties she went to that two guys kissed her. The first one she pushed away, but the second one I asked her if she kissed him back and she said yes and that she did it because it made her feel wanted, which I guess I haven't been making her feel that way lately. We cried, we talked about everything, we had great sex, and I stayed over at her house just to sleep with her and then snuck out in the morning so her parents wouldn't know.\\\\We both obviously want to work things out but aren't sure if we should. I love this girl, but the more I think about it, all I can think about is her cheating on me, and more importantly, liking it. It makes me sick to my stomach. Should I even try to fix it or would I be better off cutting all ties.\\TL;DR:}

{\color{red}\textbf{Aggregation with $\alpha = -\infty$}}
\\
\texttt{Girlfriend cheated on me for a feeling of closeness, should I even try to fix things or am I just getting screwed? }

{\color{red}\textbf{Aggregation with $\alpha = -1$}}
\\
\texttt{
   girlfriend cheated on me by kissing two guys last night. I love her and don't want to break up, but don't know what to do. Is cheating even worth fighting for?}

{\color{red}\textbf{Aggregation with $\alpha = 0$}}
\\
\texttt{
   My girlfriend of two years, whom I love dearly, recently kissed two guys at a Halloween party. Should I try to salvage the relationship or cut everything off?}
   
{\color{red}\textbf{Aggregation with $\alpha = 1$}}
\\
\texttt{
    Girlfriend of two years cheated on me by kissing two guys, I don't know if I should even try to fix it.}
    
{\color{red}\textbf{Aggregation with $\alpha = \infty$}}
\\
\texttt{Girlfriend of 2 years cheated on me with two guys last night. I feel sick to my stomach. I do still love her, should I work things out or just cut all ties? }

\clearpage
\noindent\makebox[\linewidth]{\rule{0.6\paperwidth}{0.4pt}}

{\color{red}\textbf{Prompt 4}}
\\\texttt{SUBREDDIT: r/running\\TITLE: One year post-pregnancy, aches and pains in hips and knees, can't even comfortably do day 1 of C25k. Help?\\POST: I'll try to keep this short!\\\\**Background**\\\\* I've always been an on again/off again (very casual!) jogger, typically doing 3 - 5 k\\* My knees have always been finicky, and I went to a physio who thought I had "runner's knee"\\* Pre-pregnancy, my "runner's knee" would flare up when I got to the 8 - 10 k distance range, even if I had a decent base (doing a C210k type program)\\\\**Current Problem**\\\\I had my baby a year ago, so all in all I haven't run for about 1.5 years.~~I'm quite slim and have been doing aerobics-style classes for the past year, so I'm not totally out of shape.~~Body weight exercises, aerobics, bikes and ellipticals are all fine.~However, when I run even the tinyiest bit, or even go on a long walk or a hike, my pelvis gets very sore and tight, and my knees start hurting very quickly.~~I already am doing general squats/lunges/stretching type things.~\\\\I'm starting to feel like running just isn't for me anymore. Which is a bummer, because I really enjoy running!\\\\Has anyone had something similar? Can anyone recommend some stretches or exercises that might help?~~Should I see a Dr?~~Or should I just see a physio?~~Not quite sure how to proceed.\\\\Thanks!\\TL;DR:}

{\color{red}\textbf{Aggregation with $\alpha = -\infty$}}
\\
\texttt{Have always been a runner, always had knee troubles, no longer pregnant, aches and pains from running now.~~Help? }

{\color{red}\textbf{Aggregation with $\alpha = -1$}}
\\
\texttt{
     Was a fan of running.  Has your body ached?   What happened when you started running again?  How did you go about it? :)}
    
{\color{red}\textbf{Aggregation with $\alpha = 0$}}
\\
\texttt{
    Had runner's knee, started jogging again immediately post-pregnancy, getting back to where I was at pre-pregnancy.  Now pelvis/knee pain after even tiny runs.  What do?}

{\color{red}\textbf{Aggregation with $\alpha = 1$}}
\\
\texttt{
     running post-pregnancy, pelvis/knees are very unforgiving.  How best to proceed (stretches? exercises? etc.)?  And what can I do to make it better?}

{\color{red}\textbf{Aggregation with $\alpha = \infty$}}
\\
\texttt{1.5 year post-pregnancy, hip/knee pain while "running" even a few feet, no injuries/illnesses/pain beforehand, not out of shape.~~Advice? }

\clearpage
\noindent\makebox[\linewidth]{\rule{0.6\paperwidth}{0.4pt}}

{\color{red}\textbf{Prompt 5}}
\\\texttt{SUBREDDIT: r/tifu\\TITLE: TIFU by getting scared for no reason and peeing on myself\\POST: (Pre-apology for formatting errors since i'm on my phone)\\\\Well, this was embarrassing and happened a few minutes ago.~\\I work at a small law office which shares a building with a few other businesses, a dentist's office and a military recruitment center. There is no drinking fountain in the building, so I bring a giant 2L water bottle which I sip on all day to keep hydrated and counteract all the copious amounts of alcohol i drink (after work, of course). Unfortunately, I have been cursed with a small bladder, and have to go pee sometimes 8-10 times a day. Luckily the back door of my office opens up right near the bathroom, so I can usually drain the main vein without embarrassing myself by being frequently seen there.\\\\As I was wrapping up work for the day, I realized I would have to pee one more time before leaving. It was slightly after 5, so when I walked to the bathroom, I wasn't expecting to see anyone in there (it is the kind of bathroom with a urinal and a stall).\\\\Fuckup begins when I open the door and there is a dude who had just finished shitting and was leaving the stall. For whatever reason this terrified me for a split second and I jumped and gasped.~\\\\The guy was amused and chuckled and was like, did I scare you? I muttered something like "haha, yeah, a little bit" and I kept my head down, beelined to the urinal and unleashed the krakan.~\\\\I was super embarrassed and very preoccupied so I didn't notice that the bottom part of my untucked, button-down shirt had slipped into my urine stream and was getting soaked. I didn't realize it until it began to sort of soak back onto my pant leg and getting my thigh wet.~\\\\Well, fuck, i thought. This sucks. Then i remembered my fiancee was coming to pick me up straight from work so we could go grocery shopping. I cleaned up as best as i could, but its still very obvious. Luckily I have a P coat which is long enough to cover my pee-soaked netherregion, so i may get through this unscathed. My girl didn't\\TL;DR:}

{\color{red}\textbf{Aggregation with $\alpha = -\infty$}}
\\
\texttt{was scared for no reason, and pissed on myself when doing so. Not sure if I'm going to be able to avoid a major wardrobe change before tonight... }

{\color{red}\textbf{Aggregation with $\alpha = -1$}}
\\
\texttt{
     I was startled by someone, lost my bladder control and pee'd all over myself, after I finished up and was getting ready to leave for my fiancees house. Fortunately for me, it looks like i just peed myself. }
     
{\color{red}\textbf{Aggregation with $\alpha = 0$}}
\\
\texttt{
     I got scared for no reason and peed on myself, then covered it up with my fiancee as we're about to leave work, but it's still visible.}

{\color{red}\textbf{Aggregation with $\alpha = 1$}}
\\
\texttt{ work didn't have a toilet, had to use bathroom at my office. Sharted in the bathroom stall and got my shirt stuck to me, but may have gotten rid of all evidence of the incident by bringing a P coat }

{\color{red}\textbf{Aggregation with $\alpha = \infty$}}
\\
\texttt{I jumped, freaked out, urinated, and forgot to change my wet shirt. When I got home, I had to change the dress so my fiancee wouldn't catch a glimpse. }

\neurips{\newpage
\section*{NeurIPS Paper Checklist}
\begin{enumerate}

\item {\bf Claims}
    \item[] Question: Do the main claims made in the abstract and introduction accurately reflect the paper's contributions and scope?
    \item[] Answer: \answerYes{}
    \item[] Justification: We provided theorems and corresponding proofs. Also, we have extensive experiments. 

\item {\bf Limitations}
    \item[] Question: Does the paper discuss the limitations of the work performed by the authors?
    \item[] Answer: \answerYes{}
    \item[] Justification: We provided the explanation about the lower bound results - which might have a gap with upper bound results. 

\item {\bf Theory Assumptions and Proofs}
    \item[] Question: For each theoretical result, does the paper provide the full set of assumptions and a complete (and correct) proof?
    \item[] Answer: \answerYes{}
    \item[] Justification: We provided every proof and every assumption that we used. 

    \item {\bf Experimental Result Reproducibility}
    \item[] Question: Does the paper fully disclose all the information needed to reproduce the main experimental results of the paper to the extent that it affects the main claims and/or conclusions of the paper (regardless of whether the code and data are provided or not)?
    \item[] Answer: \answerYes{}
    \item[] Justification: 
    We provided how to implement our algorithm very carefully. We disclose what package we used too. 

\item {\bf Open access to data and code}
    \item[] Question: Does the paper provide open access to the data and code, with sufficient instructions to faithfully reproduce the main experimental results, as described in supplemental material?
    \item[] Answer: \answerYes{}
    \item[] We provided the explanation and we will also include the reward function module in our GitHub after rebuttal. Other codes are already open data and open models.  

\item {\bf Experimental Setting/Details}
    \item[] Question: Does the paper specify all the training and test details (e.g., data splits, hyperparameters, how they were chosen, type of optimizer, etc.) necessary to understand the results?
    \item[] Answer: \answerYes{}
    \item[] Justification: We provided it in the Appendix about every training and test detail. 

\item {\bf Experiment Statistical Significance}
    \item[] Question: Does the paper report error bars suitably and correctly defined or other appropriate information about the statistical significance of the experiments?
    \item[] Answer: \answerYes{}
    \item[] Justification: We provided 3 times repeat on our experiments and provided error bars in the plot. 

\item {\bf Experiments Compute Resources}
    \item[] Question: For each experiment, does the paper provide sufficient information on the computer resources (type of compute workers, memory, time of execution) needed to reproduce the experiments?
    \item[] Answer:\answerYes{}
    \item[] Justification: We provided what type of GPU we used. 
    
\item {\bf Code Of Ethics}
    \item[] Question: Does the research conducted in the paper conform, in every respect, with the NeurIPS Code of Ethics \url{https://neurips.cc/public/EthicsGuidelines}?
    \item[] Answer: \answerYes{}
    \item[] Justification: We followed every instruction in \url{https://neurips.cc/public/EthicsGuidelines}. 

\item {\bf Broader Impacts}
    \item[] Question: Does the paper discuss both potential positive societal impacts and negative societal impacts of the work performed?
    \item[] Answer: \answerYes{}.
    \item[] Justification: We provided societal impacts by discussing imposing a cost of RLHF labeling for truthful responses. We provided an example of why this might be possible. 
    
\item {\bf Safeguards}
    \item[] Question: Does the paper describe safeguards that have been put in place for responsible release of data or models that have a high risk for misuse (e.g., pretrained language models, image generators, or scraped datasets)?
    \item[] Answer: \answerNA{}.
    \item[] Justification: We used an open dataset and open weight model. 

\item {\bf Licenses for existing assets}
    \item[] Question: Are the creators or original owners of assets (e.g., code, data, models), used in the paper, properly credited and are the license and terms of use explicitly mentioned and properly respected?
    \item[] Answer: \answerYes{}
    \item[] Justification: We referred to every dataset in our paper, which is also an open dataset and model. As some datasets do not mention the exact license, we rather choose to properly refer to the paper and dataset websites. 

\item {\bf New Assets}
    \item[] Question: Are new assets introduced in the paper well documented and is the documentation provided alongside the assets?
    \item[] Answer: \answerNA{}.
    \item[] Justification: We do not make a new open dataset. 

\item {\bf Crowdsourcing and Research with Human Subjects}
    \item[] Question: For crowdsourcing experiments and research with human subjects, does the paper include the full text of instructions given to participants and screenshots, if applicable, as well as details about compensation (if any)? 
    \item[] Answer: \answerNA{}.
    \item[] Justification: We do not have crowdsourcing, and we used an open dataset for human feedback, which is fully anonymous. 

\item {\bf Institutional Review Board (IRB) Approvals or Equivalent for Research with Human Subjects}
    \item[] Question: Does the paper describe potential risks incurred by study participants, whether such risks were disclosed to the subjects, and whether Institutional Review Board (IRB) approvals (or an equivalent approval/review based on the requirements of your country or institution) were obtained?
    \item[] Answer: \answerNA{}.
    \item[] Justification: We do not have crowdsourcing.

\end{enumerate}}
\end{document}